%% file: example_paper.tex
\documentclass{article}

\usepackage{microtype}
\usepackage{graphicx}
\usepackage{subfigure}
\usepackage{booktabs} 

\usepackage{hyperref}

\usepackage[accepted]{icml2024}

\usepackage{amsmath}
\usepackage{amssymb}
\usepackage{mathtools}
\usepackage{amsthm}
\usepackage{enumerate}
\usepackage{comment}

\usepackage{colortbl}
\usepackage{tablefootnote}
\usepackage{makecell}

\usepackage[capitalize,noabbrev]{cleveref}

\theoremstyle{plain}
\newtheorem{theorem}{Theorem}[section]
\newtheorem{proposition}[theorem]{Proposition}
\newtheorem{lemma}[theorem]{Lemma}

\theoremstyle{definition}
\newtheorem{definition}[theorem]{Definition}
\newtheorem{assumption}[theorem]{Assumption}
\theoremstyle{remark}

\input{math_commands.tex}

\usepackage[textsize=tiny]{todonotes}

\icmltitlerunning{A Nearly Optimal Single Loop Algorithm for Stochastic Bilevel Optimization under Unbounded Smoothness }

\begin{document}

\twocolumn[
\icmltitle{A Nearly Optimal Single Loop Algorithm \\for Stochastic Bilevel Optimization under Unbounded Smoothness}




\icmlsetsymbol{equal}{*}




\begin{icmlauthorlist}
\icmlauthor{Xiaochuan Gong}{GMU}
\icmlauthor{Jie Hao}{GMU}
\icmlauthor{Mingrui Liu}{GMU}
\end{icmlauthorlist}

\icmlaffiliation{GMU}{Department of Computer Science, George Mason University, USA}

\icmlcorrespondingauthor{Mingrui Liu}{mingruil@gmu.edu}

\icmlkeywords{Machine Learning, ICML}

\vskip 0.3in
]



\printAffiliationsAndNotice{}

\begin{abstract}
This paper studies the problem of stochastic bilevel optimization where the upper-level function is nonconvex with potentially unbounded smoothness and the lower-level function is strongly convex. This problem is motivated by meta-learning applied to sequential data, such as text classification using recurrent neural networks, where the smoothness constant of the upper-level loss function
scales linearly with the gradient norm and can be potentially unbounded. Existing algorithm crucially relies on the nested loop design, which requires significant tuning efforts and is not practical. In this paper, we address this issue by proposing a Single Loop bIlevel oPtimizer (SLIP). The proposed algorithm first updates the lower-level variable by a few steps of stochastic gradient descent, and then simultaneously updates the upper-level variable by normalized stochastic gradient descent with momentum and the lower-level variable by stochastic gradient descent. Under standard assumptions, we show that our algorithm finds an $\epsilon$-stationary point within $\widetilde{O}(1/\epsilon^4)$\footnote{Here $\widetilde{O}(\cdot)$ compresses logarithmic factors of $1/\epsilon$ and $1/\delta$, where $\delta\in(0,1)$ denotes the failure probability.} oracle calls of stochastic gradient or Hessian-vector product, both in expectation and with high probability. This complexity result is nearly optimal up to logarithmic factors without mean-square smoothness of the stochastic gradient oracle. Our proof relies on (i) a refined characterization and control of the lower-level variable and (ii) establishing a novel connection between bilevel optimization and stochastic optimization under distributional drift. Our experiments on various tasks show that our algorithm significantly outperforms strong baselines in bilevel optimization. The code is available \href{https://github.com/MingruiLiu-ML-Lab/Single-Loop-bilevel-Optimizer-under-Unbounded-Smoothness}{here}.  
\end{abstract}



\section{Introduction}
There has been a surge in interest in bilevel optimization, driven by its broad applications in machine learning. This includes meta-learning~\citep{franceschi2018bilevel,rajeswaran2019meta}, hyperparameter optimization~\citep{franceschi2018bilevel,feurer2019hyperparameter}, fair model training~\cite{roh2020fairbatch}, continual learning~\citep{borsos2020coresets,hao2023bilevel}, and reinforcement learning~\citep{konda1999actor}. The bilevel optimization has the following form:
\begin{equation}
\label{eq:bp}
\begin{aligned}
    \min_{x\in \mathbb{R}^{d_x}} \Phi(x):=f(x, y^*(x)), 
    \ \
    \text{s.t.}, \ \ y^*(x) \in \mathop{\arg \min}_{y\in \mathbb{R}^{d_y}} g(x, y),
\end{aligned}
\end{equation}%
where $f$ and $g$ are upper-level and lower-level functions respectively. For example, in meta-learning~\cite{finn2017model,franceschi2018bilevel}, the bilevel formulation is trying to find a common representation such that it can quickly adapt to different tasks by adjusting its task-specific head, where $x$ denotes the shared representation across different tasks, $y$ denotes the task-specific head. We consider the stochastic optimization setting, where $f(x,y)=\mathbb{E}_{\xi\sim\mathcal{D}_f}\left[F(x,y;\xi)\right]$ and $g(x,y)=\mathbb{E}_{\zeta\sim\mathcal{D}_g}\left[G(x,y;\zeta)\right]$, where $\mathcal{D}_f$ and $\mathcal{D}_g$ are underlying data distributions for $f$ and $g$ respectively.

Most theoretical studies and algorithmic developments on bilevel optimization typically assume that the upper-level function is smooth and nonconvex (i.e., the gradient is Lipschitz), and the lower-level function is strongly convex~\cite{ghadimi2018approximation,hong2023two,grazzi2022bilevel,ji2021bilevel}. However, recent work shows that there is a class of neural networks whose smoothness parameter is potentially unbounded~\cite{zhang2020gradient,crawshaw2022robustness}, such as recurrent neural networks (RNNs)~\cite{elman1990finding}, long-short-term memory networks (LSTMs)~\cite{hochreiter1997long} and transformers~\cite{vaswani2017attention}. Therefore, it is important to design algorithms when the upper-level function has an unbounded smoothness parameter. Recently,~\cite{hao2024bilevel} designed an algorithm under this regime and proved its convergence to stationary points. However, the structure of their algorithm is rather complex: it has a nested double loop to update the lower-level variable, which requires significant tuning efforts. It remains unclear how to design a provably efficient single-loop algorithm under this setting.

Designing a single-loop algorithm for bilevel problem with a unbounded smooth upper-level function has the following challenges. First, it is difficult to apply the similar analysis strategy of previous work~\cite{hao2024bilevel} to any single-loop algorithms. ~\citet{hao2024bilevel} designed a complicated procedure to update the lower-level variable to obtain an accurate estimator for the optimal lower-level solution at every iteration, but single-loop algorithms typically cannot achieve the same goal. Second, we need to simultaneously control the progress and their mutual dependence on the upper-level and lower-level problems, even if the lower-level variable may not be accurate.

To address these challenges, we propose a new single-loop bilevel optimizer, namely SLIP. The algorithm is surprisingly simple: after a few iterations of stochastic gradient descent (SGD) to update the lower-level variable, the algorithm simultaneously updates the upper-level variable by normalized stochastic gradient descent with momentum and updates the lower-level variable by SGD. It does not perform periodic updates of the lower-level variable as required in~\cite{hao2024bilevel}. The algorithm analysis relies on two important new techniques. First, we provide a refined characterization and control for the lower-level variable so that there is no need to get an accurate estimate of the lower-level variable at every iteration, while we can still ensure convergence. Second, we establish a novel connection between bilevel optimization and stochastic optimization under distributional drift~\cite{cutler2023stochastic}, and it helps us analyze the error of the lower-level variable. Our contributions of this paper are summarized as follows.
\begin{itemize}
\item We design a new single-loop algorithm named SLIP, for solving stochastic bilevel optimization problem where the upper-level function is nonconvex and unbounded smooth and the lower-level function is strongly convex. To the best of our knowledge, this is the first single-loop algorithm in this setting.

\item We prove that the algorithm converges to the $\epsilon$-stationary point within $\widetilde{O}(1/\epsilon^4)$ oracle calls of the stochastic gradient or the Hessian-vector product, both in expectation and with high probability. This complexity result is nearly optimal up to logarithmic factors without mean-square smoothness of the stochastic gradient oracle (A summary of our results and a comparison to prior work are provided in \Cref{tab:comparison} at \Cref{sec:comparison-table}). The proof relies on a novel characterization of the lower-level variable and a novel connection between bilevel optimization and stochastic optimization with distributional drift, which are new and not leveraged by previous work on bilevel optimization.

\item We perform extensive experiments on hyper-representation learning and data hypercleaning on text classification tasks. Our algorithm shows significant speedup compared with strong baselines in bilevel optimization.
\end{itemize}





\section{Related Work}
\textbf{Bilevel Optimization.} Bilevel optimization was introduced by~\cite{bracken1973mathematical}. Some classical algorithms were proposed for certain classes of bilevel optimization and asymptotic convergence was provided~\cite{vicente1994descent,anandalingam1990solution,white1993penalty}. Recently,~\citet{ghadimi2018approximation} pioneered the non-asymptotic convergence of gradient-based methods for bilevel optimization when the lower-level problem is strongly convex. This complexity result was improved by a series of work~\cite{hong2023two,chen2021closing,chen2022single,ji2021bilevel,chen2023optimal}. When the function has a mean squared smooth stochastic gradient, the complexity was further improved by incorporating variance reduction and momentum techniques~\cite{khanduri2021near,dagreou2022framework,guo2021randomized,yang2021provably}. There is a line of work that designed fully first-order algorithms for stochastic bilevel optimization problems~\cite{liu2022bome,kwon2023fully}. There is also another line of work that considered the setting of a non-strongly convex lower-level problem~\cite{sabach2017first,sow2022constrained,liu2020generic,shen2023penalty,kwon2023penalty,chen2023bilevel1}. The most relevant work in this paper is~\cite{hao2024bilevel}, which considered the setting of unbounded smooth functions of the upper level and strongly convex functions of the lower level. However, their algorithm is not single-loop and requires extensive tuning efforts.

\textbf{Unbounded Smoothness.} The definition of relaxed smoothness was first proposed by~\cite{zhang2020gradient}, which was motivated by the loss landscape of recurrent neural networks.~\citet{zhang2020gradient} showed gradient clipping/normalization improves over gradient descent in this setting. Later, there is a line of work focusing on improved analysis~\cite{zhang2020improved,jin2021non}, scalability~\cite{liu2022communication,crawshaw2023episode,crawshaw2023federated} and adaptive variants~\cite{crawshaw2022robustness,faw2023beyond,wang2023convergence,li2023convergence}. Recently, some works studied the momentum and variance reduction techniques under individual relaxed smooth~\cite{liu2023nearlyoptimal} or on-average relaxed smooth conditions~\cite{reisizadeh2023variance} to achieve an improved convergence rate. People also studied various notions of relaxed smoothness, including coordinate-wise relaxed smoothness~\cite{crawshaw2022robustness}, $\alpha$-symmetric generalized smoothness~\cite{chen2023generalized}, and relaxed smoothness under the bilevel setting~\cite{hao2024bilevel}. This work considers the same problem setting as in~\cite{hao2024bilevel}, and focuses on the design of the single-loop algorithm.

\textbf{Stochastic Optimization under Distributional Drift.} Stochastic optimization with time drift is a classical topic and was studied extensively in the literature~\cite{dupavc1965dynamic,guo1995exponential}, mostly in the context of least squares problems and Kalman filtering. Recent work revisits these algorithms from the perspective of sequential stochastic/online optimization~\cite{besbes2015non,wilson2018adaptive,madden2021bounds,cutler2023stochastic}. However, none of these works explores the connection between techniques in sequential optimization and the bilevel optimization problem as considered in this paper.

\section{Preliminaries, Notations and Problem Setup}
Define $\langle \cdot, \cdot\rangle$ and $\|\cdot\|$ as the inner product and Euclidean norm. Throughout the paper, we use asymptotic notation $\widetilde{O}(\cdot), \widetilde{\Theta}(\cdot), \widetilde{\Omega}(\cdot)$ to hide polylogarithmic factors in $1/\epsilon$ and $1/\delta$. Denote $f$: $\mathbb{R}^{d_x}\times\mathbb{R}^{d_y} \rightarrow \mathbb{R}$ as the upper-level function, and $g$: $\mathbb{R}^{d_x}\times\mathbb{R}^{d_y} \rightarrow \mathbb{R}$ as the lower-level function. The hypergradient $\gdphi(x)$ shown in~\cite{ghadimi2018approximation} takes the form of
\begin{equation}
    \gdphi(x) = \gdx f(x,y^*(x)) - \gdxy g(x,y^*(x))z^*(x), 
\end{equation}
where $z^*(x)$ is the solution to the linear system:
\begin{small}
\begin{equation*}
    z^*(x) = \argmin_z \frac{1}{2}\langle \gdyy g(x,y^*(x))z, z \rangle - \langle \gdy f(x,y^*(x)), z \rangle.
\end{equation*}
\end{small}%
We make the following assumptions for the paper.
%

\begin{assumption}[$(L_{x,0}, L_{x,1}, L_{y,0}, L_{y,1})$-smoothness~\cite{hao2024bilevel}] \label{ass:relax-smooth}
Let $w=(x,y)$ and $w'=(x',y')$, there exists $L_{x,0}, L_{x,1}, L_{y,0}, L_{y,1} > 0$ such that for all $w,w'$, if $\|w-w'\| \leq 1/\sqrt{2(L_{x,1}^2+L_{y,1}^2)}$, then
\begin{small}
\begin{equation*}
    \begin{aligned}
        \|\gdx f(w)-\gdx f(w')\| &\leq (L_{x,0}+L_{x,1}\|\gdx f(w)\|)\|w-w'\|, \\
        \|\gdy f(w)-\gdy f(w')\| &\leq (L_{y,0}+L_{y,1}\|\gdy f(w)\|)\|w-w'\|.
    \end{aligned}
\end{equation*}
\end{small}%
\end{assumption}
\textbf{Remark:} Assumption~\ref{ass:relax-smooth} is introduced by~\cite{hao2024bilevel}, and they empirically show that it is satisfied on recurrent neural networks with $y$ being the last linear layer and $x$ being previous layers. This assumption can be regarded as a generalization of the relaxed smoothness condition from single-level problems~\cite{zhang2020gradient,crawshaw2022robustness} to the bilevel problem.

\begin{assumption} \label{ass:f-and-g}
Suppose the following holds for objective functions $f$ and $g$: 
(i) For every $x$, $\|\gdy f(x,y^*(x)\| \leq l_{f,0}$;
(ii) For every $x$, $g(x, y)$ is $\mu$-strongly convex in $y$ for $\mu>0$; 
(iii)  $g$ is $l_{g,1}$-smooth jointly in $(x, y)$; 
(iv) $g$ is twice continuously differentiable, and $\gdxy g, \gdyy g$ are $l_{g,2}$-Lipschitz jointly in $(x, y)$.
\end{assumption}

\textbf{Remark}: Assumption~\ref{ass:f-and-g} is standard in the bilevel optimization literature~\cite{kwon2023fully,hao2024bilevel,ghadimi2018approximation}. In particular, Assumption~\ref{ass:f-and-g} (i) is theoretically and empirically justified by~\cite{hao2024bilevel} for recurrent neural networks. 
Under Assumptions~\ref{ass:relax-smooth} and~\ref{ass:f-and-g}, we can show that function $\Phi(x)$ satisfies standard relaxed smoothness condition: $\|\gdphi(x)-\gdphi(x')\| \leq (L_0+L_1\|\gdphi(x')\|)\|x-x'\|$ if $x$ and $x'$ are not far away from each other (Lemma~\ref{eq:technical-rs-phi} in Appendix).

\begin{assumption} \label{ass:noise-expectation}
The estimators used to calculate stochastic gradients and Hessian-vector products (i.e., $\gdx F(x,y;\xi)$,$\gdy F(x,y;\xi)$, $\gdy G(x,y;\xi)$, $\gdxy G(x,y;\zeta)$, $\gdyy G(x,y;\zeta)$) are \emph{unbiased} and satisfy (assume $\lambda>0$):
\begin{small}
\begin{equation*}
    \begin{aligned}
        &\E_{\xi\sim\gD_f}[\|\gdx F(x,y;\xi)-\gdx f(x,y)\|^2] \leq \sigma_{f,1}^2, \\
        &\E_{\xi\sim\gD_f}[\|\gdy F(x,y;\xi)-\gdy f(x,y)\|^2] \leq \sigma_{f,1}^2, \\
        &\Pr\{\|\gdy G(x,y;\xi)-\gdy g(x,y)\| \geq \lambda\} \leq 2\exp\left(-2\lambda^2/\sigma_{g,1}^2\right), \\
        &\E_{\zeta\sim\gD_g}[\|\gdxy G(x,y;\zeta)-\gdxy g(x,y)\|^2] \leq \sigma_{g,2}^2, \\
        &\E_{\zeta\sim\gD_g}[\|\gdyy G(x,y;\zeta)-\gdyy g(x,y)\|^2] \leq \sigma_{g,2}^2. \\
    \end{aligned}
\end{equation*}
\end{small}%
\end{assumption}

\textbf{Remark}: Assumption~\ref{ass:noise-expectation} assumes that we have access to an unbiased stochastic gradient and a Hessian-vector product with bounded variance, which is standard in the literature~\cite{ghadimi2013stochastic,ghadimi2018approximation}. We need the stochastic gradient of the lower-level problem to be light-tailed, which is also assumed by~\cite{hao2024bilevel}. This is a technical assumption that allows high probability analysis for $y$, which is a standard assumption in the optimization literature~\cite{lan2012optimal,hazan2014beyond}.


\section{Algorithm and Theoretical Analysis}
\subsection{Main Challenges and Algorithm Design}
\textbf{Main Challenges.} We first illustrate why the existing work on bilevel optimization is insufficient for solving our problem with a single-loop algorithm. First, the analyses of single-loop bilevel algorithms for smooth bilevel optimization in the literature (i.e., the upper-level function has a Lipschitz gradient)~\cite{hong2023two,dagreou2022framework,chen2023bilevel1} typically design a single potential function to track progress in terms of both upper-level and lower-level variables and show that the potential function can decrease in expectation during algorithm updates. Their analysis is crucially based on $L$-smoothness of the upper-level function to control the approximation error of hypergradient. However, when the upper-level function is $(L_{x,0}, L_{x,1}, L_{y,0}, L_{y,1})$-smooth, the bias in the hypergradient error depends on both the approximation error of the lower-level variable and the hypergradient in terms of the upper-level variable, which are statistically dependent and therefore cannot be analyzed easily using the standard expectation-based analysis for all variables. Second, the recent work of~\cite{hao2024bilevel} addressed this issue by designing a new algorithm, along with a high probability analysis for the lower-level variable and a expectation-based analysis for the upper-level variable, but their algorithm crucially relies on the nested-loop design. The algorithm in~\cite{hao2024bilevel} has two components to update its lower-level variable: initialization refinement and periodic updates. Their key idea is to obtain an accurate estimator for the optimal lower-level variable at every iteration with high probability such that the hypergradient error can be controlled. However, such a strong requirement for the lower-level variable typically cannot be satisfied by any single-loop algorithms: single-loop algorithms typically make small progress on each step and can get an accurate solution after polynomial number of iterations, but cannot guarantee good estimates at every iteration. 

\textbf{Algorithm Design.} To address these challenges, our algorithm design principle relies on the following important observation: obtaining accurate estimates for the lower-level variable at every iteration with high probability is only a sufficient condition for the algorithm in~\cite{hao2024bilevel} to work, but it is not necessary: it is possible to establish a refined characterization and control of the lower-level variable which has weaker requirements than~\cite{hao2024bilevel}. This important insight makes the design of a single loop algorithm possible. The detailed description of our algorithm is illustrated in Algorithm~\ref{alg:bilevel}. In particular, our algorithm first updates the lower-level variable by a few steps of SGD (line 3), and then simultaneously updates the upper-level variable by normalized SGD with momentum and the lower-level variable by SGD (line 4$\sim$10). Our key novelty in the algorithm design and analysis comes from the novel perspective of connecting bilevel optimization and stochastic optimization under distribution drift: the procedure of updating the lower-level variable can be regarded as a stochastic optimization under distribution drift and the drift between iterations is small. In particular, the update rule in Algorithm~\ref{alg:bilevel} for warm-start (line 3) and $y$ (line 7)  can be viewed as a special case of SGD with distribution drift as specified in Algorithm~\ref{alg:sgd}, where the drift comes from the change of $x$ and the change of $x$ is small due to the normalization operator. Note that in the warm-start step (line 3) in Algorithm~\ref{alg:bilevel}, $x_0$ is fixed and there is no distribution drift. 

\textbf{Difference between SLIP and~\cite{hao2024bilevel}.} 
First, the work of~\cite{hao2024bilevel} has the initialization refinement subroutine to obtain an accurate initial estimate for the optimal lower-level solution to the initial upper-level variable (that is, $y^*(x_0)$), which requires epoch-SGD~\cite{hazan2015beyond,ghadimi2013optimal} for polynomial number of iterations. In contrast, our algorithm has a short warm-start stage (line 3), which runs SGD for only logarithmic number of iterations. Second, the work of~\cite{hao2024bilevel} needs to periodically update its lower-level variable for polynomial number of iterations, but our algorithm only needs to simply run SGD for the lower-level variable at every iteration and our lower-level update is performed simultaneously with the upper-level update (line 4$\sim$10). 







\begin{algorithm}[!t]
    \caption{\textsc{SGD with Distributional Drift}} \label{alg:sgd}
    \begin{algorithmic}[1]
        \STATE \textbf{Input:} $\{\Tilde{x}_t\}, \Tilde{y}_0, \alpha, N$  \hfill \# \texttt{SGD-DD}$(\{\Tilde{x}_t\}, \Tilde{y}_0, \alpha, N)$
        \FOR{$t=0, 1, \dots, N-1$}
            \STATE Sample $\Tilde{\pi}_t$ from distribution $\gD_g$
            \STATE $\Tilde{y}_{t+1} = \Tilde{y}_t - \alpha\gdy G(\Tilde{x}_t,\Tilde{y}_t;\Tilde{\pi}_t)$
        \ENDFOR
    \end{algorithmic}
\end{algorithm}

\begin{algorithm}[!t]
    \caption{\textsc{Single Loop bIlevel oPtimizer (SLIP)}} \label{alg:bilevel}
    \begin{algorithmic}[1]
        \STATE \textbf{Input:} $\alpha^{\init}, \alpha, \beta, \gamma, \eta, T_0, T$ 
        \STATE \textbf{Initialize:} $x_0, y_0^{\init}, z_0, m_0 = 0$
        \STATE $y_0 = \texttt{SGD-DD}(\{x_0\}, y_0^{\init}, \alpha^{\init}, T_0)$ \hfill \# Warm-start
        \FOR{$t=0, 1, \dots, T-1$}
            \STATE Sample $\xi_t$, $\xi'_t$ independently from distribution $\gD_f$
                        \STATE Sample $\pi_t$, $\zeta_t$, $\zeta'_t$ independently from distribution $\gD_g$
            \STATE $y_{t+1} = y_t - \alpha \gdy G(x_t,y_t;\pi_t)$
            \STATE $z_{t+1} = z_{t} - \gamma[\gdyy G(x_t,y_{t};\zeta_t)z_{t} - \gdy F(x_t,y_{t};\xi_t)]$ 
            \STATE $m_{t+1} = \beta m_t + (1-\beta)(\gdx F(x_t,y_t;\xi'_t) - \gdxy G(x_t,y_t;\zeta'_t)z_t)$
            \STATE $x_{t+1} = x_t - \eta\frac{m_{t+1}}{\|m_{t+1}\|}$ 
        \ENDFOR
    \end{algorithmic}
\end{algorithm}
\subsection{Main Results}
We first introduce a few notations. Define $\sigma(\cdot)$ as the $\sigma$-algebra generated by the random variables in the argument. We define the following filtrations: $\widetilde{\mathcal{F}}_t^{1}=\sigma(\Tilde{\pi}_0,\ldots,\Tilde{\pi}_{t-1})$, $\mathcal{F}_t^{1}=\sigma(\pi_0,\ldots,\pi_{t-1})$, $\mathcal{F}_t^{2}=\sigma(\xi_0,\ldots,\xi_{t-1}, \zeta_0,\ldots,\zeta_{t-1})$, $\mathcal{F}_t^{3}=\sigma(\xi'_0,\ldots,\xi'_{t-1}, \zeta'_0,\ldots,\zeta'_{t-1})$, where $1\leq t\leq T$. Define $\widetilde{\mathcal{F}}_t=\sigma(\mathcal{F}_t^{2} \cup \mathcal{F}_t^{3})$, $\mathcal{F}_t=\sigma(\widetilde{\mathcal{F}}_{T_0}^{1} \cup\mathcal{F}_t^{1}\cup \mathcal{F}_t^{2}\cup \mathcal{F}_t^{3})$. 
We use $\mathbb{E}_t$ and $\mathbb{E}$ to denote the conditional expectation $\mathbb{E}[\cdot \mid \widetilde{\mathcal{F}}_t]$ and the total expectation over $\widetilde{\mathcal{F}}_T$ respectively. 
In Algorithm~\ref{alg:bilevel}, define $\hatphi(x,y,z;\xi,\zeta) \coloneqq \gdx F(x,y;\xi) - \gdxy G(x,y;\zeta)z$ and thus we could write line 9 as $m_{t+1} = \beta m_t + (1-\beta)\hatphi(x_t,y_t,z_t;\xi_t',\zeta_t')$.



\subsubsection{Convergence in Expectation} \label{sec:Convergence in Expectation}
\begin{theorem} \label{thm:main-thm}
Suppose Assumptions~\ref{ass:relax-smooth} and~\ref{ass:f-and-g} hold. Let $\{x_t\}$ be the iterates produced by Algorithm~\ref{alg:bilevel}. For any given $\delta\in(0,1)$ and sufficiently small $\eps$ (see the exact choice of $\eps$ in \eqref{eq:eps}), if we choose $\alpha^{\init}, \alpha,\beta,\gamma,\eta, T_0$ as 
\begin{footnotesize}
\begin{equation*}
    \alpha^{\init} = \min\left\{\frac{1}{2l_{g,1}}, \frac{\mu}{2048L_1^2\sigma_{g,1}^2\log(e/\delta)}\right\},
\end{equation*}%
\end{footnotesize}
\begin{footnotesize}
\begin{equation*}
    1-\beta = \Theta\left(\frac{\mu^2\eps^2}{L_0^2\sigma_{g,1}^2\log^2(B)}\right), 
    \eta = \frac{\mu\eps}{8l_{g,1}L_0\log(A)}(1-\beta),
\end{equation*}
\end{footnotesize}%
\begin{footnotesize}
\begin{equation*}
    \gamma = \frac{1-\beta}{\mu}, \quad \alpha = \frac{8(1-\beta)}{\mu}, \quad T_0 = \frac{\log\left(256L_1^2\|y_0^{\init}-y_0^*\|^2\right)}{\log\left(2/(2-\mu\alpha^{\init})\right)},
\end{equation*}
\end{footnotesize}%
where $\Delta_0, A$ and $B$ are defined in \eqref{eq:ABdelta},
then with probability at least $1-2\delta$ over the randomness in $\sigma(\Tilde{\gF}_{T_0}^1 \cup \mathcal{F}_T^1)$, Algorithm~\ref{alg:bilevel} guarantees $\frac{1}{T}\sum_{t=0}^{T-1}\E\|\gdphi(x_t)\| \leq 14\eps$ with at most $T = \frac{4\Delta_0}{\eta\eps}$ iterations, where the expectation is taken over over the randomness in $\widetilde{\mathcal{F}}_T$. 
\end{theorem}

\textbf{Remark:} The full specification and the proofs are included in Appendix (i.e., Theorem~\ref{thm:appendix-main}). Theorem~\ref{thm:main-thm} shows that if $\eta=\widetilde{\Theta}(\epsilon^3)$, $1-\beta=\widetilde{\Theta}(\epsilon^2)$, $\alpha=\widetilde{\Theta}(\epsilon^2)$ and $\gamma=\widetilde{\Theta}(\epsilon^2)$, then Algorithm~\ref{alg:bilevel} converges to $\epsilon$-stationary points within $\widetilde{O}(\epsilon^{-4})$ iterations in expectation. This complexity result matches the double-loop algorithm in previous work~\cite{hao2024bilevel}. In addition, in terms of the dependency on $\epsilon$, our complexity bound is nearly optimal up to logarithmic factors due to the $\Omega(\epsilon^{-4})$ lower bound of nonconvex stochastic optimization for smooth single-level problems~\cite{arjevani2023lower} when there is no mean-squared smooth condition of the stochastic gradient.

\subsubsection{High Probability Guarantees} \label{sec:high-prob}
In this section, we provide a high probability result of our algorithm, which requires the following assumption.

\begin{assumption}\label{ass:highprob}
The estimators used to calculate stochastic gradients and Hessian-vector products (i.e., $\gdx F(x,y;\xi)$,$\gdy F(x,y;\xi)$, $\gdy G(x,y;\xi)$, $\gdxy G(x,y;\zeta)$, $\gdyy G(x,y;\zeta)$) are \emph{unbiased} and satisfy (assume $\lambda>0$):
\begin{small}
\begin{equation*}
    \begin{aligned}
        &\forall \xi,\ \|\gdx F(x,y;\xi)-\gdx f(x,y)\| \leq \sigma_{f,1}, \\
        &\forall \xi,\ \|\gdy F(x,y;\xi)-\gdy f(x,y)\| \leq \sigma_{f,1},\\
        &\Pr\{\|\gdy G(x,y;\xi)-\gdy g(x,y)\| \geq \lambda\} \leq 2\exp\left(-2\lambda^2/\sigma_{g,1}^2\right), \\
        &\forall \zeta,\ \|\gdxy G(x,y;\zeta)-\gdxy g(x,y)\| \leq \sigma_{g,2}, \\
        &\forall \zeta,z, \ \|(\gdyy G(x,y;\zeta)-\gdyy g(x,y))z\| \leq \sigma_z. \\
    \end{aligned}
\end{equation*}
\end{small}%
\end{assumption}

\textbf{Remark}: Assumption~\ref{ass:highprob} is a technical assumption to establish the high probability result. It assumes that the estimators either have almost sure bounded noise or light-tailed noise. Similar assumptions have been made in the literature on optimization for relaxed smooth functions~\cite{zhang2019gradient,zhang2020improved,liu2023nearlyoptimal} and strongly convex functions~\cite{cutler2023stochastic}. A more in-depth discussion is included in Appendix~\ref{sec:appendix-justify}. Our setting is more challenging because their goal is to optimize a single-level relaxed smooth function, while our work is for bilevel optimization under unbounded smoothness. 
\begin{theorem}\label{thm:highprob}
Suppose Assumptions~\ref{ass:relax-smooth} and~\ref{ass:highprob} hold. Let $\{x_t\}$ be the iterates produced by Algorithm~\ref{alg:bilevel}. For any given $\delta\in(0,1)$ and sufficiently small $\eps$ (see exact choice of $\eps$ in \eqref{eq:eps-high-prob}), if we choose $\gamma = \frac{16}{\mu}(1-\beta)$, and the same $\alpha^{\init}, \alpha,\beta,\eta, T_0$ as in Theorem~\ref{thm:main-thm}, then with probability at least $1-4\delta$ over the randomness in $\mathcal{F}_T$, Algorithm~\ref{alg:bilevel} guarantees $\frac{1}{T}\sum_{t=0}^{T}\|\gdphi(x_t)\| \leq \eps$ with at most $T = \frac{4\Delta_0}{\eta\eps}$ iterations. 
\end{theorem}

\textbf{Remark:} Theorem~\ref{thm:highprob} establishes a high probability result of bilevel optimization under unbounded smoothness. Note that the choice of $\eta$ is the same as in Theorem~\ref{thm:main-thm} (i.e., $\eta=\widetilde{\Theta}(\epsilon^3)$), therefore we can get $\widetilde{O}(\epsilon^{-4})$ iteration complexity to find an $\epsilon$-stationary point with high probability.  To the best our knowledge, this is the first high probability convergence guarantee for stochastic bilevel optimization. 

\subsection{Proof Sketch}
In this section, we mainly provide a proof sketch of Theorem~\ref{thm:main-thm} and briefly discuss the high probability proof of Theorem~\ref{thm:highprob}. The detailed proofs can be found in Appendix~\ref{sec:in-expectation} and~\ref{sec:high-prob-proof}, respectively. Define $y_t^*=y^*(x_t), z_t^*=z^*(x_t), \Tilde{y}_t^*=y^*(\Tilde{x}_t)$. Similar to the previous work~\cite{hao2024bilevel}, it is difficult for us to handle the hypergradient bias term $\mathbb{E}_{t}[\|\hatphi(x_t,y_t,z_t;\xi_t',\zeta_t')-\nabla \Phi(x_t)\|]$: the upper bound of this quantity depends on $L_{x,1}\|y_t-y_t^*\|\|\gdphi(x_t)\|$ due to Assumption~\ref{ass:relax-smooth}. The work of~\cite{hao2024bilevel} uses a double-loop procedure to ensure that $\|y_t-y_t^*\|$ is very small (that is, the same order as $\epsilon$) for every $t$ with high probability, which is too demanding and cannot hold for our proposed single-loop algorithm. 

To address this issue, the key idea is that we do not require $\|y_t-y_t^*\|$ to be small for every $t$, instead we only need $\|y_t-y_t^*\|$ to be smaller than some constant (i.e, $\frac{1}{8L_1}$) for every $t$ and the weighted average of $\|y_t-y_t^*\|$ over all iterations is smaller than $\epsilon$. In this way, we can also handle the hypergradient bias and establish the convergence. To this end, we introduce the following lemmas. Lemma~\ref{lm:maintext-lm1} is introduced to handle the approximation error of the lower-level variable for any slowly-changing upper-level sequences. Then we apply this lemma to the warm-start stage (line 3 in Algorithm~\ref{alg:bilevel}) and the stage of simultaneous updates for lower-level and upper-level variables (line 4$\sim$10), which ends up with Lemma~\ref{lm:maintext-lm2} and Lemma~\ref{lm:maintext-lm3} respectively. 

\begin{lemma} \label{lm:maintext-lm1}
Suppose Assumption~\ref{ass:f-and-g} holds, let $\{\tilde{y}_t\}$ be the iterates produced by Algorithm~\ref{alg:sgd} with any fixed input sequence $\{\tilde{x}_t\}$ such that $\|\tilde{x}_{t+1}-\tilde{x}_t\|\leq R$ for all $t\geq0$, and constant learning rate $\alpha\leq 1/(2l_{g,1})$. Then for any fixed $t\in[N]$ and $\delta\in(0,1)$, the following estimate holds with probability at least $1-\delta$ over the randomness in $\widetilde{\gF}_t^1$:
\begin{small}
\begin{equation*}
    \|\tilde{y}_t-\tilde{y}_t^*\|^2 \leq \left(1-\frac{\mu\alpha}{2}\right)^t\|\tilde{y}_0-\tilde{y}_0^*\|^2 + \left[\frac{8\alpha \sigma_{g,1}^2}{\mu} + \frac{4R^2l_{g,1}^2}{\mu^4\alpha^2}\right]\log\frac{e}{\delta},
\end{equation*}
\end{small}%
where $e$ denotes the base of natural logarithms. As a consequence, for all $t\in[N]$ and $\delta\in(0,1)$, the following holds with probability at least $1-\delta$ over the randomness in $\widetilde{\gF}_{N}^1$:
\begin{small}
\begin{equation*}
    \|\tilde{y}_t-\tilde{y}_t^*\|^2 \leq \left(1-\frac{\mu\alpha}{2}\right)^t\|\tilde{y}_0-\tilde{y}_0^*\|^2 + \left[\frac{8\alpha \sigma_{g,1}^2}{\mu} + \frac{4R^2l_{g,1}^2}{\mu^4\alpha^2}\right]\log\frac{eN}{\delta}.
\end{equation*}
\end{small}%
\end{lemma}

\textbf{Remark:} Lemma~\ref{lm:maintext-lm1} establishes the error of the lower-level problem at every iteration for \emph{any fixed slowly changing upper-level sequence $\{\tilde{x}_t\}$} with high probability. This lemma is a generalization of the techniques of stochastic optimization under distribution drift (e.g., Theorem 6 in~\cite{cutler2023stochastic}). It will be applied to two stages of our algorithm (warm-start stage in line 3, and simultaneous update stage in line 4$\sim$10) to control the lower-level error. 


\begin{lemma}[Warm-start] \label{lm:maintext-lm2}
Suppose Assumption~\ref{ass:f-and-g} holds and given any $\delta\in(0,1)$, let $\{y_t^{\init}\}$ be the iterates produced by Algorithm~\ref{alg:sgd} starting from $y_0^{\init}$ with $R=0$ (where $R$ is defined in Lemma~\ref{lm:maintext-lm1}, it means that $\tilde{x}_t=x_0$ for any $t$). Under the same choice of learning rate $\alpha^{\init}$ as in Theorem~\ref{thm:main-thm} and run Algorithm~\ref{alg:sgd} for $T_0=\frac{\log\left(256L_1^2\|y_0^{\init}-y_0^*\|^2\right)}{\log\left(2/(2-\mu\alpha^{\init})\right)} = \widetilde{O}(1)$ iterations,
Algorithm~\ref{alg:sgd} guarantees $\|y_{T_0}^{\init}-y_0^*\| \leq 1/(8\sqrt{2}L_1)$ with probability at least $1-\delta$ over the randomness in $\widetilde{\mathcal{F}}_{T_0}^{1}$ (we denote this event as $\gE_{\init}$). 
\end{lemma}

\textbf{Remark:} Lemma~\ref{lm:maintext-lm2} shows that it requires at most a logarithmic number of iterations in the warm-start stage to get constant error of the lower-level variable, with high probability. This lemma is an application of Lemma~\ref{lm:maintext-lm1} with $R=0$ since the upper-level variable is fixed to be $x_0$.

\begin{lemma} \label{lm:maintext-lm3}
Under assumptions~\ref{ass:relax-smooth},~\ref{ass:f-and-g} and event $\gE_{\init}$, for any given $\delta\in(0,1)$ and sufficiently small $\eps$ (see exact choice of $\eps$ in \eqref{eq:eps}), under the same parameter choice as in Theorem~\ref{thm:main-thm}
and run Algorithm~\ref{alg:bilevel} for $T=\frac{4\Delta_0}{\eta\eps}$ iterations. Then for all $t\in[T]$, Algorithm~\ref{alg:bilevel} guarantees with probability at least $1-\delta$  over the randomness in $\mathcal{F}_T^1$ (we denote this event as $\gE_y$) that: 
\begin{enumerate}[(i)]
    \item $\|y_t-y_t^*\|\leq \frac{1}{8L_1}$,
    \item $\frac{1}{T}(1-\beta)\sum_{t=0}^{T-1}\sum_{i=0}^{t}\beta^{t-i}\|y_i-y_i^*\| \leq \frac{3}{32L_0}\eps$,
\end{enumerate}
where $\Delta_0$ is defined in \eqref{eq:ABdelta}.
\end{lemma}

\textbf{Remark}: Lemma~\ref{lm:maintext-lm3} provides a refined characterization of the lower-level error during the simultaneous update stage (line 4$\sim$10 in Algorithm~\ref{alg:bilevel}): the error of every iterate of $y$ is bounded by a constant and the weighted error of iterates of $y$ is small. Another important aspect of this lemma is that the statement holds with high probability over $\mathcal{F}_T^1$, which is \emph{independent of the randomness in $x$ and $z$} since $\mathcal{F}_T^1$ is independent of $\mathcal{F}_t^2$ and $\mathcal{F}_t^3$. This nice property is crucial for our subsequent analysis of the sequence $\{x_t\}$ and $\{z_t\}$ without worrying about the dependency issue on filtrations.


\begin{lemma} \label{lm:maintext-lm4}
Under Assumptions~\ref{ass:relax-smooth},~\ref{ass:f-and-g} and events $\gE_{\init} \cap \gE_y$, define $\eps_t = m_{t+1} - \gdphi(x_t)$ as the moving-average hypergradient estimation error. Then we have
\begin{small}
\begin{equation*}
    \begin{aligned}
        &\E\left[\sum_{t=0}^{T-1}\|\eps_t\|\right]
        \leq \left(\frac{\eta L_1\beta}{1-\beta} + \frac{1}{8}\right)\sum_{t=0}^{T-1}\E\|\gdphi(x_t)\| + O\left(\frac{T\eta L_0\beta}{1-\beta} \right. \\
        &\left. + T\sqrt{1-\beta}\sqrt{\sigma_{f,1}^2 + \frac{2l_{f,0}^2}{\mu^2}\sigma_{g,2}^2} + l_{g,1}\sqrt{T}\sqrt{\sum_{t=0}^{T-1}\E[\|z_t-z_t^*\|^2]} \right.\\
        &\left. + \frac{\|\gdphi(x_0)\|+L_0\Delta_{y,0}}{1-\beta} + L_0T\sqrt{\left(\frac{8\alpha \sigma_{g,1}^2}{\mu} + \frac{4\eta^2l_{g,1}^2}{\mu^4\alpha^2}\right)\log\frac{eT}{\delta}} \right)
    \end{aligned}
\end{equation*}
\end{small}%
where $\Delta_{y,0}$ is defined in \eqref{eq:ABdelta}, and the expectation is taken over the randomness in $\widetilde{\gF}_T$.
\end{lemma}

\textbf{Remark:} This lemma provides an upper bound for the cumulative hypergradient estimation error over $T$ iterations. Under the parameter setting of Theorem~\ref{thm:main-thm}, we can see that the RHS can be divided into two parts. The first part is the summation of history gradient norm, which can be dominated by a negative gradient term in the descent lemma. The second part consists of several error terms that grow sublinearly in terms of $T$ (because the averaged expected squared error for variable $z$ can be shown to grow sublinearly in $T$ as well). The fact that these error terms vanish during the optimization process is crucial to establish the convergence result. 



\textbf{Proof Sketch of High Probability Guarantees in Theorem~\ref{thm:highprob}}. The full proof is included in Appendix~\ref{sec:high-prob-proof} and we provide the roadmap of the proof here. The main difference between high probability guarantees and expectation guarantees is that we also need to provide a high probability analysis for variables $z$ and $x$. The idea of the proof is to build on the high probability result in $y$ and gradually establish high probability results for all variables. In particular, from Lemma~\ref{lm:maintext-lm2}, we know that the event $\gE_{\init}$ happens with probability $1-\delta$. Then from Lemma~\ref{lm:maintext-lm3}, we know that with probability at least $1-\delta$, we have $\gE_{y}$ happening, provided that $\gE_{\init}$ happens. The next step is to show under events $\gE_{y}$ and $\gE_{\init}$, we have a nice bound for $z$ with probability $1-\delta$ (this is proved in Lemma~\ref{lm:z-two-bound} in Appendix~\ref{sec:high-prob-proof}, and this event is denoted as $\gE_{z}$). Then under $\gE_{\init}$, $\gE_{y}$ and $\gE_{z}$, we can obtain good hypergradient error with probability $1-\delta$ (this is proved in Lemma~\ref{lm:MDS-var} in Appendix~\ref{sec:high-prob-proof}, and this event is denoted as $\gE_x$) and derive the final convergence result. Therefore, by the rule of conditional probability, we show that the good event (i.e., $\gE_x \cap \gE_y \cap \gE_z \cap \gE_{\init}$) happens with probability $(1-\delta)^4 \geq 1-4\delta$ for $\delta\in(0,1)$.



\section{Experiments}

\begin{figure*}[t]
\begin{center}
\subfigure[Training accuracy on SNLI]{\includegraphics[width=0.24\linewidth]{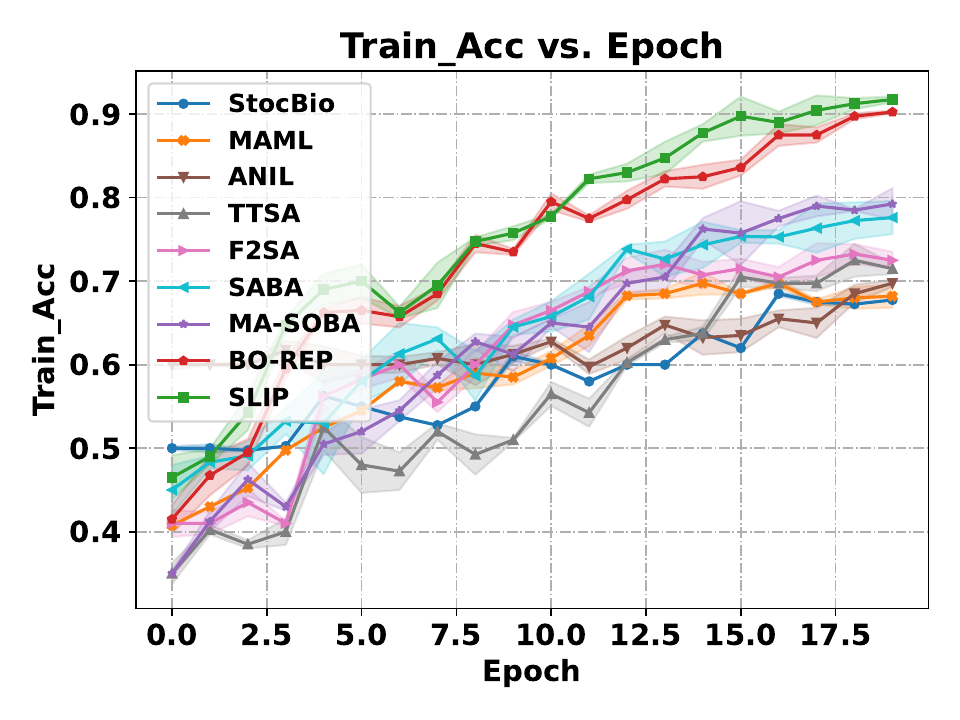}}   \   
\subfigure[Test accuracy on SNLI]{\includegraphics[width=0.24\linewidth]{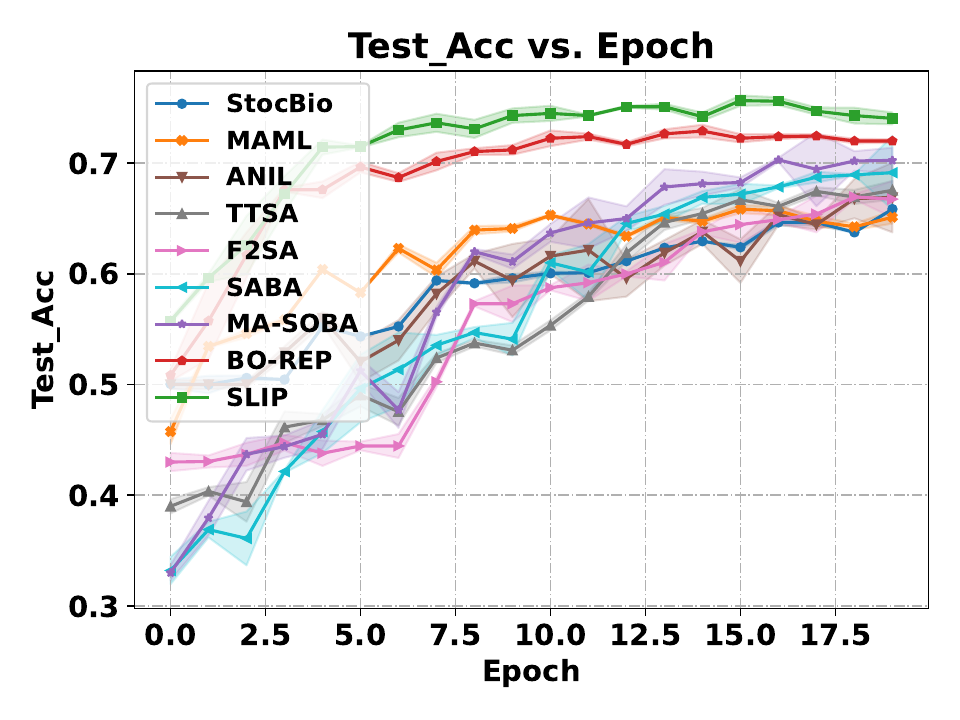}} \ 
\subfigure[Training accuracy on ARD]{\includegraphics[width=0.24\linewidth]{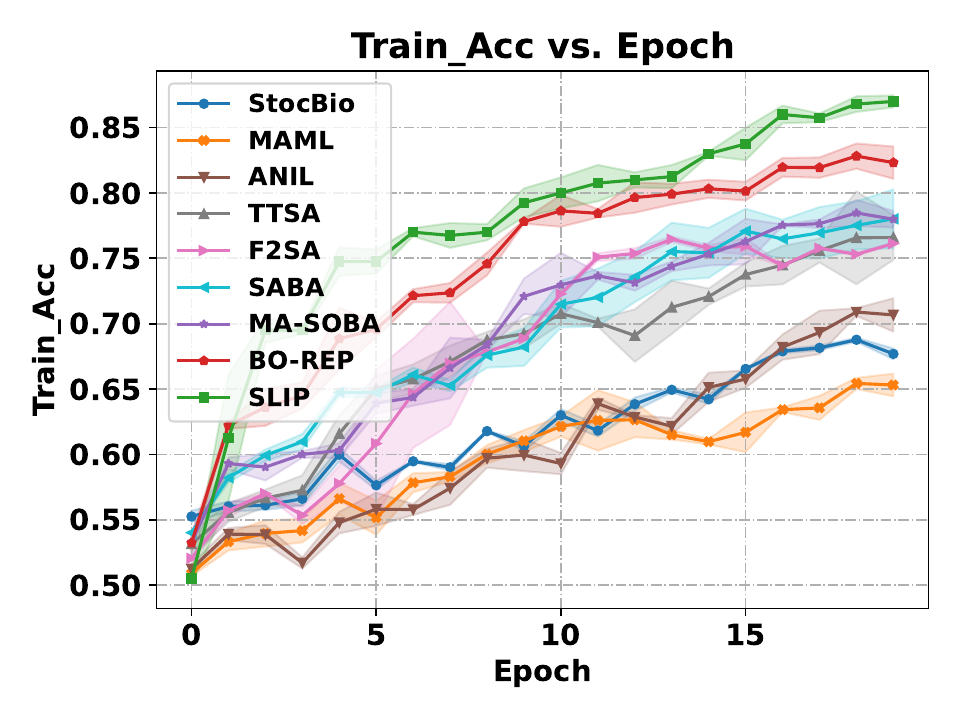}} \
\subfigure[Test accuracy on ARD]{\includegraphics[width=0.24\linewidth]{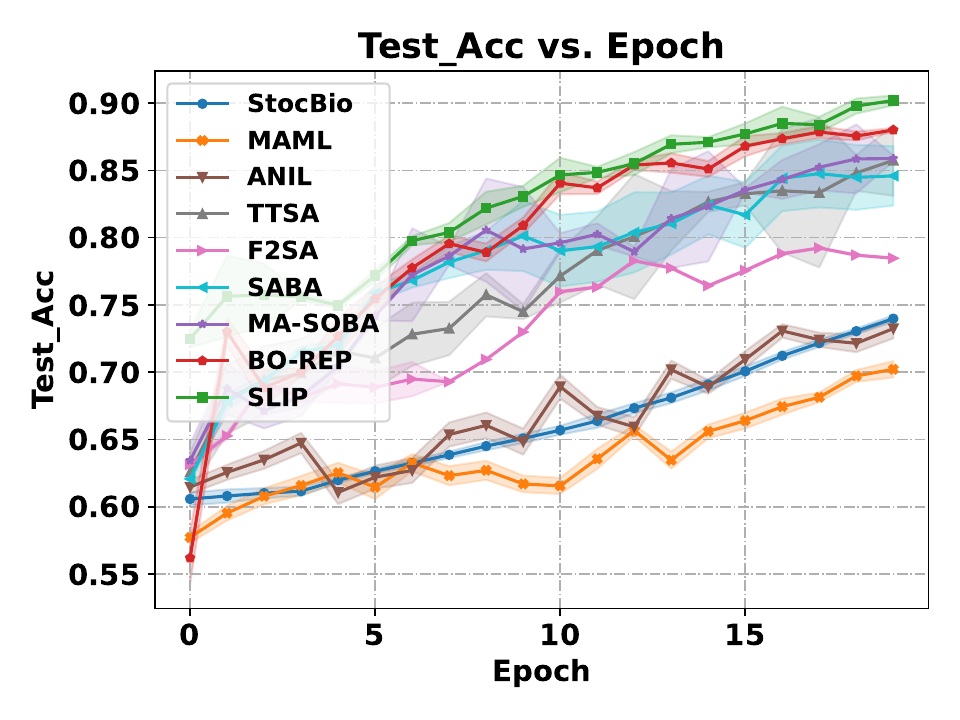}}
\end{center}
\caption{Comparison with bilevel optimization baselines on Hyper-representation. Figure (a) and (b) are the results in the SNLI dataset. Figures (c) and (d) are the results of the Amazon Review Dataset (ARD). }
\label{fig:acc_HR}
\end{figure*}

\subsection{Hyper-representation Learning} \label{sec:HR}
In this section, we conduct experiments on an important application of bilevel optimization with unbounded smooth upper-level function: hyper-representation learning for text classification (i.e., meta-learning). The goal of hyper-representation learning is to try to find a good model representation parameterized by $x$, such that it can quickly adapt to new tasks by quickly tuning the task-specific parameter $y_i$ by a few steps of gradient updates. The learning procedure can be formally characterized by bilevel optimization~\cite{ji2021bilevel, hao2024bilevel}.  We define a sequence of $K$ tasks, which consists of the training set $\{\mathcal{D}_i^{tr}\;|\; i=1, ..., K\}$ and validation set $\{\mathcal{D}_i^{val}\;|\; i=1, ..., K\}$. The loss function of the model on a uniformly sampled data set $\xi_i\sim D^{tr}_i$ can be denoted as $\mathcal{L}(x, y_i; \xi_i)$. 

From the perspective of bilevel optimization, the lower level function aims to find an optimal task-specific parameter $y_i^*$ on the training set $\mathcal{D}_{i}^{tr}$, given the meta parameter $x$. The upper-level function evaluates each $y_i^*, i=1,...,K$ on the corresponding validation set $\mathcal{D}_{i}^{val}$ and leverages all the gradient information to update the meta-parameter $x$. If we denote $y=(y_1, y_2,..., y_K)$, we can formalize this problem as follows:
\begin{equation}\label{eq:hr}
    \begin{aligned}
        & \min_{x}\frac{1}{K}\sum_{i=1}^{K}\frac{1}{|\mathcal{D}^{val}_i|}\sum_{\xi \in \mathcal{D}_{i}^{val}}\mathcal{L}(x, y^*(x); \xi), \text{ s.t.,}\; \\
        &  \ y^*(x) = \argmin_{y}\frac{1}{K}\sum_{i=1}^{K}\frac{1}{|\mathcal{D}_{i}^{tr}|}\sum_{\zeta\in \mathcal{D}^{tr}_i}\mathcal{L}(x, y_i; \zeta) + \frac{\mu}{2}\|y_i\|^2,
    \end{aligned}
\end{equation}

where the lower-level function contains a regularizer $\frac{\mu}{2}\|y_i\|^2$, which ensures that the lower-level function is strongly convex. In practice, a fully-connected layer is widely used as a classifier $y$ for the meta-learning model~\cite{bertinetto2018meta, ji2021bilevel}, and a different classifier $y_i$ will be used for a specific task. The remaining layers consist of recurrent neural layers, which are used to learn from natural language data. Therefore, the lower-level function satisfies the $\mu$-strongly convex assumption, and the upper-level function satisfies the nonconvex and unbounded smoothness assumptions (i.e., Assumption~\ref{ass:relax-smooth}). 

We performed the hyper-representation learning experiment for text classification task on Stanford Natural Language Inference (SNLI)~\cite{bowman2015large} and Amazon Review Dataset (ARD)~\cite{blitzer2006domain} datasets. SNLI is a naturalistic corpus of 570k pairs of sentences labeled ``entailment", ``contradiction", and ``independence". We construct $K=25$ tasks, where $\mathcal{D}_i^{tr}$ and $\mathcal{D}_i^{val}$ randomly sample two disjoint categories from the original data, respectively. ARD provides positive and negative customer reviews for 25 types of products. The following experimental setup is from~\cite{hao2024bilevel}: we choose three types (i.e., office products, automotive, and computer video games) as the validation set and the remaining types as the training set. The number of samples in training set and validation set keeps the same. In training phase, a good hyper-representation $x$ is obtained, and will be used in the test phase for performance evaluation. Note that $x$ is fixed during the test phase, and only task-specific parameter $y$ is fine-tuned for a quick adaptation.

We compare our algorithm SLIP with other baselines, including typical meta-learning algorithms MAML~\cite{rajeswaran2019meta} and ANIL~\cite{raghu2019rapid}, bilevel optimization method: StocBio~\cite{ji2021bilevel}, TTSA~\cite{hong2023two}, F$^2$SA~\cite{kwon2023fully}, SABA~\cite{dagreou2022framework}, MA-SOBA~\cite{chen2023bilevel1}, and BO-REP~\cite{hao2024bilevel}. For a fair comparison, we use a 2-layer recurrent neural network as the representation layers with input dimension=300, hidden dimension=4096 and output dimension=512.
A fully-connected layer is used as the classifier with its output dimention set to 2 on SNLI and 3 on ARD.

We run every method for $20$ epochs, with minibatch size $50$ for both training and validation set. Within each epoch, we run $500$ iterations of our algorithm to traverse the constructed training and validation task on SNLI, and $400$ iterations on ARD.  
The evolution of the training and test accuracy is shown in Figure~\ref{fig:acc_HR}~\footnote{Here an epoch means a full pass over the validation set, i.e., for upper-level variable $x$ update. For a more comprehensive comparison, we re-conduct experiments where an epoch means a full pass over the training set, i.e., for lower-level variable $y$ update. In this case there are fewer $x$ updates than the previous run due to the warm-start phase. The results show that SLIP is still empirically better than other baselines. See \Cref{sec:epoch-def} for more details.}. We can see that our proposed SLIP algorithm consistently outperforms all other baselines. More details can be found in Appendix~\ref{sec:exp_set_hr}.

\subsection{Data Hyper-cleaning}

\begin{figure*}[t]
\begin{center}
\subfigure[Training acc with $p=0.2$]{\includegraphics[width=0.245\linewidth]{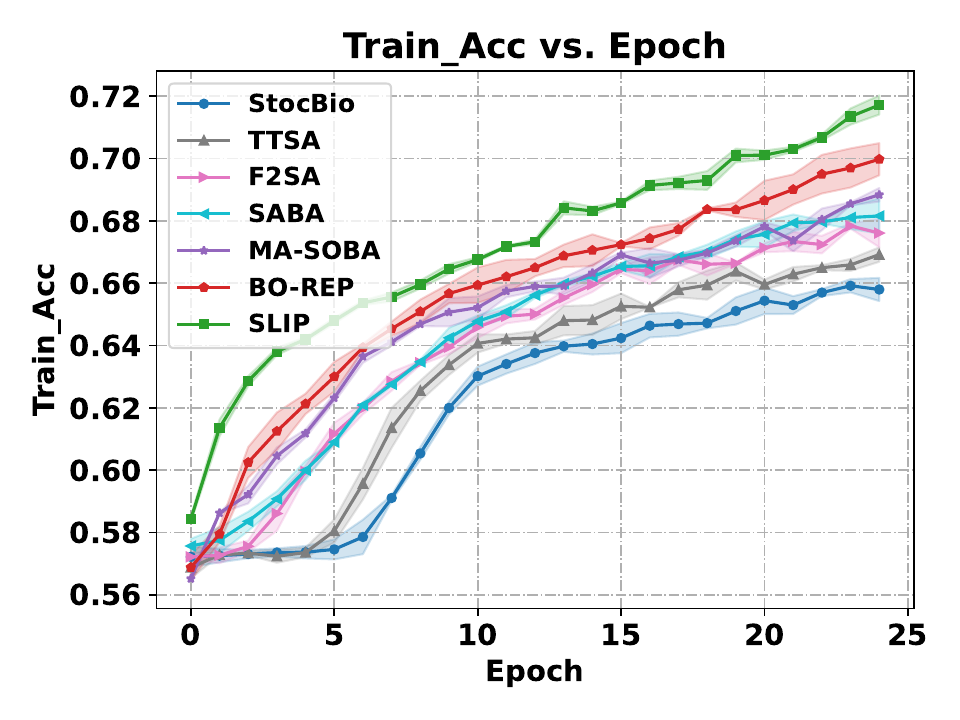}}   
\subfigure[Test acc with $p=0.2$]{\includegraphics[width=0.245\linewidth]{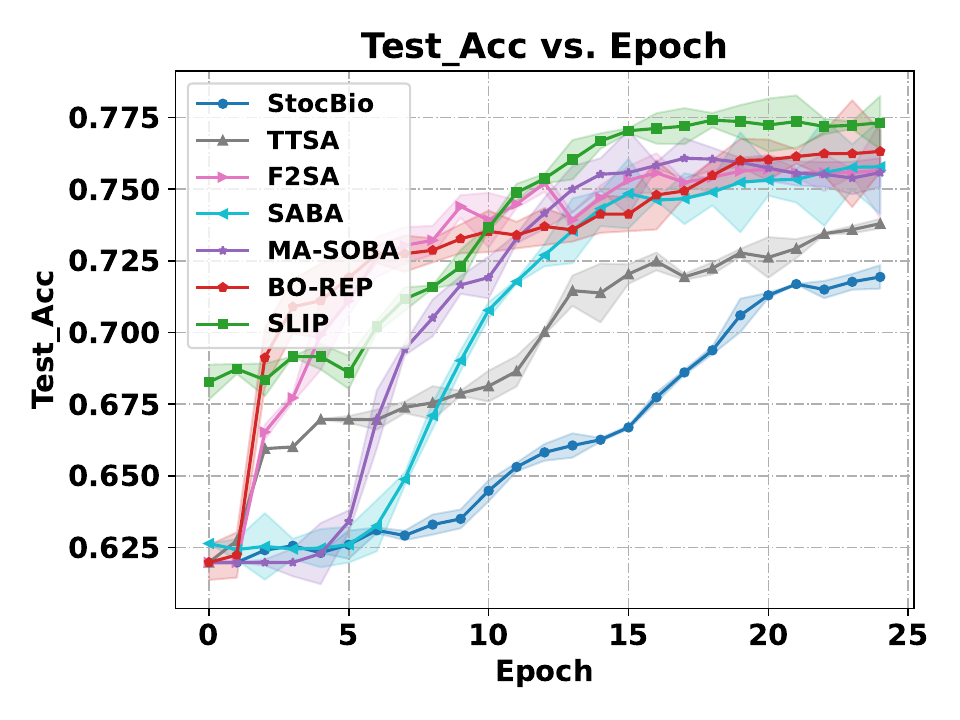}}   
\subfigure[Training acc with $p=0.4$]{\includegraphics[width=0.245\linewidth]{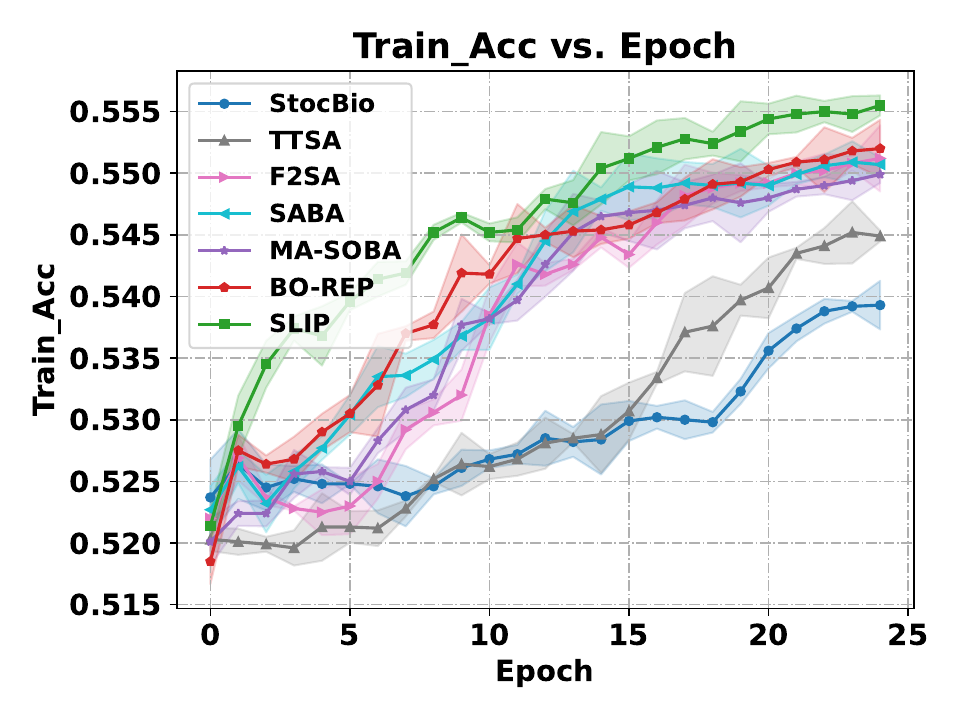}}  
\subfigure[Test acc with $p=0.4$]{\includegraphics[width=0.245\linewidth]{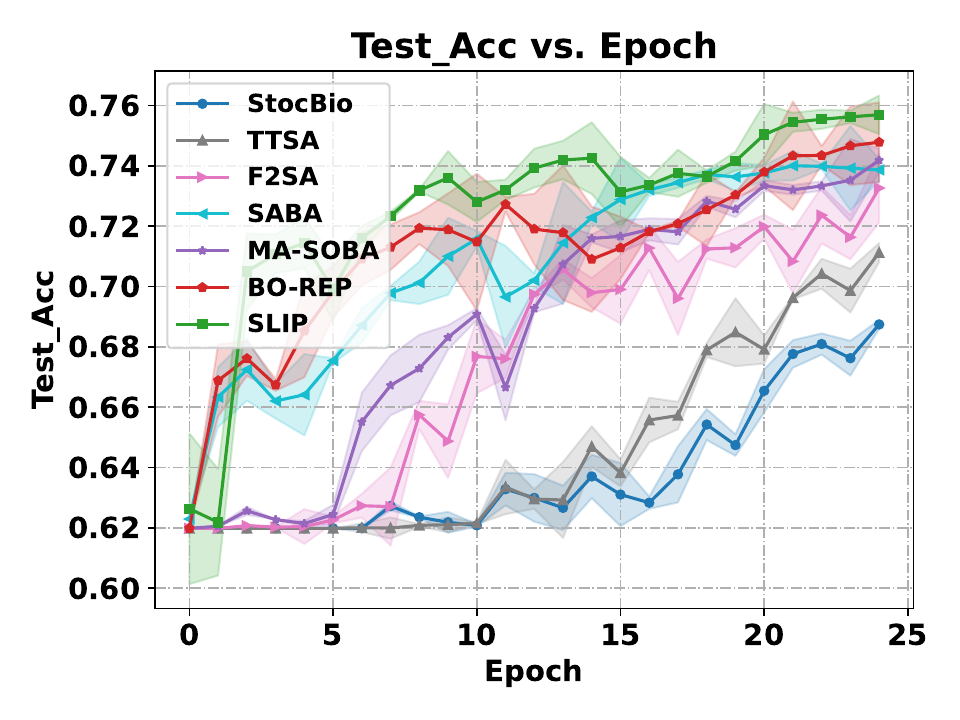}}
\end{center}
\caption{Comparison with bilevel optimization baselinses on data hyper-cleaning. Figure (a), (b) are the results with the corruption rate $p=0.2$. Figure (c), (d) are the results with the corruption rate $p=0.4$. }
\label{fig:acc_HC}
\end{figure*}

It is common that the data can be corrupted by noise, which presents a challenge to train a model at a certain level of noise. The data hyper-cleaning task~\cite{shaban2019truncated} aims to train a model on a corrupted set. An important approach to solving this problem is to learn a weight for each individual sample so that the weights associated with the noisy data will be assigned a small magnitude value by training. We typically solve the data hyper-cleaning task by bilevel optimization. Given an initial weight $x_i\in\mathbb{R}$ for each sample, a DNN model with parameter $y\in\mathbb{R}^{d}$, loss function $\mathcal{L}(\cdot)$, a corrupted training set $\mathcal{D}^{tr}$ and a non-corrupted validation set $\mathcal{D}^{val}$,  the model is trained on the weighted set $\mathcal{D}^{tr}$ and tries to achieve good performance on the non-corrupted validation set $\mathcal{D}^{val}$. Formally, this bilevel optimization can be formulated as
    \begin{equation}
        \begin{aligned}
            & \min_{x} \frac{1}{|\mathcal{D}^{val}|}\sum_{\xi\in \mathcal{D}^{val}}\mathcal{L}(y^*(x); \xi),\\ 
            & s.t., \ y^*(x)=\argmin_y \frac{1}{|\mathcal{D}^{tr}|}\sum_{\zeta_i \in \mathcal{D}^{tr}}\sigma(x_i)\mathcal{L}(y; \zeta_i) + \lambda \|y\|^2,
        \end{aligned}
    \end{equation}
where the activation function $\sigma(z) = \frac{1}{1+e^{-z}}$ and the parameter $\lambda>0$ is the regularizer factor. 

We performed the data hyper-cleaning experiment on Sentiment140~\cite{go2009twitter} for text classification. The Sentiment140 dataset provides a large corpus of tweets that have been automatically labeled with sentiment (positive or negative) based on the presence of emoticons. This data set is commonly used as noisy labels for sentiment classification. We preprocess the original training set by randomly sampling a certain proportion $p$ of data labels and flipping them, where $p$ is called the corruption rate. In this paper, we set the corruption rate to 0.2 and 0.4.

For all baselines, we adopt a 2-layer recurrent neural network which is the same as that mentioned in Section~\ref{sec:HR} with the input dimension = 300 and the output dimension=2. The sample weight $x_i$ is uniformly initialized to 1.0. We compare SLIP with other bilevel optimization algorithms, including StocBio~\cite{ji2021bilevel}, TTSA~\cite{hong2023two}, F$^2$SA~\cite{kwon2023fully}, SABA~\cite{dagreou2022framework}, MA-SOBA~\cite{chen2023bilevel1}, and BO-REP~\cite{hao2024bilevel}. We ran all the methods for 25 epochs, and it required $63$ iterations within each epoch to traverse the training and validation set with minibatch size = 512. 

The experimental results are presented in Figure~\ref{fig:acc_HC}, where (a) and (b) show the results with the corruption rate $p=0.2$ and Figure (c) and (d) show the results with the corruption rate $p=0.4$. Our SLIP outperforms all the baselines consistently, which demonstrates the effectiveness and superiority of our algorithm in data hyper-cleaning compared with others. More experimental details can be found in Appendix~\ref{sec:exp_set_hc}.


\begin{figure*}[t]
\begin{center}
\subfigure{\includegraphics[width=0.24\linewidth]{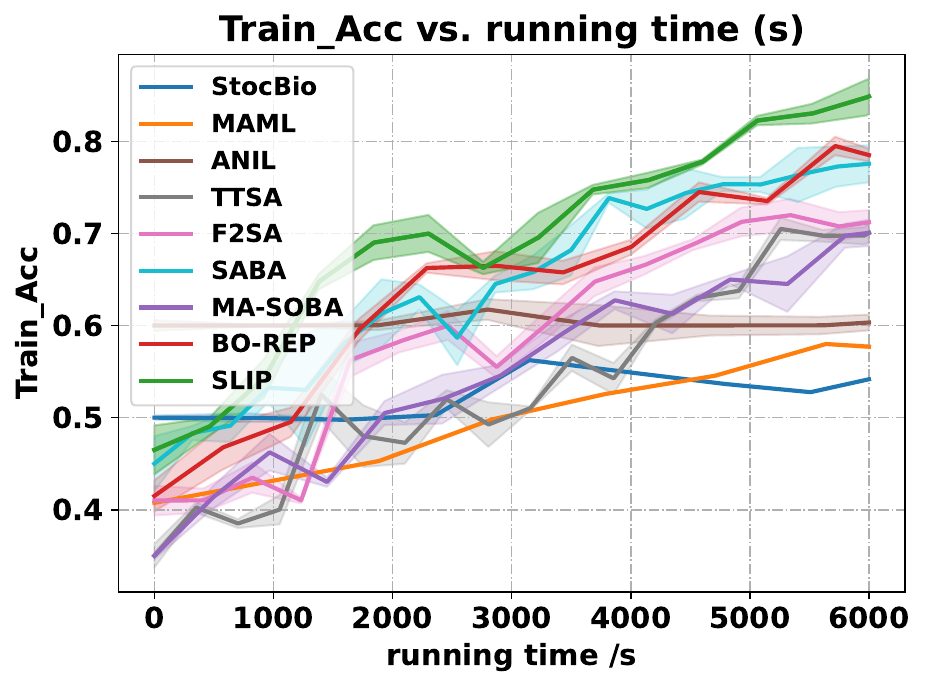}} \ \subfigure{\includegraphics[width=0.24\linewidth]{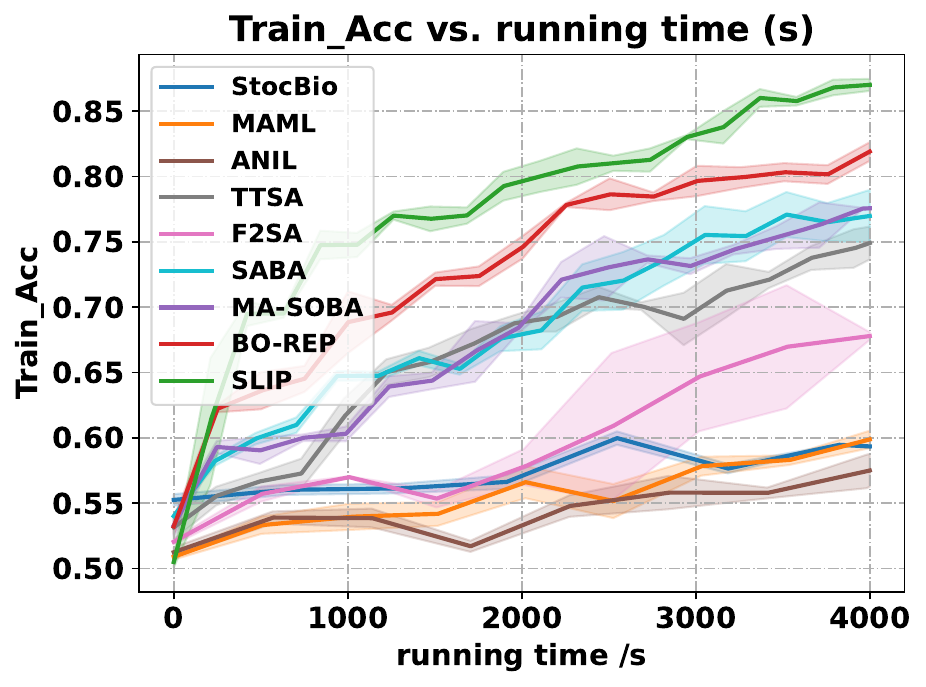}} \  
\subfigure{\includegraphics[width=0.24\linewidth]{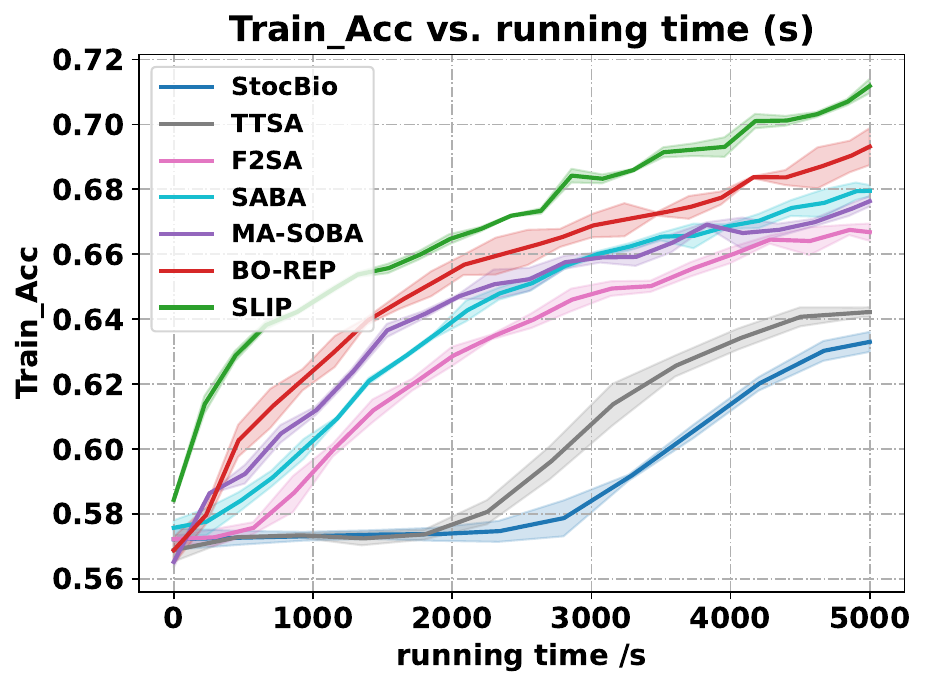}} \ 
\subfigure{\includegraphics[width=0.24\linewidth]{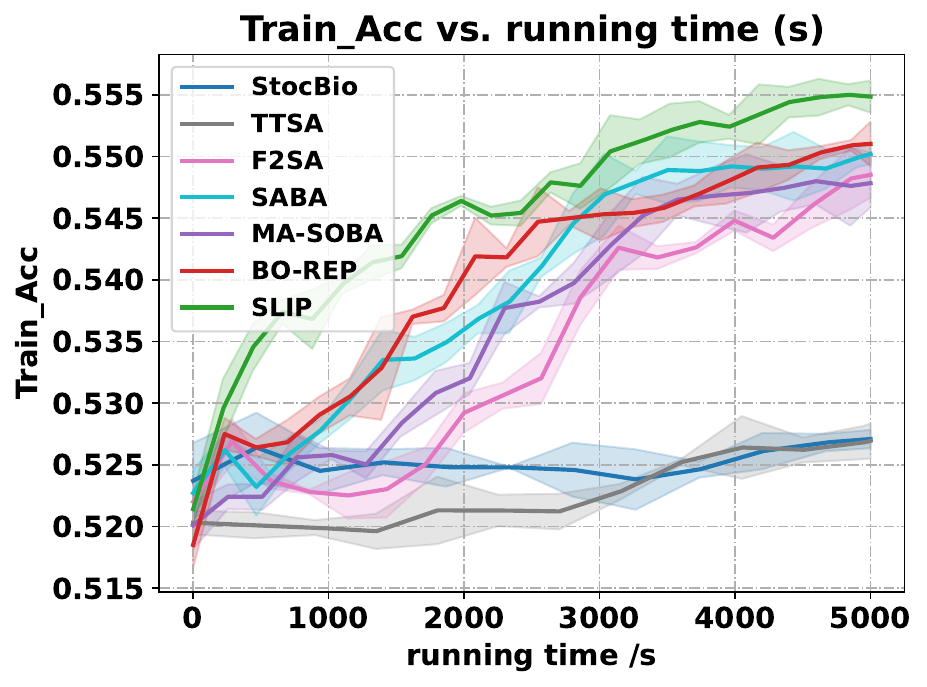}} \ 
\setcounter{subfigure}{0} 
\subfigure[Hyper-representation (SNLI)]{\includegraphics[width=0.24\linewidth]{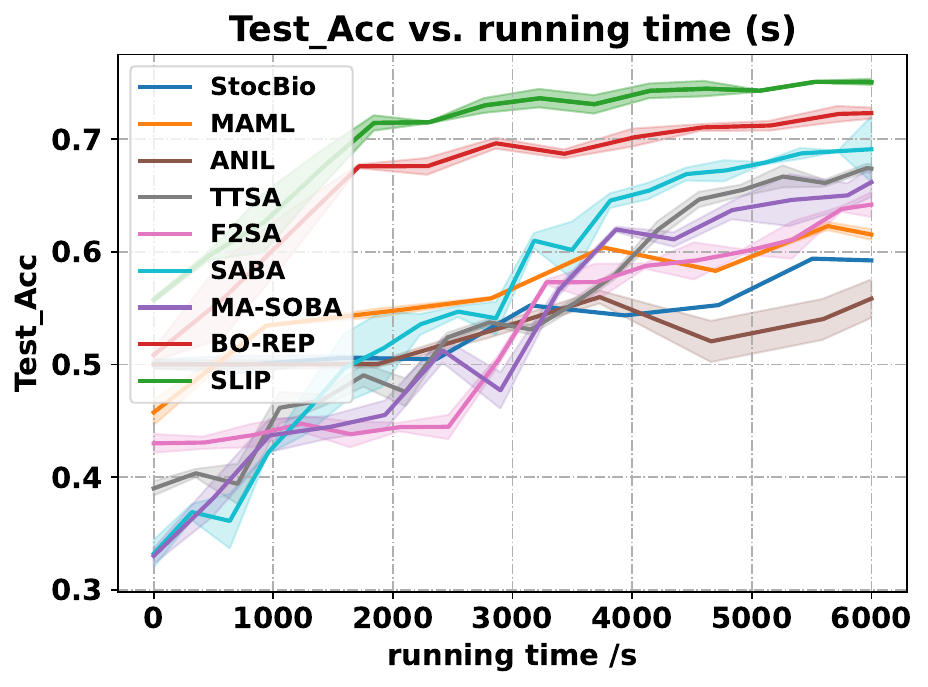}} \ 
\subfigure[Hyper-representation (ARD)]{\includegraphics[width=0.24\linewidth]{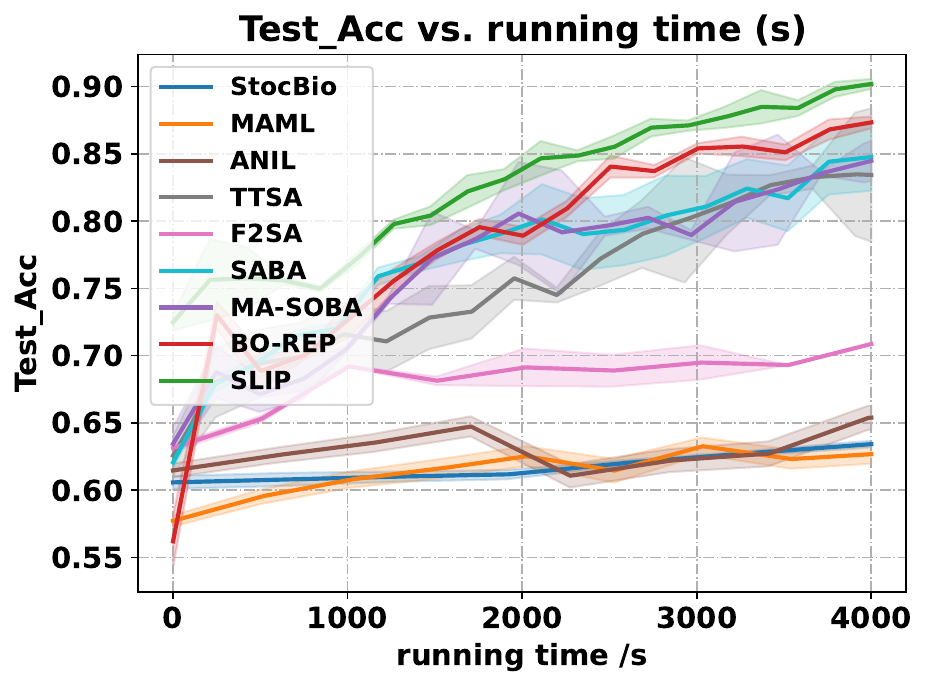}} \ 
\subfigure[Data Hyper-cleaning ($p$=0.2)]{\includegraphics[width=0.24\linewidth]{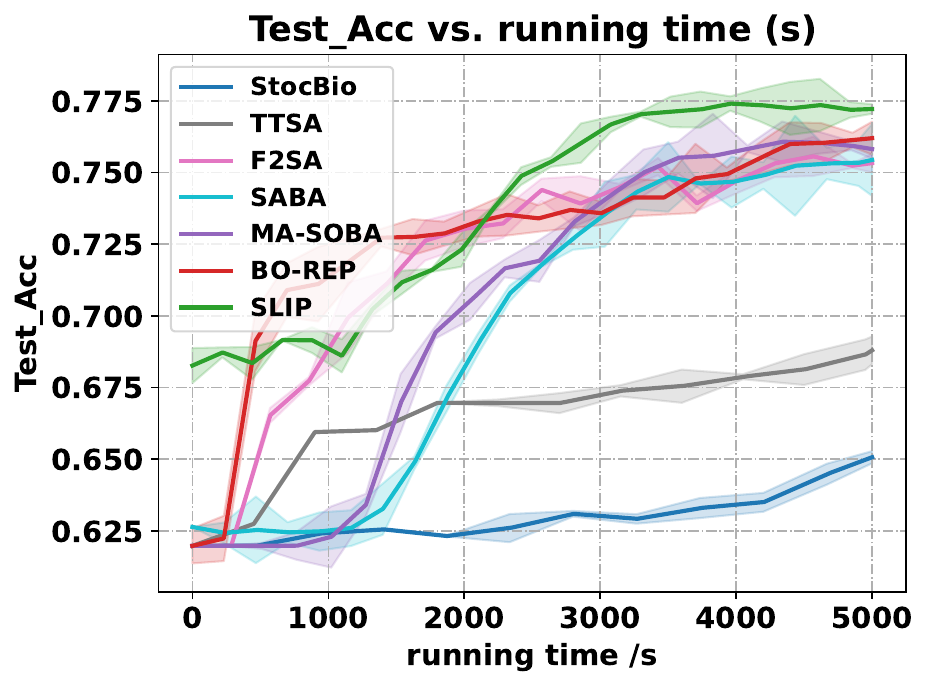}} \ 
\subfigure[Data Hyper-cleaning ($p$=0.4)]{\includegraphics[width=0.24\linewidth]{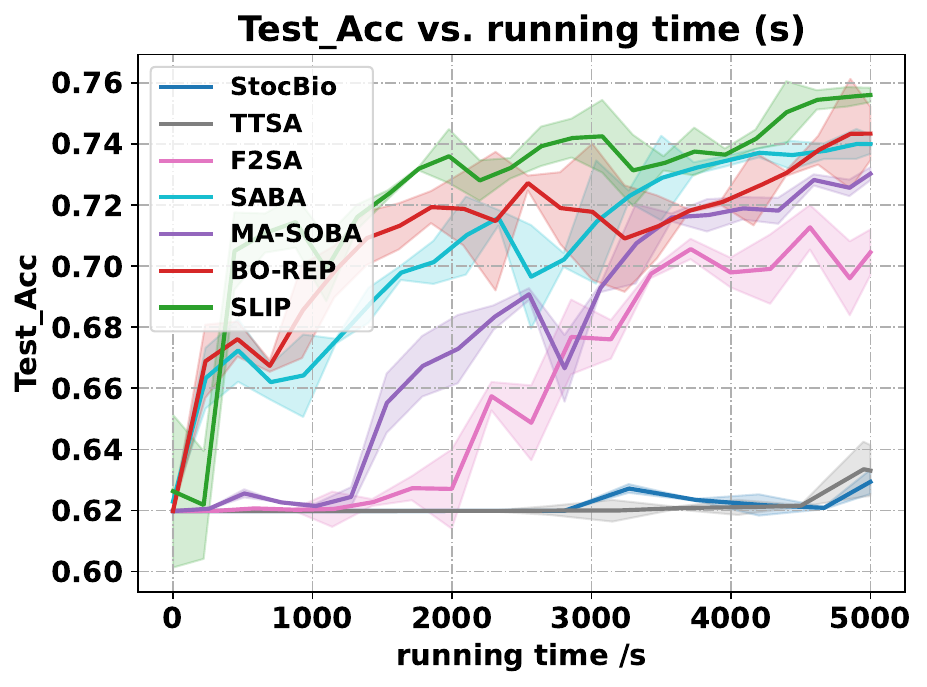}} \ 
\end{center}
\caption{Comparison on running time. (a) Results of Hyper-representation on SNLI dataset. (b) Results of Hyper-representation on Amazon Review Dataset (ARD). (c),  (d) Results of data Hyper-cleaning on Sentiment140 with corruption rate $p=0.2$ and $p=0.4$.}
\label{fig:acc_running_time}
\end{figure*}

\subsection{Running Time Comparison}
To evaluate the practical efficiency of various bilevel algorithms, we follow~\cite{ji2021bilevel, dagreou2022framework} to compare the performance of the baselines over the running time. Each algorithm runs separately on a single device with a NVIDIA RTX 6000 graphics card and an AMD EPYC 7513 32-Core Processor, while the training/test accuracy of the corresponding time is recorded. We show the results of accuracy versus running time in Figure~\ref{fig:acc_running_time}. In particular, Figures (a), (b) show the experimental results of hyper-representation learning on SNLI and ARD datasets, respectively. Figures (c), (d) show the results of data hyper-cleaning in the Sentiment140 data set, where the corruption rate is $p=0.2$ in (c) and $p=0.4$ in (d). Note that the same time scale is used for one figure for a fair comparison. We can observe that SLIP runs faster than other baselines in all the figures, and BO-REP follows. One interesting observation is that in the hyper-representation learning experiment, SLIP runs significantly faster than BO-REP~\cite{hao2024bilevel} (e.g., in Figure~\ref{fig:acc_running_time} (a), (b)). The reason is due to the single-loop nature of SLIP algorithm: when calculating the gradient of the previous layers' parameter (i.e., $x$ in~\eqref{eq:hr}), the gradient of the last layer's parameter (i.e., $y$ in \eqref{eq:hr}) is automatically calculated in the same backpropagation so that we do not need to recalculate it again for updating $y$, so it saves running time compared to BO-REP. In contrast, BO-REP needs multiple separate calculations of gradient w.r.t. $y$ during its periodic updates for $y$. 




\section{Conclusion}
In this paper, we studied the problem of stochastic bilevel optimization, where the upper-level function has potential unbounded smoothness and the lower-level problem is strongly convex. We have designed a single loop algorithm with $\widetilde{O}(1/\epsilon^{4})$ oracle calls to find an $\epsilon$-stationary point, where each oracle call invokes a stochastic gradient or a Hessian-vector product.  Our complexity bounds hold both in expectation and with high probability, under different assumptions. Our bound is nearly optimal due to the lower bounds of nonconvex stochastic optimization for single-level problems, when there is no mean squared smoothness of the stochastic gradient oracles~\cite{arjevani2023lower}.  In the future, we plan to further improve the complexity bounds for bilevel optimization under stronger assumptions. For example, one interesting question to investigate is to obtain $O(1/\epsilon^3)$ complexity under the individual relaxed smoothness assumption for the stochastic gradient/Hessian-vector product oracle.


\section*{Acknowledgements}
We would like to thank the anonymous reviewers for their helpful comments. This work has been supported by a grant from George Mason University, the Presidential Scholarship from George Mason University, a ORIEI seed funding from George Mason University, and a Cisco Faculty Research Award. The Computations were run on ARGO, a research computing cluster provided by the Office of Research Computing at George Mason University (URL: \url{https://orc.gmu.edu}). 

\section*{Impact Statement}
This paper presents work whose goal is to advance the field of Machine Learning. There are many potential societal consequences of our work, none which we feel must be specifically highlighted here.

\bibliography{ref}
\bibliographystyle{icml2024}

\newpage
\appendix
\onecolumn

\input{Appendix}




\end{document}

%% file: math_commands.tex
\usepackage{amsmath,amsfonts,bm}

\def\1{\bm{1}}

\def\eps{{\epsilon}}

\newcommand{\hatphi}{\widehat{\nabla}\Phi}

\def\gdphi{{\nabla\Phi}}
\def\gdx{{\nabla_{x}}}
\def\gdy{{\nabla_{y}}}
\def\gdxy{{\nabla_{xx}^2}}
\def\gdxy{{\nabla_{xy}^2}}
\def\gdyy{{\nabla_{yy}^2}}

\def\init{{\mathsf{init}}}
\newenvironment{sequation}{\begin{equation}\small}{\end{equation}}
\newenvironment{sequation*}{\begin{equation*}\small\begin{aligned}}{\end{aligned}\end{equation*}}

\DeclareMathAlphabet{\mathsfit}{\encodingdefault}{\sfdefault}{m}{sl}
\SetMathAlphabet{\mathsfit}{bold}{\encodingdefault}{\sfdefault}{bx}{n}

\def\gA{{\mathcal{A}}}

\def\gD{{\mathcal{D}}}
\def\gE{{\mathcal{E}}}
\def\gF{{\mathcal{F}}}
\def\gG{{\mathcal{G}}}
\def\gH{{\mathcal{H}}}
\def\gI{{\mathcal{I}}}

\def\gS{{\mathcal{S}}}

\newcommand{\E}{\mathbb{E}}

\newcommand{\R}{\mathbb{R}}

\DeclareMathOperator*{\argmin}{arg\,min}

%% file: Appendix.tex
\section{Comparison Table of Stochastic Bilevel Optimization Algorithms}
\label{sec:comparison-table}

\begin{table}[!h]
    \centering
    \caption{Comparison of stochastic bilevel optimization algorithms in the nonconvex-strongly-convex setting under different smoothness assumptions on $f$ and $g$. The oracle complexity stands for the number of oracle calls to stochastic gradients and stochastic Hessian/Jacobian-vector products to find an $\eps$-stationary point. $\mathcal{C}_L^{a,k}$ denotes $a$-times differentiability with Lipschitz $k$-th order derivatives. ``SC'' means ``strongly-convex''. $\widetilde{O}(\cdot)$ compresses logarithmic factors of $1/\eps$ and $1/\delta$, where $\delta\in(0,1)$ denotes the failure probability.}
    \label{tab:comparison}
    \renewcommand{\arraystretch}{0.7}
    \setlength{\tabcolsep}{6pt}
    \resizebox{\textwidth}{!}{
    \begin{tabular}{ccccccc}
    \toprule[2pt]
    Method \tablefootnote{We omit the comparison with variance reduction-based methods (except for SABA \cite{dagreou2022framework}) that may achieve $\widetilde{O}(\eps^{-3})$ complexity under additional mean-squared smoothness assumptions on both upper-level and lower-level problems, e.g., VRBO, MRBO \cite{yang2021provably}; SUSTAIN \cite{khanduri2021near}; SVRB \cite{guo2021randomized}; or under the finite sum setting, e.g., SRBA \cite{dagreou2023lower}.} & Loop Style & Stochastic Setting & Oracle Complexity & Upper-Level $f$ & Lower-Level $g$ & Batch Size  \\
    \midrule[1pt]
    \makecell{BSA \\ \citep{ghadimi2018approximation}} & Double & General expectation & $\widetilde{O}(\eps^{-6})$     & $\mathcal{C}_{L}^{1,1}$      &  SC and $\mathcal{C}_{L}^{2,2}$     & $\widetilde{O}(1)$  \\      \midrule
    \makecell{StocBio \\ \citep{ji2021bilevel}} & Double & General expectation &  $\widetilde{O}(\eps^{-4})$      & $\mathcal{C}_{L}^{1,1}$       & SC and $\mathcal{C}_{L}^{2,2}$      & $\widetilde{O}(\eps^{-2}) $ \\ \midrule
    \makecell{AmIGO \\ \citep{arbel2021amortized}} & Double & General expectation &  $\widetilde{O}(\eps^{-4})$      &  $\mathcal{C}_{L}^{1,1}$      &  SC and $\mathcal{C}_{L}^{2,2} $    & $O(\eps^{-2}) $  \\ \midrule
    \makecell{ALSET \\ \citep{chen2021closing}} & Double / Single \tablefootnote{ALSET can converge without the need for double loops, but at the cost of a worse dependence on $\kappa\coloneqq l_{g,1}/\mu$ in oracle complexity.} & General expectation &  $O(\eps^{-4})$      &  $\mathcal{C}_{L}^{1,1}$     & SC and $\mathcal{C}_{L}^{2,2} $     & $O(1) $ \\ \midrule
    \makecell{TTSA \\ \citep{hong2023two}}  & Single & General expectation &  $\widetilde{O}(\eps^{-5})$     &    $\mathcal{C}_{L}^{1,1}$    &  SC and $\mathcal{C}_{L}^{2,2}  $   & $\widetilde{O}(1) $   \\ \midrule
    \makecell{F$^2$SA \\ \citep{kwon2023fully}}   & Single & General expectation &   $O(\eps^{-7})$     &   $\mathcal{C}_{L}^{1,1}$    &  SC and $\mathcal{C}_{L}^{2,2}$     &$O(1) $  \\ \midrule
    \makecell{SOBA \\ \citep{dagreou2022framework}}  & Single & Finite sum &   $O(\eps^{-4})$      &  $\mathcal{C}_{L}^{2,2}$     &  SC and $\mathcal{C}_{L}^{3,3}$     & $O(1) $ \\ \midrule
    \makecell{SABA \\ \citep{dagreou2022framework}}  & Single & Finite sum &   $O(N^{4/3}\eps^{-2})$\tablefootnote{SABA \cite{dagreou2022framework} studies finite-sum problem and adapts variance reduction technique SAGA \cite{defazio2014saga}, here $N = m+n$ denotes the total number of samples.}      &  $\mathcal{C}_{L}^{2,2}$     &  SC and $\mathcal{C}_{L}^{3,3}$     & $O(1) $ \\ \midrule
    \makecell{MA-SOBA \\ \citep{chen2023optimal}} & Single & General expectation &  $O(\eps^{-4})$       & $\mathcal{C}_{L}^{1,1}$      &  SC and $\mathcal{C}_{L}^{2,2}$     &$O(1)$ \\ \midrule
    \makecell{BO-REP \\ \cite{hao2024bilevel}} & Double & General expectation &  $\widetilde{O}(\eps^{-4})$      &  $(L_{x,0},L_{x,1},L_{y,0},L_{y,1})$-smooth      & SC and $\mathcal{C}_{L}^{2,2}$      &$O(1) $ \\ \midrule
    \rowcolor[rgb]{ .741,  .843,  .933} \Gape[1pt][1pt]{\makecell{SLIP \\ (This work, \Cref{thm:main-thm}; \\ expectation guarantee)}} & Single & General expectation &   $\widetilde{O}(\eps^{-4})$      &  $(L_{x,0},L_{x,1},L_{y,0},L_{y,1})$-smooth      & SC and $\mathcal{C}_{L}^{2,2}$      & $O(1)$ \\
    \rowcolor[rgb]{ .741,  .843,  .933} \Gape[1pt][1pt]{\makecell{SLIP \\ (This work, \Cref{thm:highprob}; \\ high probability guarantee)}} & Single & General expectation &   $\widetilde{O}(\eps^{-4})$      &  $(L_{x,0},L_{x,1},L_{y,0},L_{y,1})$-smooth      & SC and $\mathcal{C}_{L}^{2,2}$      & $O(1)$ \\
    \bottomrule[2pt]
    \end{tabular}}%
\end{table}%

\section{Properties of $(L_{x,0}, L_{x,1}, L_{y,0}, L_{y,1})$-Smoothness Assumption (Assumption~\ref{ass:relax-smooth})}
\subsection{Definitions of Relaxed Smoothness}

The standard relaxed smoothness assumption originally introduced in \cite{zhang2020gradient} is defined in Definition~\ref{def:rs-jingzhao}.

\begin{definition}[\cite{zhang2020gradient}] \label{def:rs-jingzhao}
A twice differentiable function $F$ is $(K_0,K_1)$-smooth if $\|\nabla^2F(w)\|\leq K_0 + K_1\|\nabla F(w)\|$ for any $w$.
\end{definition}

The following Definition~\ref{def:rs-jikai} is an alternative definition for the $(K_0, K_1)$-smoothness. It does not need $f$ to be twice differentiable and is strictly weaker than $L$-smoothness.

\begin{definition}[Remark 2.3 in \cite{zhang2020improved}] \label{def:rs-jikai}
A differentiable function $F$ is $(K_0,K_1)$-smooth if $\|\nabla F(w)-\nabla F(w')\|\leq (K_0+K_1\|\nabla F(w)\|)\|w-w'\|$ for any $\|w-w'\|\leq 1/K_1$.
\end{definition}

\subsection{Relationships between Assumption~\ref{ass:relax-smooth} and Standard Relaxed Smoothness}

The folowing lemma shows that $(L_{x,0}, L_{x,1}, L_{y,0}, L_{y,1})$-smoothness assumption (e.g., Assumption~\ref{ass:relax-smooth}) can recover the standard relaxed smoothness (e.g., Definition~\ref{def:rs-jikai}) when the upper-level variable $x$ and the lower-level variable $y$ have the same smoothness constants.

\begin{lemma}[Lemma 6 in \cite{hao2024bilevel}]
Let $L_{x,0}=L_{y,0}=K_0/2$ and $L_{x,1}=L_{y,1}=K_1/2$, then Assumption~\ref{ass:relax-smooth} implies that for any $w=(x,y),w'=(x',y')$ such that $\|w-w'\|\leq 1/K_1$, we have 
\begin{equation*}
    \|\nabla_w f(w) - \nabla_w f(w')\| \leq (K_0 + K_1\|\nabla_w f(w)\|)\|w-w'\|.
\end{equation*}%
In other words, $(L_{x,0}, L_{x,1}, L_{y,0}, L_{y,1})$-smoothness assumption (Assumption~\ref{ass:relax-smooth}) can recover the standard relaxed smoothness assumption (Definition~\ref{def:rs-jikai}).
\end{lemma}

\section{Auxiliary Lemmas for Bilevel Optimization Problems}

\setcounter{equation}{0}
\renewcommand{\theequation}{C.\arabic{equation}}



\begin{lemma}[Hypergradient formula, Lemma 7 in \cite{hao2024bilevel}] \label{lm:hyper-formulation}
The hypergradient $\gdphi(x)$ takes the forms of 
\begin{equation*}
    \begin{aligned}
        \gdphi(x)
        &= \gdx f(x,y^*(x)) - \gdxy g(x,y^*(x))[\gdyy g(x,y^*(x))]^{-1}\gdy f(x, y^*(x)) \\
        &= \gdx f(x,y^*(x)) - \gdxy g(x,y^*(x))z^*(x), \\
    \end{aligned}
\end{equation*}%
where $z^*(x)$ is the solution to the following linear system:
\begin{equation*}
    z^*(x) = \argmin_z \frac{1}{2}\langle \gdyy g(x,y^*(x))z, z \rangle - \langle \gdy f(x,y^*(x)), z \rangle.
\end{equation*}%
\end{lemma}

\begin{lemma}[Lipschitz property, Lemma 8 in \cite{hao2024bilevel}] \label{lm:lip-y*z*}
Suppose Assumptions~\ref{ass:relax-smooth} and~\ref{ass:f-and-g} hold, then we have 
\begin{itemize}
    \item $y^*(x)$ is $(l_{g,1}/\mu)$-Lipschitz continuous.
    \item $z^*(x)$ is $l_{z^*}$-Lipschitz continuous, i.e., 
    \begin{equation*}
        \|z^*(x)-z^*(x')\| \leq l_{z^*}\|x-x'\| \quad \text{if} \quad 
        \|x-x'\| \leq \frac{1}{\sqrt{2(1+l_{g,1}^2/\mu^2)(L_{x,1}^2+L_{y,1}^2)}},
    \end{equation*}%
    where $l_{z^*}$ is defined as 
    \begin{equation*}
        l_{z^*} \coloneqq \sqrt{1+\frac{l_{g,1}^2}{\mu^2}}\left(\frac{l_{g,2}l_{f,0}}{\mu^2} + \frac{1}{\mu}(L_{y,0}+L_{y,1}l_{f,0})\right).
    \end{equation*}%
\end{itemize}
\end{lemma}

\begin{lemma}[$(L_0,L_1)$-Smoothness, Lemma 9 in \cite{hao2024bilevel}] \label{eq:technical-rs-phi}
Suppose Assumptions~\ref{ass:relax-smooth} and~\ref{ass:f-and-g} hold. Then for any $x, x'$ we have
\begin{equation*}
    \|\gdphi(x)-\gdphi(x')\| \leq (L_0 + L_1\|\gdphi(x')\|)\|x-x'\| 
    \quad \text{if} \quad
    \|x-x'\| \leq \frac{1}{\sqrt{2(1+l_{g,1}^2/\mu^2)(L_{x,1}^2+L_{y,1}^2)}},
\end{equation*}%
where $(L_0,L_1)$-smoothness constant $L_0$ and $L_1$ are defined as
\begin{equation*}
    L_0 = \sqrt{1+\frac{l_{g,1}^2}{\mu^2}}\left(L_{x,0} + L_{x,1}\frac{l_{g,1}l_{f,0}}{\mu} + \frac{l_{g,1}}{\mu}(L_{y,0}+L_{y,1}l_{f,0}) + l_{f,0}\frac{l_{g,1}l_{g,2}+l_{g,2}\mu}{\mu^2}\right) \quad \text{and} \quad
    L_1 = \sqrt{1+\frac{l_{g,1}^2}{\mu^2}}L_{x,1}.
\end{equation*}%
\end{lemma}

\begin{lemma}[Descent inequality, Lemma 10 in \cite{hao2024bilevel}] \label{eq:technical-descent}
Suppose Assumptions~\ref{ass:relax-smooth} and~\ref{ass:f-and-g} hold. 
Then for any $x, x'$ we have
\begin{equation*}
    \Phi(x) \leq \Phi(x') + \langle \gdphi(x'),x-x' \rangle + \frac{L_0+L_1\|\gdphi(x')\|}{2}\|x-x'\|^2
    \quad \text{if} \quad
    \|x-x'\| \leq \frac{1}{\sqrt{2(1+l_{g,1}^2/\mu^2)(L_{x,1}^2+L_{y,1}^2)}}.
\end{equation*}
\end{lemma}

\section{Omitted Proofs in Section~\ref{sec:Convergence in Expectation}} \label{sec:in-expectation}

\setcounter{equation}{0}
\renewcommand{\theequation}{D.\arabic{equation}}

\subsection{Tracking the minimizer of lower-level function: high-probability guarantees} \label{sec:track-y}



In this section we present a few useful lemmas in \cite{cutler2023stochastic}. In particular, Lemma~\ref{lm:one-step improvement} and~\ref{lm:distance recursion} serve as a starting point and provide standard one-step improvement guarantee and distance recursion formulation for tracking the minimizer of the lower-level objective function. Then one can apply Proposition~\ref{prop:MGF-recursive control}, which is a technical result for recursive control, to obtain the distance tracking result with high probability as shown in Lemma~\ref{prop:high-prob distance}. We present some of the proofs under notations in this paper as below for completeness.

\begin{lemma}[One-step improvement, Lemma 2 and 23 in \cite{cutler2023stochastic}] \label{lm:one-step improvement}
Consider Algorithm~\ref{alg:sgd} with arbitrary sequence $\{\Tilde{x}_t\}$, for any $y\in\R^{d_y}$ and $t\geq1$, we have the following estimate:
\begin{equation}
    \begin{aligned}
        2\alpha(g(\tilde{x}_t,\tilde{y}_{t+1}) - g(\tilde{x}_t,y)) \leq (1-\mu\alpha)\|\tilde{y}_t-y\|^2 - \|\tilde{y}_{t+1}-y\|^2 + 2\alpha\langle u_t,y-\tilde{y}_t \rangle + \frac{\alpha^2}{1-l_{g,1}\alpha}\|u_t\|^2.
    \end{aligned}
\end{equation}
where we define $u_t = \gdy G(\tilde{x}_t,\tilde{y}_t;\tilde{\pi}_t) - \gdy g(\tilde{x}_t,\tilde{y}_t)$ to be the bias of the gradient estimator at time $t$.
\end{lemma}

\begin{proof}[Proof of Lemma~\ref{lm:one-step improvement}]
Since $g(x,y)$ is $l_{g,1}$-smooth jointly in $(x,y)$, we have
\begin{equation*}
    \begin{aligned}
        g(\tilde{x}_t,\tilde{y}_{t+1})
        &\leq g(\tilde{x}_t,\tilde{y}_t) + \langle \gdy g(\tilde{x}_t,\tilde{y}_t),\tilde{y}_{t+1}-\tilde{y}_t \rangle + \frac{l_{g,1}}{2}\|\tilde{y}_{t+1}-\tilde{y}_t\|^2 \\
        &= g(\tilde{x}_t,\tilde{y}_t) + \langle \gdy G(\tilde{x}_t,\tilde{y}_t;\tilde{\pi}_t),\tilde{y}_{t+1}-\tilde{y}_t \rangle + \frac{l_{g,1}}{2}\|\tilde{y}_{t+1}-\tilde{y}_t\|^2 + \langle u_t,\tilde{y}_t-\tilde{y}_{t+1} \rangle. \\
    \end{aligned}
\end{equation*}
For any given $\delta_t>0$, applying Young's inequality yields
\begin{equation*}
    \begin{aligned}
        \langle u_t,\tilde{y}_t-\tilde{y}_{t+1} \rangle \leq \frac{\delta_t}{2}\|u_t\|^2 + \frac{1}{2\delta_t}\|\tilde{y}_{t+1}-\tilde{y}_t\|^2.
    \end{aligned}
\end{equation*}
Then for any given $y\in\R^{d_y}$, we have
\begin{sequation*}
    \begin{aligned}
        g(\tilde{x}_t,\tilde{y}_{t+1})
        &\leq g(\tilde{x}_t,\tilde{y}_t) + \langle \gdy G(\tilde{x}_t,\tilde{y}_t;\tilde{\pi}_t),\tilde{y}_{t+1}-\tilde{y}_t \rangle + \frac{l_{g,1}+\delta_t^{-1}}{2}\|\tilde{y}_{t+1}-\tilde{y}_t\|^2 + \frac{\delta_t}{2}\|u_t\|^2 \\
        &= g(\tilde{x}_t,\tilde{y}_t) + \langle \gdy G(\tilde{x}_t,\tilde{y}_t;\tilde{\pi}_t),\tilde{y}_{t+1}-\tilde{y}_t \rangle + \frac{1}{2\alpha}\|\tilde{y}_{t+1}-\tilde{y}_t\|^2 + \frac{l_{g,1}+\delta_t^{-1}-\alpha^{-1}}{2}\|\tilde{y}_{t+1}-\tilde{y}_t\|^2 + \frac{\delta_t}{2}\|u_t\|^2 \\
        &\leq g(\tilde{x}_t,\tilde{y}_t) + \langle \gdy G(\tilde{x}_t,\tilde{y}_t;\tilde{\pi}_t),y-\tilde{y}_t \rangle + \frac{1}{2\alpha}\|y-\tilde{y}_t\|^2 - \frac{1}{2\alpha}\|y-\tilde{y}_{t+1}\|^2 + \frac{l_{g,1}+\delta_t^{-1}-\alpha^{-1}}{2}\|\tilde{y}_{t+1}-\tilde{y}_t\|^2 + \frac{\delta_t}{2}\|u_t\|^2, \\
    \end{aligned}
\end{sequation*}%
where the last inequality holds since $\tilde{y}_{t+1} = \tilde{y}_t - \alpha \gdy g(\tilde{x}_t,\tilde{y}_t;\tilde{\pi}_t)$ is the unique minimizer of the $\alpha^{-1}$-strongly convex function $h(y) = \langle \gdy g(\tilde{x}_t,\tilde{y}_t;\tilde{\pi}_t),y-\tilde{y}_t \rangle + \frac{1}{2\alpha}\|y-\tilde{y}_t\|^2$, and thus $h(y) - h(\tilde{y}_{t+1}) \geq \frac{1}{2\alpha}\|y-\tilde{y}_{t+1}\|^2$ holds for any $y\in\R^{d_y}$. 
Now by $\mu$-strong convexity of $g(x,y)$ in terms of $y$, we estimate
\begin{equation*}
    \begin{aligned}
        g(\tilde{x}_t,\tilde{y}_t) + \langle \gdy G(\tilde{x}_t,\tilde{y}_t;\tilde{\pi}_t),y-\tilde{y}_t \rangle
        &= g(\tilde{x}_t,\tilde{y}_t) + \langle \gdy g(\tilde{x}_t,\tilde{y}_t),y-\tilde{y}_t \rangle + \langle u_t,y-\tilde{y}_t \rangle \\
        &\leq g(\tilde{x}_t,y) - \frac{l_{g,1}}{2}\|y-\tilde{y}_t\|^2 + \langle u_t,y-\tilde{y}_t \rangle. \\
    \end{aligned}
\end{equation*}
Therefore we have
\begin{equation*}
    \begin{aligned}
        g(\tilde{x}_t,\tilde{y}_{t+1}) 
        \leq g(\tilde{x}_t,y) - \frac{l_{g,1}}{2}\|y-\tilde{y}_t\|^2 &+ \langle u_t,y-\tilde{y}_t \rangle + \frac{1}{2\alpha}\|y-\tilde{y}_t\|^2 - \frac{1}{2\alpha}\|y-\tilde{y}_{t+1}\|^2 \\
        &+ \frac{l_{g,1}+\delta_t^{-1}-\alpha^{-1}}{2}\|\tilde{y}_{t+1}-\tilde{y}_t\|^2 + \frac{\delta_t}{2}\|u_t\|^2. \\
    \end{aligned}
\end{equation*}
Finally, taking $\delta_t = \alpha/(1-l_{g,1}\alpha)$ and rearranging yields
\begin{equation*}
    \begin{aligned}
        2\alpha(g(\tilde{x}_t,\tilde{y}_{t+1}) - g(\tilde{x}_t,y)) \leq (1-\mu\alpha)\|\tilde{y}_t-y\|^2 - \|\tilde{y}_{t+1}-y\|^2 + 2\alpha\langle u_t,y-\tilde{y}_t \rangle + \frac{\alpha^2}{1-l_{g,1}\alpha}\|u_t\|^2.
    \end{aligned}
\end{equation*}
\end{proof}

\begin{lemma}[Distance recursion, Lemma 25 in \cite{cutler2023stochastic}] \label{lm:distance recursion}
Consider Algorithm~\ref{alg:sgd} with arbitrary sequence $\{\Tilde{x}_t\}$, for any $t\geq1$, we have the following recursion:
\begin{equation}
    \begin{aligned}
        \|\tilde{y}_{t+1}-\tilde{y}_{t+1}^*\|^2 \leq (1-\mu\alpha)\|\tilde{y}_t-\tilde{y}_t^*\|^2 + 2\alpha\langle u_t,\tilde{y}_t^*-\tilde{y}_t \rangle + \frac{\alpha^2}{1-l_{g,1}\alpha}\|u_t\|^2 + \left(1+\frac{1}{\mu\alpha}\right)D_t^2,
    \end{aligned}
\end{equation}
where we define $D_t \coloneqq \|\tilde{y}_t^*-\tilde{y}_{t+1}^*\|$ to be the minimizer drift at time $t$.
\end{lemma}

\begin{proof}[Proof of Lemma~\ref{lm:distance recursion}]
Note that the $\mu$-strong convexity of $g(x,y)$ in terms of $y$ implies
\begin{equation*}
    g(\tilde{x}_t,\tilde{y}_{t+1}) - g(\tilde{x}_t,\tilde{y}_t^*) \geq \frac{\mu}{2}\|\tilde{y}_{t+1}-\tilde{y}_t^*\|^2.
\end{equation*}
Combing this estimate with Lemma~\ref{lm:one-step improvement} under the identification $y = \tilde{y}_t^*$ yields
\begin{equation*}
    (1+\mu\alpha)\|\tilde{y}_{t+1}-\tilde{y}_t^*\|^2 \leq (1-\mu\alpha)\|\tilde{y}_t-\tilde{y}_t^*\|^2 + 2\alpha\langle u_t,\tilde{y}_t^*-\tilde{y}_t \rangle + \frac{\alpha^2}{1-l_{g,1}\alpha}\|u_t\|^2.
\end{equation*}
Then we apply Young's inequality to obtain
\begin{equation*}
    \|\tilde{y}_{t+1}-\tilde{y}_{t+1}^*\|^2 \leq (1+\mu\alpha)\|\tilde{y}_{t+1}-\tilde{y}_t^*\|^2 + (1+(\mu\alpha)^{-1})\|\tilde{y}_t^*-\tilde{y}_{t+1}^*\|^2.
\end{equation*}
Combining the above inequalities yields
\begin{equation*}
    \begin{aligned}
        \|\tilde{y}_{t+1}-\tilde{y}_{t+1}^*\|^2 
        &\leq (1+\mu\alpha)\|\tilde{y}_{t+1}-\tilde{y}_t^*\|^2 + (1+(\mu\alpha)^{-1})\|\tilde{y}_t^*-\tilde{y}_{t+1}^*\|^2 \\
        &\leq (1-\mu\alpha)\|\tilde{y}_t-\tilde{y}_t^*\|^2 + 2\alpha\langle u_t,\tilde{y}_t^*-\tilde{y}_t \rangle + \frac{\alpha^2}{1-l_{g,1}\alpha}\|u_t\|^2 + (1+(\mu\alpha)^{-1})\|\tilde{y}_t^*-\tilde{y}_{t+1}^*\|^2 \\
        &= (1-\mu\alpha)\|\tilde{y}_t-\tilde{y}_t^*\|^2 + 2\alpha\langle u_t,\tilde{y}_t^*-\tilde{y}_t \rangle + \frac{\alpha^2}{1-l_{g,1}\alpha}\|u_t\|^2 + \left(1+\frac{1}{\mu\alpha}\right)D_t^2,
    \end{aligned}
\end{equation*}
where the last equality holds by definition of minimizer drift $D_t$.
\end{proof}

\begin{proposition}[Recursive control on MGF, Proposition 29 in \cite{cutler2023stochastic}] \label{prop:MGF-recursive control}
Consider scalar stochastic processes $(V_t)$, $(U_t)$, and $(X_t)$ on a probability space with filtration $(\gH_t)$, which are linked by the inequality
\begin{equation}
    V_{t+1} \leq \alpha_tV_t + U_t\sqrt{V_t} + X_t + \kappa_t
\end{equation}
for some deterministic constants $\alpha_t\in(-\infty,1]$ and $\kappa_t\in\R$. Suppose the following properties hold.
\begin{itemize}
    \item $V_t$ is non-negative and $\gH_t$-measurable.
    \item $U_t$ is mean-zero sub-Gaussian conditioned on $\gH_t$ with deterministic parameter $\sigma_t$:
    \begin{equation*}
        \E[\exp(\lambda U_t) \mid \gH_t] \leq \exp(\lambda^2\sigma_t^2/2), \quad \forall \lambda\in\R.
    \end{equation*}
    \item $X_t$ is non-negative and sub-exponential conditioned on $\gH_t$ with deterministic parameter $\nu_t$:
    \begin{equation*}
        \E[\exp(\lambda X_t) \mid \gH_t] \leq \exp(\lambda\nu_t), \quad \forall \lambda\in[0,1/\nu_t].
    \end{equation*}
\end{itemize}
Then the estimate 
\begin{equation*}
    \E[\exp(\lambda V_{t+1})] \leq \exp(\lambda(\nu_t+\kappa_t))\E[\exp(\lambda(1+\alpha_t)V_t/2)]
\end{equation*}
holds for any $\lambda$ satisfying $0\leq \lambda \leq \min\left(\frac{1-\alpha_t}{2\sigma_t^2}, \frac{1}{2\nu_t}\right)$.
\end{proposition}

\begin{lemma}[High-probability distance tracking, Theorem 6 and 30 in \cite{cutler2023stochastic}] \label{prop:high-prob distance}
Suppose that Assumption~\ref{ass:f-and-g} holds and let $\{\tilde{y}_t\}$ be the iterates produced by Algorithm~\ref{alg:sgd} with constant learning rate $\alpha\leq 1/(2l_{g,1})$. 
We further assume there exists constant $D>0$ such that the drift $D_t^2$ is sub-exponential conditioned on $\gF_t^1$ with parameter $D^2$:
\begin{equation*}
    \E\left[\exp(\lambda D_t^2) \mid \gF_t^1\right] \leq \exp(\lambda D^2) \quad \text{for all} \quad 0\leq \lambda \leq D^{-2}.
\end{equation*}
Then for any fixed $t\in[T]$ and $\delta\in(0,1)$, the following estimate holds with probability at least $1-\delta$ over the randomness in $\widetilde{\gF}_t^1$:
\begin{sequation} \label{eq:high-prob distance}
    \|\tilde{y}_t-\tilde{y}_t^*\|^2 \leq \left(1-\frac{\mu\alpha}{2}\right)^t\|\tilde{y}_0-\tilde{y}_0^*\|^2 + \left(\frac{8\alpha \sigma_{g,1}^2}{\mu} + \frac{4D^2}{\mu^2\alpha^2}\right)\log\left(\frac{e}{\delta}\right),
\end{sequation}%
where $e$ denotes the base of natural logarithms. 
\end{lemma}

\textbf{Remark:} We would like to claim that in original bound for \eqref{eq:high-prob distance}, there is an extra $c^2$ on the variance term, namely $c^2\sigma_{g,1}^2$ instead of $c^2$. By statement of Theorem 30 (footnote 6) in \cite{cutler2023stochastic}, the absolute constant satisfies $c\geq1$, so here in \eqref{eq:high-prob distance} we just take $c=1$ for simplicity.

\subsection{Proof of Lemma~\ref{lm:maintext-lm1}}

In this section we establish the tracking error of the lower-level problem $\|y_t-y_t^*\|$ at each iteration for \emph{any fixed slowly changing upper-level sequence $\{\tilde{x}_t\}$} with high probability. Our result (Lemma~\ref{lm:maintext-lm1}) can be seen as a direct generalization of Lemma~\ref{prop:high-prob distance}. It will be applied to two stages of Algorithm~\ref{alg:bilevel}, which helps us build a refined characterization of the lower-level problem. To be more specific, Lemma~\ref{lm:maintext-lm2} corresponds to warm-start stage in line 3, and Lemma~\ref{lm:maintext-lm3} corresponds to simultaneous update stage in line 4$\sim$10. 

\begin{lemma}[Restatement of Lemma~\ref{lm:maintext-lm1}] \label{lm:appendix-lm1}
Suppose that Assumption~\ref{ass:f-and-g} holds, let $\{\tilde{y}_t\}$ be the iterates produced by Algorithm~\ref{alg:sgd} with any fixed input sequence $\{\tilde{x}_t\}$ such that $\|\tilde{x}_{t+1}-\tilde{x}_t\|\leq R$ for all $t\geq0$, and constant learning rate $\alpha\leq 1/(2l_{g,1})$. 
Then for any fixed $t\in[N]$ and $\delta\in(0,1)$, the following estimate holds with probability at least $1-\delta$ over the randomness in $\widetilde{\gF}_t^1$:
\begin{sequation} \label{eq:fixed-t-bound}
    \|\tilde{y}_t-\tilde{y}_t^*\|^2 \leq \left(1-\frac{\mu\alpha}{2}\right)^t\|\tilde{y}_0-\tilde{y}_0^*\|^2 + \left(\frac{8\alpha \sigma_{g,1}^2}{\mu} + \frac{4R^2l_{g,1}^2}{\mu^4\alpha^2}\right)\log\left(\frac{e}{\delta}\right).
\end{sequation}%
As a consequence, for any given $\delta\in(0,1)$ and all $t\in[N]$, the following estimate holds with probability at least $1-\delta$ over the randomness in $\widetilde{\gF}_{T}^1$:
\begin{sequation} \label{eq:all-t-bound}
    \|\tilde{y}_t-\tilde{y}_t^*\|^2 \leq \left(1-\frac{\mu\alpha}{2}\right)^t\|\tilde{y}_0-\tilde{y}_0^*\|^2 + \left(\frac{8\alpha \sigma_{g,1}^2}{\mu} + \frac{4R^2l_{g,1}^2}{\mu^4\alpha^2}\right)\log\left(\frac{eN}{\delta}\right).
\end{sequation}%
\end{lemma}

\begin{proof}[Proof of Lemma~\ref{lm:appendix-lm1}]
By Lemma~\ref{lm:lip-y*z*}, $y^*(x)$ is $(l_{g,1}/\mu)$-Lipschitz continuous, then $D_t = \|\tilde{y}_t^*-\tilde{y}_{t+1}^*\| \leq \frac{l_{g,1}}{\mu}\|\tilde{x}_t-\tilde{x}_{t+1}\| = l_{g,1}R/\mu$ holds for any $t\in[N]$ almost surely, hence we can choose $D = l_{g,1}R/\mu$ by Lemma~\ref{eq:high-prob distance}.
Replacing $D^2$ with $l_{g,1}^2R^2/\mu^2$ in \eqref{eq:high-prob distance} yields \eqref{eq:fixed-t-bound}, and we obtain \eqref{eq:all-t-bound} by using union bound.
\end{proof}

\subsection{Proof of Lemma~\ref{lm:maintext-lm2}}

In this section we apply Lemma~\ref{lm:maintext-lm1} to obtain constant error of the lower-level variable (i.e., $\|y_0-y_0^*\|$) with high probability within a logarithmic number of iterations.

\begin{lemma}[Warm-start, Restatement of Lemma~\ref{lm:maintext-lm2}] \label{lm:warm-up}
Suppose that Assumption~\ref{ass:f-and-g} holds and given any $\delta\in(0,1)$, let $\{y_t^{\init}\}$ be the iterates produced by Algorithm~\ref{alg:sgd} starting from $y_0^{\init}$ with $R=0$ (where $R$ is defined in Lemma~\ref{lm:maintext-lm1}, it means that $\tilde{x}_t=x_0$ for any $t$)
and constant learning rate $\alpha^{\init}$ satisfying
\begin{sequation} \label{eq:alpha-init}
    \alpha^{\init} = \min\left\{\frac{1}{2l_{g,1}}, \frac{\mu}{2048L_1^2\sigma_{g,1}^2\log(e/\delta)}\right\} = O(1).
\end{sequation}%
Then Algorithm~\ref{alg:sgd} guarantees $\|y_{T_0}^{\init}-y_0^*\| \leq 1/(8\sqrt{2}L_1)$ with probability at least $1-\delta$ over the randomness in $\widetilde{\mathcal{F}}_{T_0}^{1}$ (we denote this event as $\gE_{\init}$), where the number of iterations $T_0$ satisfies
\begin{sequation} \label{eq:T0-init}
    T_0 = \frac{\log\left(256L_1^2\|y_0^{\init}-y_0^*\|^2\right)}{\log\left(2/(2-\mu\alpha^{\init})\right)} = \widetilde{O}(1).
\end{sequation}%
\end{lemma}

\begin{proof}[Proof of Lemma~\ref{lm:warm-up}]
We set $R=0$ in \eqref{eq:fixed-t-bound} such that $\tilde{x}_t=x_0$ for any $t\in[T_0]$, and learning rate satisfies $\alpha\leq 1/(2l_{g,1})$ by \eqref{eq:alpha-init}. Then by Lemma~\ref{lm:appendix-lm1}, for any $\delta\in(0,1)$, with probability at least $1-\delta$ over the randomness in $\widetilde{\mathcal{F}}_{T_0}^{1}$ we have
\begin{sequation*}
    \begin{aligned}
        \|y_{T_0}^{\init}-y_0^*\|^2 
        &\leq \left(1-\frac{\mu\alpha^{\init}}{2}\right)^{T_0}\|y_0^{\init}-y_0^*\|^2 + \frac{8\alpha^{\init}\sigma_{g,1}^2}{\mu}\log\left(\frac{e}{\delta}\right) \\
        &\leq \frac{1}{256L_1^2} + \frac{1}{256L_1^2} = \frac{1}{128L_1^2},
    \end{aligned}
\end{sequation*}%
where we use \eqref{eq:alpha-init} and \eqref{eq:T0-init} for the last inequality. Therefore, we obtain $\|y_{T_0}^{\init}-y_0^*\|\leq 1/(8\sqrt{2}L_1)$ under event $\gE_{\init}$.
\end{proof}

\subsection{Proof of Lemma~\ref{lm:maintext-lm3}}

In this section we again apply Lemma~\ref{lm:maintext-lm1} with proper parameter setting to obtain a refined control of the lower-level variable.

\begin{lemma}[Restatement of Lemma~\ref{lm:maintext-lm3}] \label{lm:high-prob 1/8L_1 bound for y}
Under Assumptions~\ref{ass:relax-smooth},~\ref{ass:f-and-g} and event $\gE_{\init}$, for any given $\delta\in(0,1)$ and any small $\eps$ satisfying
\begin{sequation} \label{eq:lm_eps}
    \begin{aligned}
        \eps \leq \min&\left\{\frac{L_0}{L_1}, \Delta_{y,0}L_0, \frac{8l_{g,1}L_0}{\mu\sqrt{2(1+l_{g,1}^2/\mu^2)(L_{x,1}^2+L_{y,1}^2)}}, \sqrt{\frac{16el_{g,1}\Delta_0L_0}{\mu\delta}}, 4\left(\frac{el_{g,1}\Delta_0L_0^3\sigma_{g,1}^2}{\mu^3\delta}\right)^{1/4}; \right. \\
        &\left.  \sqrt{\frac{32el_{g,1}\Delta_0L_0}{\mu\delta\exp(\mu/(2l_{g,1}))}}, \sqrt{\frac{32el_{g,1}\Delta_0L_0}{\mu\delta\exp(\mu l_{g,1}\Delta_{z,0}/(2\Delta_0L_0))}}, \sqrt{\frac{32el_{g,1}\Delta_0L_0}{\mu\delta\exp(\mu\Delta_{y,0}^2L_0/(2l_{g,1}\Delta_0))}}, \right. \\
        &\left. \frac{\Delta_0L_0}{\|\gdphi(x_0)\|}, \frac{L_0\sigma_{g,1}}{\sigma_{g,2}}, \frac{L_0\sigma_{g,1}}{\sqrt{\mu l_{g,1}}}, \left(\frac{2^{21}el_{g,1}\Delta_0L_0^3\sigma_{g,1}^2}{\mu^3\delta}\right)^{1/4}\exp\left(\frac{-l_{g,1}\sqrt{\sigma_{f,1}^2 + 2l_{f,0}^2\sigma_{g,2}^2/\mu^2}}{512L_0\sigma_{g,1}}\right)
        \right\},
    \end{aligned}
\end{sequation}%
if we choose parameters $\alpha, \beta, \gamma, \eta$ as 
\begin{small}
\begin{equation} \label{eq:lm_beta}
    \begin{aligned}
        1-\beta = \min&\left\{1, \frac{\mu}{16l_{g,1}}, \frac{16el_{g,1}\Delta_0L_0}{\mu\delta\eps^2}, \frac{\mu^2\eps^2}{64\cdot1024L_0^2\sigma_{g,1}^2\log^2(B)}; \frac{\min\{1, \mu^2/(32l_{g,1}^2)\}}{\sigma_{f,1}^2 + 2l_{f,0}^2\sigma_{g,2}^2/\mu^2}\eps^2, \frac{l_{g,1}^2}{8\sigma_{g,2}^2}, \frac{\mu^2}{16\sigma_{g,2}^2}; \right. \\
        &\left. \frac{32el_{g,1}\Delta_0L_0}{\mu\delta\eps^2\exp(\mu/(2l_{g,1}))}, \frac{32el_{g,1}\Delta_0L_0}{\mu\delta\eps^2\exp(\mu l_{g,1}\Delta_{z,0}^2/(2\Delta_0L_0))}, \frac{32el_{g,1}\Delta_0L_0}{\mu\delta\eps^2\exp(\mu\Delta_{y,0}^2L_0/(2l_{g,1}\Delta_0))}
        \right\},
    \end{aligned}
\end{equation}
\end{small}%
\begin{sequation} \label{eq:lm_eta}
    \begin{aligned}
        \eta = \min\left\{\frac{1}{8}\min\left(\frac{1}{L_1}, \frac{\eps}{L_0}, \frac{\eps\Delta_0}{\Delta_{y,0}^2L_0^2}, \frac{\Delta_0}{\|\gdphi(x_0)\|}, \frac{\eps\Delta_0}{l_{g,1}^2\Delta_{z,0}^2}, \frac{\mu\eps}{l_{g,1}L_0\log(A)}\right)(1-\beta), \frac{1}{\sqrt{2(1+l_{g,1}^2/\mu^2)(L_{x,1}^2+L_{y,1}^2)}}\right\},
    \end{aligned}
\end{sequation}%
\begin{sequation} \label{eq:lm_gamma+alpha+T}
    \begin{aligned}
        \gamma = \frac{1}{\mu}(1-\beta), \quad\quad \alpha = \frac{8}{\mu}(1-\beta), 
    \end{aligned}
\end{sequation}%
where $e$ denotes the base of natural logarithms, and we define $\Delta_0, \Delta_{y,0}, \Delta_{z,0}, A$ and $B$ as 
\begin{sequation} \label{eq:ABdelta}
    \begin{aligned}
        \Delta_0 \coloneqq \Phi(x_0) - \inf_{x\in\R^{d_x}}\Phi(x), \quad
        \Delta_{y,0} \coloneqq \|y_0-y_0^*\|, \quad
        \Delta_{z,0}\coloneqq\|z_0-z_0^*\|, \quad
        A \coloneqq \left(\frac{32el_{g,1}\Delta_0L_0}{\mu\delta\eps^2(1-\beta)}\right)^2, \quad
        B \coloneqq \left(\frac{2^{21}el_{g,1}\Delta_0L_0^3\sigma_{g,1}^2}{\mu^3\delta\eps^4}\right)^4.
    \end{aligned}
\end{sequation}%
Run Algorithm~\ref{alg:bilevel} for $T=\frac{4\Delta_0}{\eta\eps}$ iterations, then for all $t\in[T]$, Algorithm~\ref{alg:bilevel} guarantees with probability at least $1-\delta$  over the randomness in $\mathcal{F}_T^1$ (we denote this event as $\gE_y$) that:
\begin{enumerate}
    \item $\|y_t-y_t^*\|\leq \frac{1}{8L_1}$. \label{eq:1/8L1}
    \item $\frac{1}{T}(1-\beta)\sum_{t=0}^{T-1}\sum_{i=0}^{t}\beta^{t-i}\|y_i-y_i^*\| \leq \frac{3}{32L_0}\eps$. \label{eq:3eps/32L0}
\end{enumerate}
\end{lemma}

\begin{proof}[Proof of Lemma~\ref{lm:high-prob 1/8L_1 bound for y}]
To begin, by choice of $\beta$ and $\eta$ as in \eqref{eq:lm_beta} and \eqref{eq:lm_eta} we have
\begin{sequation*}
    \alpha = 8\gamma = \frac{8(1-\beta)}{\mu} \leq \frac{8}{\mu}\frac{\mu}{16l_{g,1}} \leq \frac{1}{2l_{g,1}}
    \quad\quad \text{and} \quad\quad
    \eta \leq \frac{1}{\sqrt{2(1+l_{g,1}^2/\mu^2)(L_{x,1}^2+L_{y,1}^2)}},
\end{sequation*}%
thus satisfy the condition for applying Lemma~\ref{lm:appendix-lm1} and Lemma~\ref{eq:technical-rs-phi}, i.e., $(L_0,L_1)$-smoothness of function $\Phi$.
By Lemma~\ref{lm:appendix-lm1} (we set $R=\eta$ here) and choice of $\alpha, \gamma, T$ as in \eqref{eq:lm_gamma+alpha+T}, with probability at least $1-\delta$  over the randomness in $\mathcal{F}_t^1$ (we denote this event as $\gE_y$) we have 
\begin{sequation} \label{eq:important_lm_1}
    \begin{aligned}
        \|y_t-y_t^*\|^2 
        &\leq \left(1-\frac{\mu\alpha}{2}\right)^t\|y_0-y_0^*\|^2 + \left(\frac{8\alpha\sigma_{g,1}^2}{\mu} + \frac{4\eta^2l_{g,1}^2}{\mu^4\alpha^2}\right)\log\left(\frac{eT}{\delta}\right) \\
        &= \left(1-\frac{\mu\alpha}{2}\right)^t\|y_0-y_0^*\|^2 + \left(\frac{64\gamma\sigma_{g,1}^2}{\mu} + \frac{\eta^2l_{g,1}^2}{16\mu^4\gamma^2}\right)\log\left(\frac{eT}{\delta}\right) \\
        &= \left(1-\frac{\mu\alpha}{2}\right)^t\|y_0-y_0^*\|^2 + \left(\frac{64(1-\beta)\sigma_{g,1}^2}{\mu^2} + \frac{\eta^2l_{g,1}^2}{16\mu^2(1-\beta)^2}\right)\log\left(\frac{4e\Delta_0}{\eta\delta\eps}\right). \\
    \end{aligned}
\end{sequation}%
Before delving into details, let's first briefly outline the structure of the proof for this lemma.

\subsection*{Proof for $\|y_t-y_t^*\|\leq \frac{1}{8L_1}$:}
For the first part of the proof, we split it into the following four parts.

In \textbf{step 1}, we claim that with choice of $\eps$ and $\beta$ as in \eqref{eq:lm_eps} and \eqref{eq:lm_beta}, the exact formula for $\eta$ and $1-\beta$ are as the following:
\begin{sequation} \label{eq:lm-exact-eta-beta}
    \begin{aligned}
        \eta = \frac{\mu\eps}{8l_{g,1}L_0\log(A)}(1-\beta)
        \quad\quad \text{and} \quad\quad
        1-\beta = \frac{\mu^2\eps^2}{64\cdot1024L_0^2\sigma_{g,1}^2\log^2(B)}.
    \end{aligned}
\end{sequation}%

In \textbf{step 2}, we show under event $\gE_{\init}$, for any $t\in[T]$, it holds that
\begin{sequation*}
    \left(1-\frac{\mu\alpha}{2}\right)^t\|y_0-y_0^*\|^2 \leq \frac{1}{128L_1^2}.
\end{sequation*}

In \textbf{step 3}, we show that with suitable choice of $\eps$, we have
\begin{sequation*}
    \frac{64(1-\beta) \sigma_{g,1}^2}{\mu^2} \leq \frac{\eta^2l_{g,1}^2}{16\mu^2(1-\beta)^2},
\end{sequation*}%
and thus we could combine the above two terms together
\begin{sequation*}
    \frac{64(1-\beta) \sigma_{g,1}^2}{\mu^2} + \frac{\eta^2l_{g,1}^2}{16\mu^2(1-\beta)^2} \leq \frac{\eta^2l_{g,1}^2}{16\mu^2(1-\beta)^2} + \frac{\eta^2l_{g,1}^2}{16\mu^2(1-\beta)^2} = \frac{\eta^2l_{g,1}^2}{8\mu^2(1-\beta)^2}.
\end{sequation*}%
In \textbf{step 4}, we show that 
\begin{sequation*}
    \frac{\eta^2l_{g,1}^2}{8\mu^2(1-\beta)^2}\log\left(\frac{4e\Delta_0}{\eta\delta\eps}\right) \leq \frac{\eps^2}{128L_0^2}.
\end{sequation*}%
Finally we merge the above four steps together and obtain
\begin{sequation*}
    \|y_t-y_t^*\|^2 \leq \frac{1}{128L_1^2} + \frac{\eps^2}{512L_0^2}
\end{sequation*}%
holds for any $t\in[T]$, and again by choice of $\eps \leq L_0/L_1$ as in \eqref{eq:lm_eps} we complete the first part of the proof. Now we begin proof step by step accordingly.

\paragraph{Step 1.} 
One could check Lemma~\ref{lm:verify} for details.


\paragraph{Step 2.} By Lemma~\ref{lm:warm-up}, under event $\gE_{\init}$, for any $t\in[T]$ we have 
\begin{sequation*}
    \left(1-\frac{\mu\alpha}{2}\right)^t\|y_0-y_0^*\|^2 \leq \|y_0-y_0^*\|^2 = \|y_{T_0}^{\init}-y_0^*\|^2 \leq \frac{1}{128L_1^2}.
\end{sequation*}

\paragraph{Step 3.} 
With choice of $\eta$ as in \eqref{eq:lm-exact-eta-beta}, it suffices to show that
\begin{sequation*}
    \frac{64(1-\beta) \sigma_{g,1}^2}{\mu^2} \leq \frac{\eta^2l_{g,1}^2}{16\mu^2(1-\beta)^2} \quad \Longleftrightarrow \quad \frac{64(1-\beta) \sigma_{g,1}^2}{\mu^2} \leq \frac{\eps^2}{1024L_0^2\log^2(A)},
\end{sequation*}%
which is equivalent to
\begin{sequation*}
    (1-\beta)\log^2(A) \leq \frac{\mu^2\eps^2}{64\cdot1024L_0^2\sigma_{g,1}^2} \quad \Longleftrightarrow \quad \log^2(A) \leq \log^2(B) \quad \Longleftrightarrow \quad \log(A) \leq \log(B),
\end{sequation*}%
where we use $A \geq 4$ and $B\geq 4$ by choice of $\beta$ and $\eps$ as in \eqref{eq:lm_beta} and \eqref{eq:lm_eps}, i.e.,
\begin{small}
\begin{equation*}
    1-\beta \leq \frac{16el_{g,1}\Delta_0L_0}{\mu\delta\eps^2}, \quad \eps \leq 4\left(\frac{el_{g,1}\Delta_0L_0^3\sigma_{g,1}^2}{\mu^3\delta}\right)^{1/4}.
\end{equation*}
\end{small}%
Plugging in the definition of $A$ and $B$ leads to
\begin{sequation} \label{eq:step3-equivalent}
    \log(A) \leq \log(B) \quad \Longleftrightarrow \quad A \leq B \quad \Longleftrightarrow \quad
    \left(\frac{32el_{g,1}\Delta_0L_0}{\mu\delta\eps^2(1-\beta)}\right)^2 \leq \left(\frac{2^{21}el_{g,1}\Delta_0L_0^3\sigma_{g,1}^2}{\mu^3\delta\eps^4}\right)^4.
\end{sequation}%
Recall \eqref{eq:lm-exact-eta-beta}, step 1 claims that
\begin{sequation*}
    1-\beta = \frac{\mu^2\eps^2}{64\cdot1024L_0^2\sigma_{g,1}^2\log^2(B)}.
\end{sequation*}%
Plug the above expression of $1-\beta$ into \eqref{eq:step3-equivalent}, then \eqref{eq:step3-equivalent} turns into the inequality $\sqrt{B}\log^4(B) \leq B$.
To summarize, now we have the following equivalence:
\begin{sequation*}
    \frac{64(1-\beta) \sigma_{g,1}^2}{\mu^2} \leq \frac{\eta^2l_{g,1}^2}{16\mu^2(1-\beta)^2} \quad \Longleftrightarrow \quad \sqrt{B}\log^4(B) \leq B.
\end{sequation*}%
Thus we only need to set $B\geq 22\times10^{10}$ to satisfy the above inequality in order to make the claim of step 3 hold. In fact, with choice of $\eps$ as in \eqref{eq:lm_eps}, one can easily verify that
\begin{sequation*}
    \eps \leq 4\left(\frac{el_{g,1}\Delta_0L_0^3\sigma_{g,1}^2}{\mu^3\delta}\right)^{1/4} \quad \Longrightarrow \quad B = \left(\frac{2^{21}el_{g,1}\Delta_0L_0^3\sigma_{g,1}^2}{\mu^3\delta\eps^4}\right)^4 \geq 22\times10^{10}.
\end{sequation*}%
This completes the proof of step 3 and we conclude that 
\begin{sequation*}
    \frac{64(1-\beta) \sigma_{g,1}^2}{\mu^2} + \frac{\eta^2l_{g,1}^2}{16\mu^2(1-\beta)^2} \leq \frac{\eta^2l_{g,1}^2}{16\mu^2(1-\beta)^2} + \frac{\eta^2l_{g,1}^2}{16\mu^2(1-\beta)^2} = \frac{\eta^2l_{g,1}^2}{8\mu^2(1-\beta)^2}.
\end{sequation*}%
\paragraph{Step 4.} 
Under choice of $\eta$ and $\beta$ as in \eqref{eq:lm-exact-eta-beta} and \eqref{eq:lm_beta}, we have $A\geq 4$ and thus
\begin{sequation*}
    \begin{aligned}
        \frac{\eta^2l_{g,1}^2}{8\mu^2(1-\beta)^2}\log\left(\frac{4e\Delta_0}{\eta\delta\eps}\right) 
        = \frac{\eps^2}{512L_0^2\log^2(A)}\log\left(\frac{32el_{g,1}\Delta_0L_0}{\mu\delta\eps^2(1-\beta)}\log(A)\right) \leq \frac{\eps^2}{512L_0^2\log(A)}\log\left(\sqrt{A}\log(A)\right) 
        \leq \frac{\eps^2}{512L_0^2},
    \end{aligned}
\end{sequation*}%
where we use the fact that $\log(\sqrt{A}\log(A)) \leq \log(A) \leq \log^2(A)$ for any $A\geq 4$.
\paragraph{Merge Step.} Finally, merging the above four steps and combining with \eqref{eq:important_lm_1} gives that, under event $\gE_{\init}$, for any $t\in[T]$,
\begin{sequation} \label{eq:merge-step}
    \begin{aligned}
        \|y_t-y_t^*\|^2 
        &\leq \left(1-\frac{\mu\alpha}{2}\right)^t\|y_0-y_0^*\|^2 + \left(\frac{64(1-\beta) \sigma_{g,1}^2}{\mu^2} + \frac{\eta^2l_{g,1}^2}{16\mu^2(1-\beta)^2}\right)\log\left(\frac{4e\Delta_0}{\eta\delta\eps}\right) \\
        &\leq \left(1-\frac{\mu\alpha}{2}\right)^t\|y_0-y_0^*\|^2 + \frac{\eta^2l_{g,1}^2}{8\mu^2(1-\beta)^2}\log\left(\frac{4e\Delta_0}{\eta\delta\eps}\right) \\
        &\leq \frac{1}{128L_1^2} + \frac{\eps^2}{512L_0^2},
    \end{aligned}
\end{sequation}%
which, together with $\eps \leq L_0/L_1$ yields the first part of the result.

\subsection*{Proof for $(1-\beta)\sum_{t=0}^{T-1}\sum_{i=0}^{t}\beta^{t-i}\|y_i-y_i^*\| \leq \frac{3\eps}{32L_0}$:}
As for the second part, under event $\gE_{\init} \cap \gE_y$ we have
\begin{small}
\begin{equation} \label{eq:second-part1}
    \begin{aligned}
        \frac{1}{T}&(1-\beta)\sum_{t=0}^{T-1}\sum_{i=0}^{t}\beta^{t-i}\|y_i-y_i^*\|
        \overset{(i)}{\leq} \frac{(1-\beta)}{T}\sum_{t=0}^{T-1}\sum_{i=0}^{t}\beta^{t-i}\sqrt{\left(1-\frac{\mu\alpha}{2}\right)^i\|y_0-y_0^*\|^2 + \left(\frac{8\alpha \sigma_{g,1}^2}{\mu} + \frac{4\eta^2l_{g,1}^2}{\mu^4\alpha^2}\right)\log\left(\frac{eT}{\delta}\right)} \\
        &\overset{(ii)}{\leq} \frac{1}{T}(1-\beta)\sum_{t=0}^{T-1}\sum_{i=0}^{t}\beta^{t-i}\left[\left(1-\frac{\mu\alpha}{2}\right)^{i/2}\|y_0-y_0^*\| + \sqrt{\left(\frac{8\alpha \sigma_{g,1}^2}{\mu} + \frac{4\eta^2l_{g,1}^2}{\mu^4\alpha^2}\right)\log\left(\frac{eT}{\delta}\right)}\right] \\
        &\leq \frac{1}{T}(1-\beta)\sum_{t=0}^{T-1}\left[\beta^t\sum_{i=0}^{t}\left(\frac{\sqrt{1-\mu\alpha/2}}{\beta}\right)^i\|y_0-y_0^*\| + \frac{1}{1-\beta}\sqrt{\left(\frac{8\alpha \sigma_{g,1}^2}{\mu} + \frac{4\eta^2l_{g,1}^2}{\mu^4\alpha^2}\right)\log\left(\frac{eT}{\delta}\right)}\right] \\
        &\overset{(iii)}{\leq} \frac{4\Delta_{y,0}}{T(\mu\alpha-4(1-\beta))} + \sqrt{\left(\frac{8\alpha \sigma_{g,1}^2}{\mu} + \frac{4\eta^2l_{g,1}^2}{\mu^4\alpha^2}\right)\log\left(\frac{eT}{\delta}\right)}, \\
    \end{aligned}
\end{equation}
\end{small}%
where $(i)$ follows by the first line of \eqref{eq:important_lm_1}, $(ii)$ follows by $\sqrt{a+b}\leq \sqrt{a}+\sqrt{b}$ for $a,b\geq0$, $(iii)$ follows since we have
\begin{small}
\begin{equation*}
    \begin{aligned}
        (1-\beta)\sum_{t=0}^{T-1}\beta^t\sum_{i=0}^{t}\left(\frac{\sqrt{1-\mu\alpha/2}}{\beta}\right)^i
        &\leq (1-\beta)\sum_{t=0}^{T-1}\beta^t\frac{\beta}{\beta-\sqrt{1-\mu\alpha/2}}
        \leq \frac{\beta}{\beta-\sqrt{1-\mu\alpha/2}}
        \leq \frac{\beta\left(\beta+\sqrt{1-\mu\alpha/2}\right)}{\mu\alpha/2-(1-\beta^2)} \\
        &\leq \frac{2}{\mu\alpha/2-(1-\beta)(1+\beta)}
        \leq \frac{2}{\mu\alpha/2-2(1-\beta)}
        = \frac{4}{\mu\alpha-4(1-\beta)},
    \end{aligned}
\end{equation*}
\end{small}%
and also by definition of $\Delta_{y,0}$ in \eqref{eq:ABdelta} and $\mu\alpha = 8(1-\beta)$ by \eqref{eq:lm_gamma+alpha+T} and hence $\mu\alpha-4(1-\beta)>0$.

Moreover, we have
\begin{sequation} \label{eq:second-part2}
    \begin{aligned}
        \frac{4\Delta_{y,0}}{T(\mu\alpha-4(1-\beta))} + \sqrt{\left(\frac{8\alpha \sigma_{g,1}^2}{\mu} + \frac{4\eta^2l_{g,1}^2}{\mu^4\alpha^2}\right)\log\left(\frac{eT}{\delta}\right)}
        &\quad\overset{(i)}{\leq} \frac{\Delta_{y,0}}{T(1-\beta)} + \sqrt{\frac{\eps^2}{512L_0^2}}
        \overset{(ii)}{=} \frac{\eta\eps\Delta_{y,0}}{4\Delta_0(1-\beta)} + \frac{\eps}{16\sqrt{2}L_0} \\
        &\overset{(iii)}{\leq} \frac{\eps}{32L_0} + \frac{\eps}{16\sqrt{2}L_0} 
        \leq \frac{3}{32L_0}\eps,
    \end{aligned}
\end{sequation}%
where $(i)$ follows from $\mu\alpha = 8(1-\beta)$, \eqref{eq:important_lm_1} and \eqref{eq:merge-step}, $(ii)$ follows from the choice of $T=4\Delta_0/(\eta\eps)$ and $(iii)$ follows from $\eta \leq \eps\Delta_0(1-\beta)/(8\Delta_{y,0}^2L_0^2) \leq \Delta_0(1-\beta)/(8\Delta_{y,0}L_0)$ by $\eps\leq \Delta_{y,0}L_0$ from \eqref{eq:eps}.

Combining \eqref{eq:second-part1} and \eqref{eq:second-part2} finally yields $(1-\beta)\sum_{t=0}^{T-1}\sum_{i=0}^{t}\beta^{t-i}\|y_i-y_i^*\| \leq \frac{3\eps}{32L_0}$.
\end{proof}

\subsection{Proof of Lemma~\ref{lm:maintext-lm4}}

In this section, we aim to leverage the cumulative moving average hypergradient estimation error, namely Lemma~\ref{lm:maintext-lm4}. To this end, we present Lemma~\ref{lm:sum-bound-z},~\ref{lm:hyper-var} and~\ref{lm:hypergradient-bias} to give upper bound for the cumulative error of the linear system estimator, the variance as well as the bias of the hypergradient estimator, respectively. It's worth noting that we can use Lemma~\ref{lm:maintext-lm3} and independence of filtration to handle the most difficult part in Lemma~\ref{lm:hypergradient-bias}, namely $L_{x,1}\|y_t-y_t^*\|\|\nabla\Phi(x_t)\|$.

\begin{lemma} \label{lm:sum-bound-z}
Under Assumptions~\ref{ass:relax-smooth},~\ref{ass:f-and-g} and event $\gE_{\init} \cap \gE_y$, if we choose 
$\gamma\leq\min\left\{1/(4\mu), \mu/(16\sigma_{g,2}^2), \alpha/8\right\}$, 
then we have the following estimate:
\begin{small}
\begin{equation}
    \begin{aligned}
        &\sum_{t=0}^{T-1}\E[\|z_t-z_t^*\|^2]
        \leq \frac{1}{\mu\gamma}\|z_0-z_0^*\|^2 + \frac{10(1-\mu\gamma)}{\mu^3(\alpha-2\gamma)}\left(\frac{l_{g,2}^2l_{f,0}^2}{\mu^2}+(L_{y,0}+L_{y,1}l_{f,0})^2\right)\|y_0-y_0^*\|^2 \\
        &\ \ + T\left\{\frac{2\gamma}{\mu}\left(\frac{2l_{f,0}^2}{\mu^2}\sigma_{g,2}^2+\sigma_{f,1}^2\right) + \frac{4l_{z^*}^2}{\mu^2}\frac{\eta^2}{\gamma^2} + \frac{5}{\mu^2}\left(\frac{l_{g,2}^2l_{f,0}^2}{\mu^2}+(L_{y,0}+L_{y,1}l_{f,0})^2\right) \left(\frac{8\alpha \sigma_{g,1}^2}{\mu} + \frac{4\eta^2l_{g,1}^2}{\mu^4\alpha^2}\right)\log\left(\frac{eT}{\delta}\right)\right\},
    \end{aligned}
\end{equation}
\end{small}%
where the expectation is taken over the randomness in $\widetilde{\gF}_T$.
\end{lemma}

\begin{proof}[Proof of Lemma~\ref{lm:sum-bound-z}]
Follow the same procedure as Lemma 13 in \cite{hao2024bilevel}, we have 
\begin{sequation*}
    \begin{aligned}
        \E_t[\|z_{t+1}-z_{t+1}^*\|^2] 
        \leq (1-\mu\gamma)\|z_t-z_t^*\|^2 + \frac{5\gamma}{\mu}&\left(\frac{l_{g,2}^2l_{f,0}^2}{\mu^2}+(L_{y,0}+L_{y,1}l_{f,0})^2\right)\|y_t-y_t^*\|^2 
        + 2\left(\frac{2l_{f,0}^2}{\mu^2}\sigma_{g,2}^2+\sigma_{f,1}^2\right)\gamma^2 + \frac{4}{\mu\gamma}l_{z^*}^2\eta^2. \\
    \end{aligned}
\end{sequation*}%
Under event $\gE_{\init} \cap \gE_y$, we have 
\begin{small}
\begin{equation*}
    \begin{aligned}
        \|y_t-y_t^*\|^2 \leq \left(1-\frac{\mu\alpha}{2}\right)^t\|y_0-y_0^*\|^2 + \left(\frac{8\alpha \sigma_{g,1}^2}{\mu} + \frac{4\eta^2l_{g,1}^2}{\mu^4\alpha^2}\right)\log\left(\frac{eT}{\delta}\right),
    \end{aligned}
\end{equation*}
\end{small}%
which gives
\begin{small}
\begin{equation*}
    \begin{aligned}
        \E[\|z_t&-z_t^*\|^2] 
        \leq (1-\mu\gamma)^{t}\|z_0-z_0^*\|^2 + \sum_{i=0}^{t-1}(1-\mu\gamma)^{t-i-1}\left\{2\left(\frac{2l_{f,0}^2}{\mu^2}\sigma_{g,2}^2+\sigma_{f,1}^2\right)\gamma^2 + \frac{4}{\mu\gamma}l_{z^*}^2\eta^2 \right.\\
        &+ \left. \frac{5\gamma}{\mu}\left(\frac{l_{g,2}^2l_{f,0}^2}{\mu^2}+(L_{y,0}+L_{y,1}l_{f,0})^2\right) \left[\left(1-\frac{\mu\alpha}{2}\right)^{i}\|y_0-y_0^*\|^2 + \left(\frac{8\alpha \sigma_{g,1}^2}{\mu} + \frac{4\eta^2l_{g,1}^2}{\mu^4\alpha^2}\right)\log\left(\frac{eT}{\delta}\right)\right]\right\}. \\
    \end{aligned}
\end{equation*}
\end{small}%
Taking summation yields
\begin{small}
\begin{equation*}
    \begin{aligned}
        \sum_{t=0}^{T-1}\E[\|z_t&-z_t^*\|^2] 
        \leq \sum_{t=0}^{T-1}(1-\mu\gamma)^{t}\|z_0-z_0^*\|^2 + \sum_{t=0}^{T-1}\sum_{i=0}^{t-1}(1-\mu\gamma)^{t-i-1}\left\{2\left(\frac{2l_{f,0}^2}{\mu^2}\sigma_{g,2}^2+\sigma_{f,1}^2\right)\gamma^2 + \frac{4}{\mu\gamma}l_{z^*}^2\eta^2 \right.\\
        &+ \left. \frac{5\gamma}{\mu}\left(\frac{l_{g,2}^2l_{f,0}^2}{\mu^2}+(L_{y,0}+L_{y,1}l_{f,0})^2\right) \left[\left(1-\frac{\mu\alpha}{2}\right)^{i}\|y_0-y_0^*\|^2 + \left(\frac{8\alpha \sigma_{g,1}^2}{\mu} + \frac{4\eta^2l_{g,1}^2}{\mu^4\alpha^2}\right)\log\left(\frac{eT}{\delta}\right)\right]\right\} \\
        \leq{}& \sum_{t=0}^{T-1}(1-\mu\gamma)^{t}\|z_0-z_0^*\|^2 + \sum_{t=0}^{T-1}\sum_{i=0}^{t-1}(1-\mu\gamma)^{t-i-1}\left\{2\left(\frac{2l_{f,0}^2}{\mu^2}\sigma_{g,2}^2+\sigma_{f,1}^2\right)\gamma^2 + \frac{4}{\mu\gamma}l_{z^*}^2\eta^2 \right.\\
        &+ \left. \frac{5\gamma}{\mu}\left(\frac{l_{g,2}^2l_{f,0}^2}{\mu^2}+(L_{y,0}+L_{y,1}l_{f,0})^2\right) \left(\frac{8\alpha \sigma_{g,1}^2}{\mu} + \frac{4\eta^2l_{g,1}^2}{\mu^4\alpha^2}\right)\log\left(\frac{eT}{\delta}\right)\right\} \\
        &+ \frac{5\gamma}{\mu}\left(\frac{l_{g,2}^2l_{f,0}^2}{\mu^2}+(L_{y,0}+L_{y,1}l_{f,0})^2\right)\sum_{t=0}^{T-1}\sum_{i=0}^{t-1}(1-\mu\gamma)^{t-i-1}\left(1-\frac{\mu\alpha}{2}\right)^i\|y_0-y_0^*\|^2. \\
    \end{aligned}
\end{equation*}
\end{small}%
Since $\gamma\leq \alpha/8$, then $\alpha-2\gamma > 0$, for any $t_0\in[T]$ we have 
\begin{small}
\begin{equation*}
    \begin{aligned}
        \sum_{t=0}^{T-1}\sum_{i=0}^{t-1}(1-\mu\gamma)^{t-i-1}\left(1-\frac{\mu\alpha}{2}\right)^i = \sum_{t=0}^{T-1}(1-\mu\gamma)^{t-1}\sum_{i=0}^{t-1}\left(\frac{1-\mu\alpha/2}{1-\mu\gamma}\right)^i \leq \sum_{t=0}^{T-1}(1-\mu\gamma)^{t-1}\frac{2(1-\mu\gamma)}{\mu(\alpha-2\gamma)} \leq \frac{2(1-\mu\gamma)}{\mu^2\gamma(\alpha-2\gamma)}.
    \end{aligned}
\end{equation*}
\end{small}%
Therefore, we conclude that 
\begin{small}
\begin{equation*}
    \begin{aligned}
        \sum_{t=0}^{T-1}&\E[\|z_t-z_t^*\|^2]
        \leq \frac{1}{\mu\gamma}\|z_0-z_0^*\|^2 + \frac{10(1-\mu\gamma)}{\mu^3(\alpha-2\gamma)}\left(\frac{l_{g,2}^2l_{f,0}^2}{\mu^2}+(L_{y,0}+L_{y,1}l_{f,0})^2\right)\|y_0-y_0^*\|^2 \\
        &+ T\left\{\frac{2\gamma}{\mu}\left(\frac{2l_{f,0}^2}{\mu^2}\sigma_{g,2}^2+\sigma_{f,1}^2\right) + \frac{4l_{z^*}^2}{\mu^2}\frac{\eta^2}{\gamma^2} + \frac{5}{\mu^2}\left(\frac{l_{g,2}^2l_{f,0}^2}{\mu^2}+(L_{y,0}+L_{y,1}l_{f,0})^2\right) \left(\frac{8\alpha \sigma_{g,1}^2}{\mu} + \frac{4\eta^2l_{g,1}^2}{\mu^4\alpha^2}\right)\log\left(\frac{eT}{\delta}\right)\right\}.
    \end{aligned}
\end{equation*}
\end{small}%
\end{proof}

\begin{lemma}[Variance, Lemma 14 in \cite{hao2024bilevel}] \label{lm:hyper-var}
Under Assumptions~\ref{ass:relax-smooth},~\ref{ass:f-and-g}, we have
\begin{sequation}
    \begin{aligned}
        \E[\|\hatphi(x_t,y_t,z_t;\xi_t',\zeta_t') - \E_t[\hatphi(x_t,y_t,z_t;\xi_t',\zeta_t')]\|^2] \leq \sigma_{f,1}^2 + \frac{2l_{f,0}^2}{\mu^2}\sigma_{g,2}^2 + 2\sigma_{g,2}^2\E[\|z_t-z_t^*\|^2], 
    \end{aligned}
\end{sequation}%
where we define $\hatphi(x_t,y_t,z_t;\xi_t',\zeta_t') = \gdx F(x_t,y_t;\xi_t') - \gdxy G(x_t,y_t;\zeta_t')z_k$ as the hypergradient estimator, and the total expectation is taken over the randomness in $\widetilde{\gF}_T$.
\end{lemma}

\begin{lemma}[Bias, Lemma 4 in \cite{hao2024bilevel}] \label{lm:hypergradient-bias}
Under Assumptions~\ref{ass:relax-smooth},~\ref{ass:f-and-g}, we have
\begin{sequation}
    \begin{aligned}
        \|\E_t[\hatphi(x_t,y_t,z_t;\xi_t',\zeta_t')] - \gdphi(x_t)\| \leq L_{x,1}\|y_t-y_t^*\|\|\gdphi(x_t)\| + \left(L_{x,0} + L_{x,1}\frac{l_{g,1}l_{f,0}}{\mu} + \frac{l_{g,2}l_{f,0}}{\mu}\right)\|y_t-y_t^*\| + l_{g,1}\|z_t-z_t^*\|.
    \end{aligned}
\end{sequation}%
\end{lemma}

With Lemma~\ref{lm:sum-bound-z},~\ref{lm:hyper-var} and~\ref{lm:hypergradient-bias}, we are now ready to leverage the cumulative estimation error for the hypergradient.

\begin{lemma}[Restatement of Lemma~\ref{lm:maintext-lm4}] \label{lm:moving-average error in expectation}
Under Assumptions~\ref{ass:relax-smooth},~\ref{ass:f-and-g} and event $\gE_{\init} \cap \gE_y$, define $\eps_t = m_{t+1} - \gdphi(x_t)$ to be the moving-average hypergradient estimation error. Then we have
\begin{sequation} \label{eq:moving-average error sum}
    \begin{aligned}
        \E\left[\sum_{t=0}^{T-1}\|\eps_t\|\right]
        \leq{}& \left(\frac{\eta L_1\beta}{1-\beta} + \frac{1}{8}\right)\sum_{t=0}^{T-1}\E\|\gdphi(x_t)\| + T\sqrt{1-\beta}\sqrt{\sigma_{f,1}^2 + \frac{2l_{f,0}^2}{\mu^2}\sigma_{g,2}^2} + \frac{T\eta L_0\beta}{1-\beta} \\
        &+ \sqrt{T}\left(\sqrt{2}\sigma_{g,2}\sqrt{1-\beta} + l_{g,1}\right)\sqrt{\sum_{t=0}^{T-1}\E[\|z_t-z_t^*\|^2]} + \frac{\beta}{1-\beta}\|m_0-\gdphi(x_0)\| \\
        &+ \left(L_{x,0} + L_{x,1}\frac{l_{g,1}l_{f,0}}{\mu} + \frac{l_{g,2}l_{f,0}}{\mu}\right)\left[\frac{4\Delta_{y,0}}{\mu\alpha-4(1-\beta)} + T\sqrt{\left(\frac{8\alpha \sigma_{g,1}^2}{\mu} + \frac{4\eta^2l_{g,1}^2}{\mu^4\alpha^2}\right)\log\left(\frac{eT}{\delta}\right)}\right], \\
    \end{aligned}
\end{sequation}%
where $\Delta_{y,0}$ is defined in \eqref{eq:ABdelta}, and the expectation is taken over the randomness in $\widetilde{\gF}_T$.
\end{lemma}

\begin{proof}[Proof of Lemma~\ref{lm:moving-average error in expectation}]
We define $\eps_t = m_{t+1} - \gdphi(x_t)$, $\hat{\eps}_t = \hatphi(x_t,y_t,z_t;\xi_t',\zeta_t') - \gdphi(x_t)$ and $S(a,b) = \gdphi(a) - \gdphi(b)$. It is easy to see that $\|S(a,b)\| \leq (L_0+L_1\|\gdphi(a)\|)\|a-b\|$ by $(L_0,L_1)$-smoothness of function $\Phi$. Specifically, $\|x_t-x_{t+1}\|=\eta$ holds true for any $t$ and thus $\|S(x_t,x_{t+1})\|\leq (L_0+L_1\|\gdphi(x_t)\|)\eta$. By definition of $\eps_t, \hat{\eps}_t$ and $S(a,b)$, we have the following recursion:
\begin{sequation} \label{eq:error-recursion}
    \eps_{t+1} = \beta\eps_t + \beta S(x_t,x_{t+1}) + (1-\beta)\hat{\eps}_{t+1}.
\end{sequation}%
Then we apply \eqref{eq:error-recursion} recursively and obtain
\begin{sequation*}
    \eps_t = \beta^{t+1}(m_0-\gdphi(x_0)) + \beta\sum_{i=0}^{t-1}\beta^{t-i-1}S(x_i,x_{i+1}) + (1-\beta)\sum_{i=0}^{t}\beta^{t-i}\hat{\eps}_{i},
\end{sequation*}%
which gives
\begin{small}
\begin{equation} \label{eq:estimation-recursion}
    \begin{aligned}
        \|\eps_t\| 
        &\leq \beta^{t+1}\|m_0-\gdphi(x_0)\| + \eta\beta\sum_{i=0}^{t-1}\beta^{t-i-1}(L_0+L_1\|\gdphi(x_i)\|) + (1-\beta)\left\|\sum_{i=0}^{t}\beta^{t-i}\hat{\eps}_{i}\right\|. \\
    \end{aligned}
\end{equation}
\end{small}%
Taking summation and total expectation yields
\begin{small}
    \begin{align}
        &\E\left[\sum_{t=0}^{T-1}\|\eps_t\|\right]
        \leq{} (1-\beta)\sum_{t=0}^{T-1}\E\left\|\sum_{i=0}^{t}\beta^{t-i}\hat{\eps}_{i}\right\| + \frac{\eta L_1\beta}{1-\beta}\sum_{t=0}^{T-1}\E\|\gdphi(x_t)\| + \frac{T\eta L_0\beta}{1-\beta} + \frac{\beta}{1-\beta}\|m_0-\gdphi(x_0)\| \notag \\
        &\leq (1-\beta)\sum_{t=0}^{T-1}\E\left\|\sum_{i=0}^{t}\beta^{t-i}\left(\hatphi(x_i,y_i,z_i;\xi_i',\zeta_i')-\E_i[\hatphi(x_i,y_i,z_i;\xi_i',\zeta_i')]\right)\right\| + \frac{\eta L_1\beta}{1-\beta}\sum_{t=0}^{T-1}\E\|\gdphi(x_t)\| \label{eq:error-bound1-1} \\
        &\quad+ (1-\beta)\sum_{t=0}^{T-1}\E\left\|\sum_{i=0}^{t}\beta^{t-i}\left(\E_{i}[\hatphi(x_i,y_i,z_i;\xi_i',\zeta_i')]-\gdphi(x_i)\right)\right\| 
        + \frac{T\eta L_0\beta}{1-\beta} + \frac{\beta}{1-\beta}\|m_0-\gdphi(x_0)\|. \label{eq:error-bound1-2}
    \end{align}
\end{small}%
For the first term of \eqref{eq:error-bound1-1}, we follow the same procedure as equation (68) of Lemma 5 in \cite{hao2024bilevel} and obtain
\begin{small}
\begin{equation} \label{eq:sum-hyper-var}
    \begin{aligned}
        (1-\beta)&\sum_{t=0}^{T-1}\E\left\|\sum_{i=0}^{t}\beta^{t-i}\left(\hatphi(x_i,y_i,z_i;\xi_i',\zeta_i')-\E_i[\hatphi(x_i,y_i,z_i;\xi_i',\zeta_i')]\right)\right\| \\
        &\quad\quad\quad \leq T\sqrt{1-\beta}\sqrt{\sigma_{f,1}^2 + \frac{2l_{f,0}^2}{\mu^2}\sigma_{g,2}^2} + \sqrt{2}\sigma_{g,2}\sqrt{T}\sqrt{1-\beta}\sqrt{\sum_{t=0}^{T-1}\E[\|z_t-z_t^*\|^2]}.
    \end{aligned}
\end{equation}
\end{small}%
For the first term of \eqref{eq:error-bound1-2}, by triangle inequality and Lemma~\ref{lm:hypergradient-bias} we have
\begin{small}
\begin{equation} \label{eq:sum-hyper-bias}
    \begin{aligned}
        &(1-\beta)\sum_{t=0}^{T-1}\E\left\|\sum_{i=0}^{t}\beta^{t-i}\left(\E_{i}[\hatphi(x_i,y_i,z_i;\xi_i',\zeta_i')]-\gdphi(x_i)\right)\right\| 
        \leq (1-\beta)\sum_{t=0}^{T-1}\sum_{i=0}^{t}\beta^{t-i}l_{g,1}\E[\|z_i-z_i^*\|] \\
        &\ \  + (1-\beta)\sum_{t=0}^{T-1}\sum_{i=0}^{t}\beta^{t-i}L_{x,1}\|y_i-y_i^*\|\E\|\gdphi(x_i)\| + (1-\beta)\sum_{t=0}^{T-1}\sum_{i=0}^{t}\beta^{t-i}\left(L_{x,0} + L_{x,1}\frac{l_{g,1}l_{f,0}}{\mu} + \frac{l_{g,2}l_{f,0}}{\mu}\right)\|y_i-y_i^*\|. \\
    \end{aligned}
\end{equation}
\end{small}%
Now we proceed to upper bound the right-hand side of \eqref{eq:sum-hyper-bias}, respectively.

For the first term on right-hand side of \eqref{eq:sum-hyper-bias}, we have
\begin{small}
\begin{equation} \label{eq:B.17-first}
    \begin{aligned}
        (1-\beta)\sum_{t=0}^{T-1}\sum_{i=0}^{t}\beta^{t-i}l_{g,1}\E[\|z_i-z_i^*\|]
        \leq \sum_{t=0}^{T-1}l_{g,1}\sqrt{\E[\|z_t-z_t^*\|^2]} \leq l_{g,1}\sqrt{T}\sqrt{\sum_{t=0}^{T-1}\E[\|z_t-z_t^*\|^2]}. \\
    \end{aligned}
\end{equation}
\end{small}%
For the second term on right-hand side of \eqref{eq:sum-hyper-bias}, we have
\begin{small}
\begin{equation} \label{eq:B.17-second}
    \begin{aligned}
        (1-\beta)\sum_{t=0}^{T-1}\sum_{i=0}^{t}\beta^{t-i}L_{x,1}\|y_i-y_i^*\|\E\|\gdphi(x_i)\| 
        \leq (1-\beta)\frac{L_{x,1}}{8L_1}\sum_{t=0}^{T-1}\sum_{i=0}^{t}\beta^{t-i}\E\|\gdphi(x_i)\|
        \leq \frac{1}{8}\sum_{t=0}^{T-1}\E\|\gdphi(x_t)\|, \\
    \end{aligned}
\end{equation}
\end{small}%
where we use $\|y_t-y_t^*\|\leq 1/(8L_1)$ for any $t\in[T]$ by Lemma~\ref{lm:high-prob 1/8L_1 bound for y} and the fact that $L_{x,1}\leq L_1$.

For the third term on right-hand side of \eqref{eq:sum-hyper-bias}, by \eqref{eq:second-part1} we have
\begin{small}
\begin{equation} \label{eq:B.17-third}
    \begin{aligned}
        (1-\beta)&\sum_{t=0}^{T-1}\sum_{i=0}^{t}\beta^{t-i}\left(L_{x,0} + L_{x,1}\frac{l_{g,1}l_{f,0}}{\mu} + \frac{l_{g,2}l_{f,0}}{\mu}\right)\|y_i-y_i^*\| \\
        &\leq \left(L_{x,0} + L_{x,1}\frac{l_{g,1}l_{f,0}}{\mu} + \frac{l_{g,2}l_{f,0}}{\mu}\right)\left[\frac{4\Delta_{y,0}}{\mu\alpha-4(1-\beta)} + T\sqrt{\left(\frac{8\alpha \sigma_{g,1}^2}{\mu} + \frac{4\eta^2l_{g,1}^2}{\mu^4\alpha^2}\right)\log\left(\frac{eT}{\delta}\right)}\right]. \\
    \end{aligned}
\end{equation}
\end{small}%

Plugging \eqref{eq:B.17-first}, \eqref{eq:B.17-second} and \eqref{eq:B.17-third} into \eqref{eq:sum-hyper-bias}, and then combining \eqref{eq:sum-hyper-bias} with \eqref{eq:sum-hyper-var} yields the upper bound:
\begin{small}
\begin{equation*}
    \begin{aligned}
        \E\left[\sum_{t=0}^{T-1}\|\eps_t\|\right]
        \leq{}& \left(\frac{\eta L_1\beta}{1-\beta} + \frac{1}{8}\right)\sum_{t=0}^{T-1}\E\|\gdphi(x_t)\| + T\sqrt{1-\beta}\sqrt{\sigma_{f,1}^2 + \frac{2l_{f,0}^2}{\mu^2}\sigma_{g,2}^2} + \frac{T\eta L_0\beta}{1-\beta} \\
        &+ \sqrt{T}\left(\sqrt{2}\sigma_{g,2}\sqrt{1-\beta} + l_{g,1}\right)\sqrt{\sum_{t=0}^{T-1}\E[\|z_t-z_t^*\|^2]} + \frac{\beta}{1-\beta}\|m_0-\gdphi(x_0)\| \\
        &+ \left(L_{x,0} + L_{x,1}\frac{l_{g,1}l_{f,0}}{\mu} + \frac{l_{g,2}l_{f,0}}{\mu}\right)\left[\frac{4\Delta_{y,0}}{\mu\alpha-4(1-\beta)} + T\sqrt{\left(\frac{8\alpha \sigma_{g,1}^2}{\mu} + \frac{4\eta^2l_{g,1}^2}{\mu^4\alpha^2}\right)\log\left(\frac{eT}{\delta}\right)}\right]. \\
    \end{aligned}
\end{equation*}
\end{small}%
\end{proof}

\subsection{Proof of Theorem~\ref{thm:main-thm}}

Before statement of Theorem~\ref{thm:main-thm}, we first modify Lemma~\ref{eq:technical-descent} to give a characterization for the objective function value decrease in terms of the true hypergradient $\nabla\Phi(x_t)$ and the moving average hypergradient estimation error $\eps_t$. Similar results also appear in \cite{jin2021non}, \cite{liu2023nearlyoptimal} and \cite{hao2024bilevel}.

\begin{lemma} \label{lm:descent-thm}
Consider an algorithm that starts at $x_0$ and updates via $x_{t+1}=x_t-\eta\frac{m_{t+1}}{\|m_{t+1}\|}$, where $\{m_t\}$ is any arbitrary sequence of points. Define $\eps_t \coloneqq m_{t+1}-\gdphi(x_t)$ to be the estimation error. Then for any $\eta$ satisfying
\begin{sequation} \label{eq:descent-eta}
    \eta \leq \frac{1}{\sqrt{2(1+l_{g,1}^2/\mu^2)(L_{x,1}^2+L_{y,1}^2)}},
\end{sequation}%
we have the following one-step improvement:
\begin{sequation} \label{eq:descent-one-step}
    \Phi(x_{t+1}) - \Phi(x_t) \leq -\left(\eta - \frac{1}{2}L_1\eta^2\right)\|\gdphi(x_t)\| + \frac{1}{2}L_0\eta^2 + 2\eta\|\eps_t\|.
\end{sequation}%
Moreover, by a telescope sum and total expectation (taken over the randomness in $\widetilde{\gF}_T$) we have
\begin{sequation} \label{eq:descent-sum-expect}
    \left(1-\frac{1}{2}\eta L_1\right)\frac{1}{T}\sum_{t=0}^{T-1}\E\|\gdphi(x_t)\| \leq \frac{\Phi(x_0)-\Phi(x_t)}{T\eta} + \frac{1}{2}\eta L_0 + \frac{2}{T}\E\left[\sum_{t=0}^{T-1}\|\eps_t\|\right].
\end{sequation}%
\end{lemma}

\begin{proof}[Proof of Lemma~\ref{lm:descent-thm}]
Since $\eta$ satisfies \eqref{eq:descent-eta}, we apply Lemma~\ref{eq:technical-rs-phi} and~\ref{eq:technical-descent} with $x=x_{t+1}$ and $x'=x_t$ to obtain
\begin{sequation} \label{eq:phi-recur}
    \begin{aligned}
        \Phi(x_{t+1}) &\leq \Phi(x_t) + \langle \gdphi(x_t),x_{t+1}-x_t \rangle + \frac{L_0+L_1\|\gdphi(x_t)\|}{2}\|x_{t+1}-x_t\|^2 \\
        &= \Phi(x_t) - \eta\langle \gdphi(x_t),\frac{m_{t+1}}{\|m_{t+1}\|} \rangle + \frac{1}{2}\eta^2(L_0+L_1\|\gdphi(x_t)\|) \\
        &= \Phi(x_t) - \eta\langle m_{t+1}-\eps_t,\frac{m_{t+1}}{\|m_{t+1}\|} \rangle + \frac{1}{2}\eta^2(L_0+L_1\|\gdphi(x_t)\|) \\
        &= \Phi(x_t) - \eta\|m_{t+1}\| + \eta\langle \eps_t,\frac{m_{t+1}}{\|m_{t+1}\|} \rangle + \frac{1}{2}\eta^2(L_0+L_1\|\gdphi(x_t)\|) \\
        &\overset{(i)}{\leq} \Phi(x_t) - \eta\|m_{t+1}\| + \eta\|\eps_t\| + \frac{1}{2}\eta^2(L_0+L_1\|\gdphi(x_t)\|) \\
        &\overset{(ii)}{\leq} \Phi(x_t) - \eta\|\gdphi(x_t)\| + 2\eta\|\eps_t\| + \frac{1}{2}\eta^2(L_0+L_1\|\gdphi(x_t)\|) \\
    \end{aligned}
\end{sequation}%
where we use Cauchy-Schwarz inequality for $(i)$ and $\|m_{t+1}\| = \|\gdphi(x_t)+\eps_t\| \geq \|\gdphi(x_t)\| - \|\eps_t\|$ for $(ii)$. Rearranging it gives \eqref{eq:descent-one-step}. Moreover, dividing $1/(T\eta)$ on both sides of \eqref{eq:descent-one-step}, then taking telescope sum and total expectation yields \eqref{eq:descent-sum-expect}.
\end{proof}

With Lemma~\ref{lm:moving-average error in expectation} and~\ref{lm:descent-thm}, we are now ready to prove Theorem~\ref{thm:main-thm}.

\begin{theorem}[Restatement of Theorem~\ref{thm:main-thm}] \label{thm:appendix-main}
Suppose Assumptions~\ref{ass:relax-smooth} and~\ref{ass:f-and-g} hold. Let $\{x_t\}$ be the iterates produced by Algorithm~\ref{alg:bilevel}. For any given $\delta\in(0,1)$ and any small $\eps$ satisfying
\begin{sequation} \label{eq:eps}
    \begin{aligned}
        \eps \leq \min&\left\{\frac{L_0}{L_1}, \Delta_{y,0}L_0, \frac{8l_{g,1}L_0}{\mu\sqrt{2(1+l_{g,1}^2/\mu^2)(L_{x,1}^2+L_{y,1}^2)}}, \sqrt{\frac{16el_{g,1}\Delta_0L_0}{\mu\delta}}, 4\left(\frac{el_{g,1}\Delta_0L_0^3\sigma_{g,1}^2}{\mu^3\delta}\right)^{1/4}; \right. \\
        &\left.  \sqrt{\frac{32el_{g,1}\Delta_0L_0}{\mu\delta\exp(\mu/(2l_{g,1}))}}, \sqrt{\frac{32el_{g,1}\Delta_0L_0}{\mu\delta\exp(\mu l_{g,1}\Delta_{z,0}/(2\Delta_0L_0))}}, \sqrt{\frac{32el_{g,1}\Delta_0L_0}{\mu\delta\exp(\mu\Delta_{y,0}^2L_0/(2l_{g,1}\Delta_0))}}, \right. \\
        &\left. \frac{\Delta_0L_0}{\|\gdphi(x_0)\|}, \frac{L_0\sigma_{g,1}}{\sigma_{g,2}}, \frac{L_0\sigma_{g,1}}{\sqrt{\mu l_{g,1}}}, \left(\frac{2^{21}el_{g,1}\Delta_0L_0^3\sigma_{g,1}^2}{\mu^3\delta}\right)^{1/4}\exp\left(\frac{-l_{g,1}\sqrt{\sigma_{f,1}^2 + 2l_{f,0}^2\sigma_{g,2}^2/\mu^2}}{512L_0\sigma_{g,1}}\right)
        \right\},
    \end{aligned}
\end{sequation}%
if we choose parameters $\alpha, \beta, \gamma, \eta$ as
\begin{sequation} \label{eq:beta}
    \begin{aligned}
        1-\beta = \min&\left\{1, \frac{\mu}{16l_{g,1}}, \frac{16el_{g,1}\Delta_0L_0}{\mu\delta\eps^2}, \frac{\mu^2\eps^2}{64\cdot1024L_0^2\sigma_{g,1}^2\log^2(B)}; \frac{\min\{1, \mu^2/(32l_{g,1}^2)\}}{\sigma_{f,1}^2 + 2l_{f,0}^2\sigma_{g,2}^2/\mu^2}\eps^2, \frac{l_{g,1}^2}{8\sigma_{g,2}^2}, \frac{\mu^2}{16\sigma_{g,2}^2}; \right. \\
        &\left. \frac{32el_{g,1}\Delta_0L_0}{\mu\delta\eps^2\exp(\mu/(2l_{g,1}))}, \frac{32el_{g,1}\Delta_0L_0}{\mu\delta\eps^2\exp(\mu l_{g,1}\Delta_{z,0}^2/(2\Delta_0L_0))}, \frac{32el_{g,1}\Delta_0L_0}{\mu\delta\eps^2\exp(\mu\Delta_{y,0}^2L_0/(2l_{g,1}\Delta_0))}
        \right\}
    \end{aligned}
\end{sequation}%
\begin{sequation} \label{eq:eta}
    \begin{aligned}
        \eta = \min\left\{\frac{1}{8}\min\left(\frac{1}{L_1}, \frac{\eps}{L_0}, \frac{\eps\Delta_0}{\Delta_{y,0}^2L_0^2}, \frac{\Delta_0}{\|\gdphi(x_0)\|}, \frac{\eps\Delta_0}{l_{g,1}^2\Delta_{z,0}^2}, \frac{\mu\eps}{l_{g,1}L_0\log(A)}\right)(1-\beta), \frac{1}{\sqrt{2(1+l_{g,1}^2/\mu^2)(L_{x,1}^2+L_{y,1}^2)}}\right\},
    \end{aligned}
\end{sequation}%
\begin{sequation} \label{eq:gamma+alpha+T}
    \begin{aligned}
        \gamma = \frac{1}{\mu}(1-\beta), \quad \alpha = 8\gamma, 
    \end{aligned}
\end{sequation}%
where $\Delta_0, \Delta_{y,0}, \Delta_{z,0}, A, B$ are defined in \eqref{eq:ABdelta}, then with probability at least $1-2\delta$ over the randomness in $\sigma(\widetilde{\gF}_{T_0}^1 \cup \mathcal{F}_T^1)$, Algorithm~\ref{alg:bilevel} guarantees $\frac{1}{T}\sum_{t=0}^{T-1}\E\|\gdphi(x_t)\| \leq 14\eps$ with at most $T = \frac{4\Delta_0}{\eta\eps}$ iterations, where the expectation is taken over over the randomness in $\widetilde{\mathcal{F}}_T$.
\end{theorem}

\begin{proof}[Proof of Theorem~\ref{thm:appendix-main}]
Before we begin the proof, let's first briefly describe the simple motivation behind the seemingly complex choice of parameters: we aim to choose $\eps, \alpha, \beta,\gamma, \eta$ carefully such that all of the conditions needed for Lemma~\ref{eq:technical-rs-phi},~\ref{eq:technical-descent},~\ref{lm:sum-bound-z} and~\ref{lm:high-prob 1/8L_1 bound for y} hold. Especially for Lemma~\ref{lm:high-prob 1/8L_1 bound for y}, we need to choose suitable $\beta$ and $\eps$ such that the following two terms in \eqref{eq:eta} and \eqref{eq:beta} dominant (please check Lemma~\ref{lm:verify} for more details):
\begin{sequation}
    \begin{aligned}
        \eta = \frac{\mu\eps}{8l_{g,1}L_0\log(A)}(1-\beta)
        \quad\quad \text{and} \quad\quad
        1-\beta = \frac{\mu^2\eps^2}{64\cdot1024L_0^2\sigma_{g,1}^2\log^2(B)},
    \end{aligned}
\end{sequation}%
and this is primarily where those ``lengthy" formulas come from. In addition, we also need to choose proper $\beta,\eta$ and total number of iterations $T$ such that \eqref{eq:main-line-2}, \eqref{eq:main-line-3} and \eqref{eq:main-line-4} (those three terms are mainly form Lemma~\ref{lm:moving-average error in expectation}) can be controlled and hence small enough to guarantee the convergence of Algorithm~\ref{alg:bilevel}. With this in hand, now we start proof.

By Lemma~\ref{lm:descent-thm} and definition \eqref{eq:ABdelta} of $\Delta_0$, we have the following estimate:
\begin{sequation} \label{eq:aaa}
    \left(1-\frac{1}{2}\eta L_1\right)\frac{1}{T}\sum_{t=0}^{T-1}\E\|\gdphi(x_t)\| \leq \frac{\Delta_0}{T\eta} + \frac{1}{2}\eta L_0 + \frac{2}{T}\E\left[\sum_{t=0}^{T-1}\|\eps_t\|\right].
\end{sequation}%
Under event $\gE_{\init}\cap\gE_y$, plug \eqref{eq:moving-average error sum} into the above inequality, then we have 
\begin{small}
    \begin{align}
        &\left(1-\left(\frac{1}{2} + \frac{2\beta}{1-\beta}\right)\eta L_1 - \frac{1}{4}\right) \frac{1}{T}\sum_{t=0}^{T-1}\E\|\gdphi(x_t)\| \label{eq:main-line-1} \\
        &\quad\quad\leq 2\sqrt{1-\beta}\sqrt{\sigma_{f,1}^2 + \frac{2l_{f,0}^2}{\mu^2}\sigma_{g,2}^2} + \frac{2\eta L_0\beta}{1-\beta} + \frac{\Delta_0}{T\eta} + \frac{1}{2}\eta L_0 + \frac{2\beta}{T(1-\beta)}\|m_0-\gdphi(x_0)\| \label{eq:main-line-2} \\
        &\quad\quad\quad+ 2\left(L_{x,0} + L_{x,1}\frac{l_{g,1}l_{f,0}}{\mu} + \frac{l_{g,2}l_{f,0}}{\mu}\right)\left[\frac{4\Delta_{y,0}}{T(\mu\alpha-4(1-\beta))} + \sqrt{\left(\frac{8\alpha \sigma_{g,1}^2}{\mu} + \frac{4\eta^2l_{g,1}^2}{\mu^4\alpha^2}\right)\log\left(\frac{eT}{\delta}\right)}\right] \label{eq:main-line-3} \\
        &\quad\quad\quad+ 2\left(\sqrt{2}\sigma_{g,2}\sqrt{1-\beta} + l_{g,1}\right)\sqrt{\frac{1}{T}\sum_{t=0}^{T-1}\E[\|z_t-z_t^*\|^2]} \label{eq:main-line-4}.
    \end{align}
\end{small}%
Now we proceed to bound \eqref{eq:main-line-1}, \eqref{eq:main-line-2}, \eqref{eq:main-line-3} and \eqref{eq:main-line-4}, respectively.

For left-hand side \eqref{eq:main-line-1} of the inequality, we have 
\begin{sequation} \label{eq:main-line-1-bound}
    \left(1-\left(\frac{1}{2} + \frac{2\beta}{1-\beta}\right)\eta L_1 - \frac{1}{4}\right) \frac{1}{T}\sum_{t=0}^{T-1}\E\|\gdphi(x_t)\| \geq \frac{1}{2T}\sum_{t=0}^{T-1}\E\|\gdphi(x_t)\|,
\end{sequation}%
since the following holds
\begin{sequation*}
    1-\left(\frac{1}{2} + \frac{2\beta}{1-\beta}\right)\eta L_1 - \frac{1}{4} = 1 - \frac{1+3\beta}{2(1-\beta)}\eta L_1 - \frac{1}{4} \geq 1 - \frac{2\eta L_1}{1-\beta} - \frac{1}{4} \geq \frac{1}{2},
\end{sequation*}%
where we use $\beta\leq 1$ for the first inequality and $\eta \leq (1-\beta)/(8L_1)$ for the second inequality.

For the first term \eqref{eq:main-line-2} on right-hand side of the inequality, we have
\begin{sequation} \label{eq:main-line-2-bound}
    \begin{aligned}
        2\sqrt{1-\beta}&\sqrt{\sigma_{f,1}^2 + 2\frac{2l_{f,0}^2}{\mu^2}\sigma_{g,2}^2} + \frac{2\eta L_0\beta}{1-\beta} + \frac{\Delta_0}{T\eta} + \frac{1}{2}\eta L_0 + \frac{2\beta}{T(1-\beta)}\|m_0-\gdphi(x_0)\| \\
        &\overset{(i)}{\leq} 2\eps + \frac{1}{4}\eps + \frac{1}{4}\eps + \frac{1}{16}\eps + \frac{\eta\eps\beta}{2(1-\beta)\Delta_0}\|\gdphi(x_0)\|
        \overset{(ii)}{\leq} 2\eps + \frac{1}{4}\eps + \frac{1}{4}\eps + \frac{1}{16}\eps + \frac{1}{16}\eps = \frac{21}{8}\eps.
    \end{aligned}
\end{sequation}%
where $(i)$ and $(ii)$ follow from the choice of $\beta$ and $\eta$ as in \eqref{eq:beta} and \eqref{eq:eta} that 
\begin{sequation*}
    1-\beta \leq \frac{\eps^2}{\sigma_{f,1}^2+2l_{f,0}^2\sigma_{g,2}^2/\mu^2}, \quad
    \eta \leq \frac{\eps}{8L_0}(1-\beta), \quad
    T = \frac{4\Delta_0}{\eta\eps}, \quad
    m_0 = 0; \quad
    \eta \leq \frac{\Delta_0}{8\|\gdphi(x_0)\|}(1-\beta).
\end{sequation*}%
For the second term \eqref{eq:main-line-3} on right-hand side of the inequality, we have
\begin{sequation} \label{eq:main-line-3-bound}
    \begin{aligned}
        &2\left(L_{x,0} + L_{x,1}\frac{l_{g,1}l_{f,0}}{\mu} + \frac{l_{g,2}l_{f,0}}{\mu}\right)\left[\frac{4\Delta_{y,0}}{T(\mu\alpha-4(1-\beta))} + \sqrt{\left(\frac{8\alpha \sigma_{g,1}^2}{\mu} + \frac{4\eta^2l_{g,1}^2}{\mu^4\alpha^2}\right)\log\left(\frac{eT}{\delta}\right)}\right] \\
        &\quad\quad\overset{(i)}{\leq} 2L_0\left(\frac{\Delta_{y,0}}{T(1-\beta)} + \sqrt{\frac{\eps^2}{512L_0^2}}\right) 
        \overset{(ii)}{=} 2L_0\left(\frac{\eta\eps\Delta_{y,0}}{4\Delta_0(1-\beta)} + \frac{\eps}{16\sqrt{2}L_0}\right) 
        \overset{(iii)}{\leq} 2L_0\left(\frac{\eps}{32L_0} + \frac{\eps}{16\sqrt{2}L_0}\right) 
        \leq \frac{3}{16}\eps,
    \end{aligned}
\end{sequation}%
where $(i)$ follows from $\mu\alpha = 8(1-\beta)$ and \eqref{eq:important_lm_1}, \eqref{eq:merge-step} in Lemma~\ref{lm:high-prob 1/8L_1 bound for y}, $(ii)$ follows from the choice of $T=4\Delta_0/(\eta\eps)$ and $(iii)$ follows from $\eta \leq \eps\Delta_0(1-\beta)/(8\Delta_{y,0}^2L_0^2) \leq \Delta_0(1-\beta)/(8\Delta_{y,0}L_0)$ by $\eps\leq \Delta_{y,0}L_0$ from \eqref{eq:eps}.

For the third term \eqref{eq:main-line-4} on right-hand side of the inequality, we have
\begin{sequation} \label{eq:main-line-4-bound}
    \begin{aligned}
        &2\left(\sqrt{2}\sigma_{g,2}\sqrt{1-\beta} + l_{g,1}\right)\sqrt{\frac{1}{T}\sum_{t=0}^{T-1}\E[\|z_t-z_t^*\|^2]}
        \leq 3l_{g,1}\sqrt{\frac{1}{T}\sum_{t=0}^{T-1}\E[\|z_t-z_t^*\|^2]} \\
        &\leq 3l_{g,1}\sqrt{\frac{1}{T\mu\gamma}\|z_0-z_0^*\|^2 + \frac{10(1-\mu\gamma)}{T\mu^3(\alpha-2\gamma)}\left(\frac{l_{g,2}^2l_{f,0}^2}{\mu^2}+(L_{y,0}+L_{y,1}l_{f,0})^2\right)\|y_0-y_0^*\|^2} \\
        &\quad+ 3l_{g,1}\sqrt{\frac{2\gamma}{\mu}\left(\frac{2l_{f,0}^2}{\mu^2}\sigma_{g,2}^2+\sigma_{f,1}^2\right) + \frac{4l_{z^*}^2}{\mu^2}\frac{\eta^2}{\gamma^2} + \frac{5}{\mu^2}\left(\frac{l_{g,2}^2l_{f,0}^2}{\mu^2}+(L_{y,0}+L_{y,1}l_{f,0})^2\right) \left(\frac{8\alpha \sigma_{g,1}^2}{\mu} + \frac{4\eta^2l_{g,1}^2}{\mu^4\alpha^2}\right)\log\left(\frac{eT}{\delta}\right)} \\
        &\overset{(i)}{\leq} 3\sqrt{\frac{l_{g,1}^2}{T(1-\beta)}\|z_0-z_0^*\|^2 + \frac{5(1-\mu\gamma)}{3T(1-\beta)}\left[\frac{l_{g,1}^2}{\mu^2}\left(\frac{l_{g,2}^2l_{f,0}^2}{\mu^2}+(L_{y,0}+L_{y,1}l_{f,0})^2\right)\right]\|y_0-y_0^*\|^2} \\
        &\quad+ 3\sqrt{\frac{2l_{g,1}^2}{\mu^2}\left(\frac{2l_{f,0}^2}{\mu^2}\sigma_{g,2}^2+\sigma_{f,1}^2\right)(1-\beta) + \frac{4l_{g,1}^2l_{z^*}^2}{\mu^2}\frac{\eta^2}{\gamma^2} + \left[\frac{5l_{g,1}^2}{\mu^2}\left(\frac{l_{g,2}^2l_{f,0}^2}{\mu^2}+(L_{y,0}+L_{y,1}l_{f,0})^2\right)\right] \left(\frac{8\alpha \sigma_{g,1}^2}{\mu} + \frac{4\eta^2l_{g,1}^2}{\mu^4\alpha^2}\right)\log\frac{eT}{\delta}} \\
        &\overset{(ii)}{\leq} 3\sqrt{\frac{l_{g,1}^2\eta\eps\Delta_{z,0}^2}{4\Delta_0(1-\beta)} + \frac{5\eta\eps L_0^2\Delta_{y,0}^2}{12\Delta_0(1-\beta)}}
        + 3\sqrt{\frac{1}{16}\eps^2 + \frac{4L_0^2\eta^2}{(1-\beta)^2} + 5L_0^2 \frac{\eps^2}{128L_0^2}}
        \overset{(iii)}{\leq} 3\sqrt{\frac{\eps^2}{32} + \frac{5\eps^2}{96}} + 3\sqrt{\frac{\eps^2}{16} + \frac{\eps^2}{16} + \frac{5\eps^2}{128}} \leq 4\eps, \\
    \end{aligned}
\end{sequation}%
where $(i)$ follows from $\mu\gamma = 1-\beta$ and $\alpha = 8\gamma$, $(ii)$ follows from \eqref{eq:important_lm_1}, \eqref{eq:merge-step} in Lemma~\ref{lm:high-prob 1/8L_1 bound for y} and the fact that
\begin{sequation*}
    T = \frac{4\Delta_0}{\eta\eps}, \quad
    \frac{l_{g,1}^2}{\mu^2}\left(\frac{l_{g,2}^2l_{f,0}^2}{\mu^2}+(L_{y,0}+L_{y,1}l_{f,0})^2\right) \leq L_0^2, \quad
    1-\beta \leq \frac{\mu^2\eps^2}{32l_{g,1}^2(\sigma_{f,1}^2 + 2l_{f,0}^2\sigma_{g,2}^2/\mu^2)}, \quad
    l_{g,1}^2l_{z^*}^2 \leq L_0^2, \quad
    \mu\gamma = 1-\beta,
\end{sequation*}%
and $(iii)$ follows from the choice of $\eta$ as in \eqref{eq:eta}:
\begin{sequation*}
    \eta \leq \frac{\eps\Delta_0}{8l_{g,1}^2\Delta_{z,0}^2}(1-\beta), \quad
    \eta \leq \frac{\eps\Delta_0}{8\Delta_{y,0}^2L_0^2}(1-\beta), \quad
    \eta \leq \frac{\eps}{8L_0}(1-\beta).
\end{sequation*}%
Combining \eqref{eq:main-line-1-bound}, \eqref{eq:main-line-2-bound}, \eqref{eq:main-line-3-bound} and \eqref{eq:main-line-4-bound} together yields
\begin{sequation*}
    \frac{1}{T}\sum_{t=0}^{T-1}\E\|\gdphi(x_t)\| \leq 2\left(\frac{21}{8}\eps + \frac{3}{16}\eps + 4\eps\right) \leq 14\eps.
\end{sequation*}%
Also note that
\begin{small}
\begin{equation*}
    \Pr(\gE_{\init}\cap\gE_y) = \Pr(\gE_y \mid \gE_{\init}) \cdot \Pr(\gE_{\init}) \geq (1-\delta)^2 \geq 1-2\delta.
\end{equation*}
\end{small}%
Therefore, with probability at least $1-2\delta$ over the randomness in $\mathcal{F}_T^1$, we have $\frac{1}{T}\sum_{t=0}^{T-1}\E\|\gdphi(x_t)\|\leq 14\eps$, where the expectation is taken over over the randomness in $\widetilde{\mathcal{F}}_T$.
\end{proof}

\subsection{Omitted Proofs in Lemma~\ref{lm:high-prob 1/8L_1 bound for y} and Theorem~\ref{thm:appendix-main}}

\begin{lemma} \label{lm:verify}
Under the same parameter choice in Theorem~\ref{thm:appendix-main}, we have the following facts:
\begin{small}
\begin{equation*}
    \eta = \frac{\mu\eps}{8l_{g,1}L_0\log(A)}(1-\beta)
    \quad\quad \text{and} \quad\quad
    1-\beta = \frac{\mu^2\eps^2}{64\cdot1024L_0^2\sigma_{g,1}^2\log^2(B)},
\end{equation*}
\end{small}%
where $A$ and $B$ are defined in \eqref{eq:ABdelta}.
\end{lemma}

\begin{proof}[Proof of Lemma~\ref{lm:verify}]
Let us verify this fact respectively.

\paragraph{Verification for $\eta$.} First, we have
\begin{small}
\begin{equation*}
    \eps \leq \min\left\{\frac{L_0}{L_1}, \frac{\Delta_0L_0}{\|\gdphi(x_0)\|}, \frac{8l_{g,1}L_0}{\mu\sqrt{2(1+l_{g,1}^2/\mu^2)(L_{x,1}^2+L_{y,1}^2)}}\right\}
\end{equation*}
\end{small}%
which implies
\begin{small}
\begin{equation*}
    \frac{1}{8L_1}(1-\beta) \leq \frac{\eps}{8L_0}(1-\beta),
    \quad
    \frac{\Delta_0}{8\|\gdphi(x_0)\|} \leq \frac{\eps}{8L_0}(1-\beta),
    \quad
    \frac{1}{\sqrt{2(1+l_{g,1}^2/\mu^2)(L_{x,1}^2+L_{y,1}^2)}} \leq \frac{\eps}{8L_0}(1-\beta).
\end{equation*}
\end{small}%
Also we have 
\begin{small}
\begin{equation*}
    1-\beta \leq \min\left\{\frac{32el_{g,1}\Delta_0L_0}{\mu\delta\eps^2\exp(\mu/(2l_{g,1}))}, \frac{32el_{g,1}\Delta_0L_0}{\mu\delta\eps^2\exp(\mu l_{g,1}\Delta_{z,0}^2/(2\Delta_0L_0))}, \frac{32el_{g,1}\Delta_0L_0}{\mu\delta\eps^2\exp(\mu\Delta_{y,0}^2L_0/(2l_{g,1}\Delta_0))}\right\},
\end{equation*}
\end{small}%
which implies that
\begin{small}
\begin{equation*}
    \frac{\eps}{8L_0}(1-\beta) \leq \frac{\mu\eps}{8l_{g,1}L_0\log(A)}(1-\beta),
    \quad
    \frac{\eps\Delta_0}{8\Delta_{y,0}^2L_0^2}(1-\beta) \leq \frac{\mu\eps}{8l_{g,1}L_0\log(A)}(1-\beta),
    \quad
    \frac{\eps\Delta_0}{l_{g,1}^2\Delta_{z,0}^2}(1-\beta) \leq \frac{\mu\eps}{8l_{g,1}L_0\log(A)}(1-\beta).
\end{equation*}
\end{small}%
Therefore, we conclude that 
\begin{small}
\begin{equation*}
    \eta = \frac{\mu\eps}{8l_{g,1}L_0\log(A)}(1-\beta).
\end{equation*}
\end{small}%

\paragraph{Verification for $1-\beta$.} First, we have
\begin{small}
\begin{equation*}
    \eps \leq \min\left\{\sqrt{\frac{16el_{g,1}\Delta_0L_0}{\mu\delta}}, \sqrt{\frac{32el_{g,1}\Delta_0L_0}{\mu\delta\exp(\mu/(2l_{g,1}))}}, \sqrt{\frac{32el_{g,1}\Delta_0L_0}{\mu\delta\exp(\mu l_{g,1}\Delta_{z,0}/(2\Delta_0L_0))}}, \sqrt{\frac{32el_{g,1}\Delta_0L_0}{\mu\delta\exp(\mu\Delta_{y,0}^2L_0/(2l_{g,1}\Delta_0))}}\right\},
\end{equation*}
\end{small}%
which implies that
\begin{small}
\begin{equation*}
    \frac{16el_{g,1}\Delta_0L_0}{\mu\delta\eps^2} \geq 1, 
    \quad
    \frac{32el_{g,1}\Delta_0L_0}{\mu\delta\eps^2\exp(\mu/(2l_{g,1}))} \geq 1,
    \quad
    \frac{32el_{g,1}\Delta_0L_0}{\mu\delta\eps^2\exp(\mu l_{g,1}\Delta_{z,0}^2/(2\Delta_0L_0))} \geq 1, 
    \quad
    \frac{32el_{g,1}\Delta_0L_0}{\mu\delta\eps^2\exp(\mu\Delta_{y,0}^2L_0/(2l_{g,1}\Delta_0))} \geq 1.
\end{equation*}
\end{small}%
Also, we have
\begin{small}
\begin{equation*}
    \eps \leq \min\left\{4\left(\frac{el_{g,1}\Delta_0L_0^3\sigma_{g,1}^2}{\mu^3\delta}\right)^{1/4}, \frac{L_0\sigma_{g,1}}{\sigma_{g,2}}, \frac{L_0\sigma_{g,1}}{\sqrt{\mu l_{g,1}}}\right\},
\end{equation*}
\end{small}%
which implies that
\begin{small}
\begin{equation*}
    B = \left(\frac{2^{21}el_{g,1}\Delta_0L_0^3\sigma_{g,1}^2}{\mu^3\delta\eps^4}\right)^4 \geq 4
    \quad \Longrightarrow \quad
    \frac{\mu^2\eps^2}{64\cdot1024L_0^2\sigma_{g,1}^2\log^2(B)} \leq \frac{\mu^2}{16\sigma_{g,2}^2} \leq \frac{l_{g,1}^2}{8\sigma_{g,2}^2},
    \quad
    \frac{\mu^2\eps^2}{64\cdot1024L_0^2\sigma_{g,1}^2\log^2(B)} \leq \frac{\mu}{16l_{g,1}} < 1.
\end{equation*}
\end{small}%
Finally,
\begin{small}
\begin{equation*}
    \eps \leq \left(\frac{2^{21}el_{g,1}\Delta_0L_0^3\sigma_{g,1}^2}{\mu^3\delta}\right)^{1/4}\exp\left(\frac{-l_{g,1}\sqrt{\sigma_{f,1}^2 + 2l_{f,0}^2\sigma_{g,2}^2/\mu^2}}{512L_0\sigma_{g,1}}\right)
\end{equation*}
\end{small}%
implies that
\begin{small}
\begin{equation*}
    \frac{\mu^2\eps^2}{64\cdot1024L_0^2\sigma_{g,1}^2\log^2(B)} \leq \frac{\mu^2/(32l_{g,1}^2)}{\sigma_{f,1}^2 + 2l_{f,0}^2\sigma_{g,2}^2/\mu^2}\eps^2 = \frac{\min\{1, \mu^2/(32l_{g,1}^2)\}}{\sigma_{f,1}^2 + 2l_{f,0}^2\sigma_{g,2}^2/\mu^2}\eps^2.
\end{equation*}
\end{small}%
Therefore, we conclude that 
\begin{small}
\begin{equation*}
    1-\beta = \frac{\mu^2\eps^2}{64\cdot1024L_0^2\sigma_{g,1}^2\log^2(B)}.
\end{equation*}
\end{small}
\end{proof}


\section{Omitted Proofs in Section~\ref{sec:high-prob}} \label{sec:high-prob-proof}
\setcounter{equation}{0}
\renewcommand{\theequation}{E.\arabic{equation}}

\textbf{Remark:} In this section, for the high probability proof, by a slight abuse of notation, we use $\E_t$ to denote the conditional expectation $\E[ \cdot \mid \gF_t^3]$.

\subsection{Justification for Assumption~\ref{ass:highprob}} \label{sec:appendix-justify}

In this section, we provide justification for the last statement of Assumption~\ref{ass:highprob}. 
One example satisfying this assumption is that the random noise $\zeta$ is chosen based on the information of $z$. This makes sense in our setting because our algorithm access $x,y,z$ first and then sample the random data to construct stochastic estimators. For example, we can choose $\zeta$ based on the following formula:
\begin{equation}
\label{eq:exampleassumption}
    \|\gdyy G(x,y;\zeta)-\gdyy g(x,y)\| \stackrel{d}{=} \frac{\tau}{\|z\|},
\end{equation}
where $\stackrel{d}{=}$ means that the LHS and RHS have the same distribution, and $\tau$ can be any one dimensional bounded random variable, for instance, $\tau$ has a truncated normal (Gaussian) distribution lies within the interval $(-\sigma_z,\sigma_z)$, then we have

\begin{equation*}
    \forall \zeta,z, \ \|(\gdyy G(x,y;\zeta)-\gdyy g(x,y))z\| \leq \sigma_z.
\end{equation*}

A specific example satisfying~\eqref{eq:exampleassumption} is the following. Define the noise structure as $\gdyy G(x,y;\zeta)=\gdyy g(x,y)+\Gamma$, where $\Gamma=\text{diag}(0,0,\ldots,\tau/\|z\|)$. 







\subsection{Tracking the linear system solution: high-probability guarantees}

In this section we follow the similar techniques as in Section~\ref{sec:track-y} to provide one-step improvement, distance recursion and distance tracking with high probability for estimator of the linear system solution, corresponding to Lemma~\ref{lm:z-one-step improvement},~\ref{lm:z-distance recursion} and~\ref{lm:z-high-prob distance tracking}, respectively. In particular, the proofs in this section are more involved than those in Section~\ref{sec:track-y} since we also need to handle error terms introduced by the lower-level variable besides variance and distribution drift. The key ideas are (i): to choose proper coefficient to apply Young's inequality and (ii): to make use of ``good event'', namely $\gE_{\init}\cap\gE_y$, to bound the additional error terms. 

Now we start to prove one-step improvement and distance recursion for linear system estimator $z_t$.

\begin{lemma}[One-step improvement] \label{lm:z-one-step improvement}
Consider Algorithm~\ref{alg:bilevel} with sequence $\{z_t\}$ and constant learning rate $\gamma \leq 1/(4l_{g,1})$, then for any $z$ and $t\geq1$, we have the following estimate:
\begin{equation}
    \begin{aligned}
        2\gamma(h(x_t,z_{t+1}) &- h(x_t,z))
        \leq (1-\mu\gamma)\|z_t-z\|^2 - \|z_{t+1}-z\|^2 + 2\gamma\langle v_t, z_t-z \rangle + \frac{2\gamma^2}{1-2l_{g,1}\gamma}\|v_t\|^2 \\
        &+ 8\gamma^2 l_{g,2}^2\|y_t-y_t^*\|^2\|z_t-z_t^*\|^2 +  8\gamma^2\frac{l_{g,2}^2l_{f,0}^2}{\mu^2}\|y_t-y_t^*\|^2 + 4\gamma^2(L_{y,0}+L_{y,1}l_{f,0})^2\|y_t-y_t^*\|^2,  \\
    \end{aligned}
\end{equation}
where we define $v_t$ as 
\begin{equation*}
    \begin{aligned}
        v_t = [\gdyy g(x_t,y_t) - \gdyy G(x_t,y_t;\zeta_t)]z_t - [\gdy f(x_t,y_t) - \gdy F(x_t,y_t;\xi_t)].
    \end{aligned}
\end{equation*}
\end{lemma}

\begin{proof}[Proof of Lemma~\ref{lm:z-one-step improvement}]
We define the objective function $h(x,z)$ as the following:
\begin{equation*}
    h(x,z) = \frac{1}{2}\langle \gdyy g(x,y^*(x))z, z \rangle - \langle \gdy f(x,y^*(x)), z \rangle.
\end{equation*}
Since $h$ is $l_{g,1}$-smooth in $z$, we have
\begin{small}
\begin{equation} \label{eq:zz-1}
    \begin{aligned}
        h(x_t,z_{t+1})
        &\leq h(x_t,z_t) + \langle \gdyy g(x_t,y_t^*)z_t - \gdy f(x_t,y_t^*), z_{t+1}-z_t \rangle + \frac{l_{g,1}}{2}\|z_{t+1}-z_t\|^2 \\
        &\leq h(x_t,z_t) + \langle \gdyy g(x_t,y_t)z_t - \gdy f(x_t,y_t), z_{t+1}-z_t \rangle + \frac{l_{g,1}}{2}\|z_{t+1}-z_t\|^2 \\
        &\quad+ \langle [\gdyy g(x_t,y_t^*)-\gdyy g(x_t,y_t)]z_t, z_{t+1}-z_t \rangle + \langle [\gdy f(x_t,y_t)-\gdy f(x_t,y_t^*)], z_{t+1}-z_t \rangle \\
        &\leq h(x_t,z_t) + \langle \gdyy G(x_t,y_t;\zeta_t)z_t - \gdy F(x_t,y_t;\xi_t), z_{t+1}-z_t \rangle + \frac{l_{g,1}}{2}\|z_{t+1}-z_t\|^2 \\
        &\quad+ \langle [\gdyy g(x_t,y_t^*)-\gdyy g(x_t,y_t)]z_t, z_{t+1}-z_t \rangle + \langle [\gdy f(x_t,y_t)-\gdy f(x_t,y_t^*)], z_{t+1}-z_t \rangle \\
        &\quad+ \langle [\gdyy g(x_t,y_t)-\gdyy G(x_t,y_t;\zeta_t)]z_t, z_{t+1}-z_t \rangle + \langle [\gdy F(x_t,y_t;\xi_t)-\gdy f(x_t,y_t)], z_{t+1}-z_t \rangle \\
        &= h(x_t,z_t) + \langle \gdyy G(x_t,y_t;\zeta_t)z_t - \gdy F(x_t,y_t;\xi_t), z_{t+1}-z_t \rangle + \frac{l_{g,1}}{2}\|z_{t+1}-z_t\|^2 \\
        &\quad+ \langle [\gdyy g(x_t,y_t^*)-\gdyy g(x_t,y_t)]z_t, z_{t+1}-z_t \rangle + \langle [\gdy f(x_t,y_t)-\gdy f(x_t,y_t^*)], z_{t+1}-z_t \rangle \\
        &\quad+ \langle v_t, z_{t+1}-z_t \rangle, \\
    \end{aligned}
\end{equation}
\end{small}%
where we define $v_t$ as the following:
\begin{small}
\begin{equation*}
    \begin{aligned}
        v_t = [\gdyy g(x_t,y_t) - \gdyy G(x_t,y_t;\zeta_t)]z_t - [\gdy f(x_t,y_t) - \gdy F(x_t,y_t;\xi_t)]. \\
    \end{aligned}
\end{equation*}
\end{small}%
Next, given any $\delta_t>0$, we apply Young's inequality to obtain
\begin{small}
\begin{equation} \label{eq:zz-2}
    \begin{aligned}
        \langle v_t, z_{t+1}-z_t \rangle &\leq \frac{\delta_t}{2}\|v_t\|^2 + \frac{1}{2\delta_t}\|z_{t+1}-z_t\|^2.
    \end{aligned}
\end{equation}
\end{small}%
Also, we estimate the following by Young's inequality (with $\phi_t = 4\gamma$):
\begin{small}
\begin{equation} \label{eq:zz-3}
    \begin{aligned}
        \langle \gdyy g(x_t,y_t^*)z_t-\gdyy g(x_t,y_t)z_t,& z_{t+1}-z_t \rangle
        \leq \frac{\phi_t}{2}\|\gdyy g(x_t,y_t^*)-\gdyy g(x_t,y_t)\|^2\|z_t\|^2 + \frac{1}{2\phi_t}\|z_{t+1}-z_t\|^2 \\
        &\leq \frac{\phi_t}{2}l_{g,2}^2\|y_t-y_t^*\|^2(2\|z_t^*\|^2 + 2\|z_t-z_t^*\|^2) + \frac{1}{2\phi_t}\|z_{t+1}-z_t\|^2 \\
        &\leq \frac{\phi_t}{2}l_{g,2}^2\|y_t-y_t^*\|^2\left(\frac{2l_{f,0}^2}{\mu^2} + 2\|z_t-z_t^*\|^2\right) + \frac{1}{2\phi_t}\|z_{t+1}-z_t\|^2 \\
        &\leq 4\gamma l_{g,2}^2\|y_t-y_t^*\|^2\|z_t-z_t^*\|^2 +  4\gamma\frac{l_{g,2}^2l_{f,0}^2}{\mu^2}\|y_t-y_t^*\|^2 + \frac{(4\gamma)^{-1}}{2}\|z_{t+1}-z_t\|^2, \\
    \end{aligned}
\end{equation}
\end{small}%
where we use $\|z_t^*\| = \|[\gdyy g(x,y^*(x))]^{-1}\gdy f(x, y^*(x))\| \leq l_{f,0}/\mu$ for the second inequality.

Again we apply Young's inequality to estimate (with $\varphi_t = 4\gamma$):
\begin{small}
\begin{equation} \label{eq:zz-4}
    \begin{aligned}
        \langle \gdy f(x_t,y_t)-\gdy f(x_t,y_t^*), z_{t+1}-z_t \rangle 
        &\leq \frac{\varphi_t}{2}\|\gdy f(x_t,y_t)-\gdy f(x_t,y_t^*)\|^2 + \frac{1}{2\varphi_t}\|z_{t+1}-z_t\|^2 \\
        &\leq \frac{\varphi_t}{2}(L_{y,0}+L_{y,1}l_{f,0})^2\|y_t-y_t^*\|^2 + \frac{1}{2\varphi_t}\|z_{t+1}-z_t\|^2 \\
        &\leq 2\gamma(L_{y,0}+L_{y,1}l_{f,0})^2\|y_t-y_t^*\|^2 + \frac{(4\gamma)^{-1}}{2}\|z_{t+1}-z_t\|^2. \\
    \end{aligned}
\end{equation}
\end{small}%

Therefore, given any $z$, combining \eqref{eq:zz-1}, \eqref{eq:zz-2}, \eqref{eq:zz-3} and \eqref{eq:zz-4} we have
\begin{small}
\begin{equation*}
    \begin{aligned}
        h(x_t,z_{t+1})
        &\leq h(x_t,z_t) + \langle \gdyy G(x_t,y_t;\zeta_t)z_t - \gdy F(x_t,y_t;\xi_t), z_{t+1}-z_t \rangle + \frac{l_{g,1}}{2}\|z_{t+1}-z_t\|^2 \\
        &\quad+ \langle [\gdyy g(x_t,y_t^*)-\gdyy g(x_t,y_t)]z_t, z_{t+1}-z_t \rangle + \langle [\gdy f(x_t,y_t)-\gdy f(x_t,y_t^*)], z_{t+1}-z_t \rangle \\
        &\quad+ \langle v_t, z_{t+1}-z_t \rangle \\
        &\leq h(x_t,z_t) + \langle \gdyy G(x_t,y_t;\zeta_t)z_t - \gdy F(x_t,y_t;\xi_t), z_{t+1}-z_t \rangle + \frac{\delta_t^{-1}+l_{g,1}+2(4\gamma)^{-1}}{2}\|z_{t+1}-z_t\|^2 + \frac{\delta_t}{2}\|v_t\|^2 \\
        &\quad+ 4\gamma l_{g,2}^2\|y_t-y_t^*\|^2\|z_t-z_t^*\|^2 +  4\gamma\frac{l_{g,2}^2l_{f,0}^2}{\mu^2}\|y_t-y_t^*\|^2 + 2\gamma(L_{y,0}+L_{y,1}l_{f,0})^2\|y_t-y_t^*\|^2 \\
        &\leq h(x_t,z_t) + \langle \gdyy G(x_t,y_t;\zeta_t)z_t - \gdy F(x_t,y_t;\xi_t), z_{t+1}-z_t \rangle + \frac{1}{2\gamma}\|z_{t+1}-z_t\|^2 \\
        &\quad+ \frac{\delta_t^{-1}+l_{g,1}+(2\gamma)^{-1}-\gamma^{-1}}{2}\|z_{t+1}-z_t\|^2 + \frac{\delta_t}{2}\|v_t\|^2 \\
        &\quad+ 4\gamma l_{g,2}^2\|y_t-y_t^*\|^2\|z_t-z_t^*\|^2 +  4\gamma\frac{l_{g,2}^2l_{f,0}^2}{\mu^2}\|y_t-y_t^*\|^2 + 2\gamma(L_{y,0}+L_{y,1}l_{f,0})^2\|y_t-y_t^*\|^2 \\
        &\leq h(x_t,z_t) + \langle \gdyy G(x_t,y_t;\zeta_t)z_t - \gdy F(x_t,y_t;\xi_t), z-z_t \rangle + \frac{1}{2\gamma}\|z-z_t\|^2 -\frac{1}{2\gamma}\|z-z_{t+1}\|^2 \\
        &\quad+ \frac{\delta_t^{-1}+l_{g,1}-(2\gamma)^{-1}}{2}\|z_{t+1}-z_t\|^2 + \frac{\delta_t}{2}\|v_t\|^2 \\
        &\quad+ 4\gamma l_{g,2}^2\|y_t-y_t^*\|^2\|z_t-z_t^*\|^2 +  4\gamma\frac{l_{g,2}^2l_{f,0}^2}{\mu^2}\|y_t-y_t^*\|^2 + 2\gamma(L_{y,0}+L_{y,1}l_{f,0})^2\|y_t-y_t^*\|^2 \\
    \end{aligned}
\end{equation*}
\end{small}%
where the last inequality holds since $z_{t+1} = z_{t} - \gamma[\gdyy G(x_t,y_{t};\zeta_t)z_{t} - \gdy F(x_t,y_{t};\xi_t)]$ is the unique minimizer of the $\gamma^{-1}$-strongly convex function $l(z) = \langle \gdyy G(x_t,y_{t};\zeta_t)z_{t} - \gdy F(x_t,y_{t};\xi_t),z-z_t \rangle + \frac{1}{2\gamma}\|z-z_t\|^2$, and thus $l(z) - l(z_{t+1}) \geq \frac{1}{2\gamma}\|z-z_{t+1}\|^2$ holds for any $z$. 
Now by $\mu$-strong convexity of $h(x,z)$ in terms of $z$, we estimate
\begin{small}
\begin{equation*}
    \begin{aligned}
        h(x_t,z_t) &+ \langle \gdyy G(x_t,y_t;\zeta_t)z_t - \gdy F(x_t,y_t;\xi_t), z-z_t \rangle \\
        &= h(x_t,z) + \langle \gdyy g(x_t,y_t)z_t -\gdy f(x_t,y_t), z-z_t \rangle + \langle v_t,z_t-z \rangle \\
        &\leq h(x_t,z) - \frac{\mu}{2}\|z-z_t\|^2 + \langle v_t,z_t-z \rangle \\
    \end{aligned}
\end{equation*}
\end{small}%
Thus we have
\begin{small}
\begin{equation*}
    \begin{aligned}
        h(x_t,z_{t+1}) 
        &\leq h(x_t,z) - \frac{\mu}{2}\|z-z_t\|^2 + \langle v_t,z_t-z \rangle + \frac{1}{2\gamma}\|z-z_t\|^2 -\frac{1}{2\gamma}\|z-z_{t+1}\|^2 \\
        &\quad+ \frac{\delta_t^{-1}+l_{g,1}-(2\gamma)^{-1}}{2}\|z_{t+1}-z_t\|^2 + \frac{\delta_t}{2}\|v_t\|^2 \\
        &\quad+ 4\gamma l_{g,2}^2\|y_t-y_t^*\|^2\|z_t-z_t^*\|^2 +  4\gamma\frac{l_{g,2}^2l_{f,0}^2}{\mu^2}\|y_t-y_t^*\|^2 + 2\gamma(L_{y,0}+L_{y,1}l_{f,0})^2\|y_t-y_t^*\|^2 \\
    \end{aligned}
\end{equation*}
\end{small}%
Finally, taking $\delta_t=2\gamma/(1-2l_{g,1}\gamma)$ and rearranging yields
\begin{small}
\begin{equation*}
    \begin{aligned}
        2\gamma(h(x_t,z_{t+1}) - h(x_t,z))
        &\leq (1-\mu\gamma)\|z_t-z\|^2 - \|z_{t+1}-z\|^2 + 2\gamma\langle v_t, z_t-z \rangle + \frac{2\gamma^2}{1-2l_{g,1}\gamma}\|v_t\|^2 \\
        &\quad+ 8\gamma^2 l_{g,2}^2\|y_t-y_t^*\|^2\|z_t-z_t^*\|^2 +  8\gamma^2\frac{l_{g,2}^2l_{f,0}^2}{\mu^2}\|y_t-y_t^*\|^2 + 4\gamma^2(L_{y,0}+L_{y,1}l_{f,0})^2\|y_t-y_t^*\|^2,  \\
    \end{aligned}
\end{equation*}
\end{small}%
which is as claimed.
\end{proof}

\begin{lemma}[Distance recursion] \label{lm:z-distance recursion}
Consider Algorithm~\ref{alg:bilevel} with sequence $\{z_t\}$ and constant learning rate $\gamma \leq 1/(4l_{g,1})$, then under event $\gE_{\init} \cap \gE_y$, for any $t\geq1$, we have the following recursion:
\begin{small}
\begin{equation}
    \begin{aligned}
        \|z_{t+1}-z_{t+1}^*\|^2 
        &\leq \left(1-\frac{\mu\gamma}{2}\right)\|z_t-z_t^*\|^2 + 2\gamma\langle v_t, z_t-z_t^* \rangle + \frac{2\gamma^2}{1-2l_{g,1}\gamma}\|v_t\|^2 + \left(1+\frac{1}{\mu\gamma}\right)\|z_t^*-z_{t+1}^*\|^2 \\ 
        &\quad+ 8\gamma^2\frac{l_{g,2}^2l_{f,0}^2}{\mu^2}\|y_t-y_t^*\|^2 + 4\gamma^2(L_{y,0}+L_{y,1}l_{f,0})^2\|y_t-y_t^*\|^2.  \\
    \end{aligned}
\end{equation}
\end{small}%
\end{lemma}

\begin{proof}[Proof of Lemma~\ref{lm:z-distance recursion}]
Note that the $\mu$-strong convexity implies 
\begin{small}
\begin{equation*}
    \frac{\mu}{2}\|z_{t+1}-z_t^*\|^2 \leq h(x_t,z_{t+1}) - h(x_t,z_t^*).
\end{equation*}
\end{small}%
Combining this estimate with Lemma~\ref{lm:z-one-step improvement} under the identification $z=z_t^*$ yields
\begin{small}
\begin{equation*}
    \begin{aligned}
        (1+\mu\gamma)\|z_{t+1}-z_t^*\|^2 
        &\leq (1-\mu\gamma)\|z_t-z_t^*\|^2 + 2\gamma\langle v_t, z_t-z_t^* \rangle + \frac{2\gamma^2}{1-2l_{g,1}\gamma}\|v_t\|^2 \\ 
        &\quad+ 8\gamma^2 l_{g,2}^2\|y_t-y_t^*\|^2\|z_t-z_t^*\|^2 +  8\gamma^2\frac{l_{g,2}^2l_{f,0}^2}{\mu^2}\|y_t-y_t^*\|^2 + 4\gamma^2(L_{y,0}+L_{y,1}l_{f,0})^2\|y_t-y_t^*\|^2 \\
        &\leq (1-\mu\gamma+8\gamma^2 l_{g,2}^2\|y_t-y_t^*\|^2)\|z_t-z_t^*\|^2 + 2\gamma\langle v_t, z_t-z_t^* \rangle + \frac{2\gamma^2}{1-2l_{g,1}\gamma}\|v_t\|^2 \\ 
        &\quad+ 8\gamma^2\frac{l_{g,2}^2l_{f,0}^2}{\mu^2}\|y_t-y_t^*\|^2 + 4\gamma^2(L_{y,0}+L_{y,1}l_{f,0})^2\|y_t-y_t^*\|^2,  \\
    \end{aligned}
\end{equation*}
\end{small}%
under event $\gE_{\init} \cap \gE_y$, for any $t\in[T]$ we have 
\begin{small}
\begin{equation*}
    8\gamma^2l_{g,2}^2\|y_t-y_t^*\|^2 \leq \frac{\gamma^2l_{g,2}^2}{8L_1^2} \leq  \frac{\mu\gamma}{2},
\end{equation*}
\end{small}%
where for the first inequality we use Lemma~\ref{lm:high-prob 1/8L_1 bound for y}, and for the second inequality we use \eqref{eq:gamma+alpha-high-prob} and \eqref{eq:beta-high-prob} that 
\begin{small}
\begin{equation*}
    \gamma = \frac{16}{\mu}(1-\beta) \quad\quad \text{and} \quad\quad 1-\beta \leq \frac{\mu^2L_1^2}{4l_{g,2}^2}.
\end{equation*}
\end{small}%
Then we have
\begin{small}
\begin{equation} \label{eq:zz-5}
    \begin{aligned}
        (1+\mu\gamma)\|z_{t+1}-z_t^*\|^2 
        &\leq \left(1-\frac{\mu\gamma}{2}\right)\|z_t-z_t^*\|^2 + 2\gamma\langle v_t, z_t-z_t^* \rangle + \frac{2\gamma^2}{1-2l_{g,1}\gamma}\|v_t\|^2 \\ 
        &\quad+ 8\gamma^2\frac{l_{g,2}^2l_{f,0}^2}{\mu^2}\|y_t-y_t^*\|^2 + 4\gamma^2(L_{y,0}+L_{y,1}l_{f,0})^2\|y_t-y_t^*\|^2,  \\
    \end{aligned}
\end{equation}
\end{small}%
Next, under event $\gE_{\init} \cap \gE_y$ and an application of Young's inequality combining with \eqref{eq:zz-5} reveals
\begin{small}
\begin{equation*}
    \begin{aligned}
        \|z_{t+1}-z_{t+1}^*\|^2 
        &\leq (1+\mu\gamma)\|z_{t+1}-z_t^*\|^2 + (1+(\mu\gamma)^{-1})\|z_t^*-z_{t+1}^*\|^2 \\
        &\leq \left(1-\frac{\mu\gamma}{2}\right)\|z_t-z_t^*\|^2 + 2\gamma\langle v_t, z_t-z_t^* \rangle + \frac{2\gamma^2}{1-2l_{g,1}\gamma}\|v_t\|^2 + \left(1+\frac{1}{\mu\gamma}\right)\|z_t^*-z_{t+1}^*\|^2 \\ 
        &\quad+ 8\gamma^2\frac{l_{g,2}^2l_{f,0}^2}{\mu^2}\|y_t-y_t^*\|^2 + 4\gamma^2(L_{y,0}+L_{y,1}l_{f,0})^2\|y_t-y_t^*\|^2.  \\
    \end{aligned}
\end{equation*}
\end{small}%
\end{proof}

In the following lemmas and theorems, we will use the following parameter settings. In particular, we choose
\begin{small}
\begin{equation} \label{eq:eps-high-prob}
    \begin{aligned}
        \eps = \eps'G,\quad
        \eps' \leq \min&\left\{\frac{L_0}{L_1}, \Delta_{y,0}L_0, \frac{64l_{g,1}\Delta_{z,0}L_0}{\sqrt{4El_{g,1}^2+l_{z^*}^2}}, \frac{8l_{g,1}L_0}{\mu\sqrt{2(1+l_{g,1}^2/\mu^2)(L_{x,1}^2+L_{y,1}^2)}}, \sqrt{\frac{16el_{g,1}\Delta_0L_0}{\mu\delta}}, 4\left(\frac{el_{g,1}\Delta_0L_0^3\sigma_{g,1}^2}{\mu^3\delta}\right)^{1/4}; \right. \\
        &\left. \frac{\sigma_{g,1}L_0L_1}{l_{g,2}},  \sqrt{\frac{32el_{g,1}\Delta_0L_0}{\mu\delta\exp(\mu/(2l_{g,1}))}}, \sqrt{\frac{32el_{g,1}\Delta_0L_0}{\mu\delta\exp(\mu l_{g,1}\Delta_{z,0}/(2\Delta_0L_0))}}, \sqrt{\frac{32el_{g,1}\Delta_0L_0}{\mu\delta\exp(\mu\Delta_{y,0}^2L_0/(2l_{g,1}\Delta_0))}}, \right. \\
        &\left. \frac{\Delta_0}{\Delta_{z,0}}, \frac{\Delta_0L_0}{\|\gdphi(x_0)\|}, \frac{L_0\sigma_{g,1}}{\sigma_{g,2}}, \frac{L_0\sigma_{g,1}}{\sqrt{\mu l_{g,1}}}, \left(\frac{2^{21}el_{g,1}\Delta_0L_0^3\sigma_{g,1}^2}{\mu^3\delta}\right)^{1/4}\exp\left(\frac{-l_{g,1}\sqrt{\sigma_{f,1}^2 + 2l_{f,0}^2\sigma_{g,2}^2/\mu^2}}{512L_0\sigma_{g,1}}\right)
        \right\},
    \end{aligned}
\end{equation}
\end{small}%
\begin{small}
\begin{equation} \label{eq:beta-high-prob}
    \begin{aligned}
        1-\beta = \min&\left\{1, \frac{\mu}{16l_{g,1}}, \frac{16el_{g,1}\Delta_0L_0}{\mu\delta\eps'^2}, \frac{\mu^2\eps'^2}{64\cdot1024L_0^2\sigma_{g,1}^2\log^2(B)}; \frac{\min\{1, \mu^2/(32l_{g,1}^2)\}}{\sigma_{f,1}^2 + 2l_{f,0}^2\sigma_{g,2}^2/\mu^2}\eps'^2, \frac{l_{g,1}^2}{8\sigma_{g,2}^2}, \frac{\mu^2}{16\sigma_{g,2}^2}; \right. \\
        &\left. \frac{\mu^2L_1^2}{4l_{g,2}^2}, \frac{32el_{g,1}\Delta_0L_0}{\mu\delta\eps'^2\exp(\mu/(2l_{g,1}))}, \frac{32el_{g,1}\Delta_0L_0}{\mu\delta\eps'^2\exp(\mu l_{g,1}\Delta_{z,0}^2/(2\Delta_0L_0))}, \frac{32el_{g,1}\Delta_0L_0}{\mu\delta\eps'^2\exp(\mu\Delta_{y,0}^2L_0/(2l_{g,1}\Delta_0))}
        \right\},
    \end{aligned}
\end{equation}
\end{small}%
\begin{small}
\begin{equation} \label{eq:eta-high-prob}
    \begin{aligned}
        \eta = \min\left\{\frac{1}{8}\min\left(\frac{1}{L_1}, \frac{\eps'}{L_0}, \frac{\Delta_0}{\Delta_{z,0}L_0}, \frac{\eps'\Delta_0}{\Delta_{y,0}^2L_0^2}, \frac{\Delta_0}{\|\gdphi(x_0)\|}, \frac{\eps'\Delta_0}{l_{g,1}^2\Delta_{z,0}^2}, \frac{\mu\eps'}{l_{g,1}L_0\log(A)}\right)(1-\beta), \frac{1}{\sqrt{2(1+l_{g,1}^2/\mu^2)(L_{x,1}^2+L_{y,1}^2)}}\right\},
    \end{aligned}
\end{equation}
\end{small}%
\begin{small}
\begin{equation} \label{eq:gamma+alpha-high-prob}
    \begin{aligned}
        \alpha^{\init} = \min\left\{\frac{1}{2l_{g,1}}, \frac{\mu}{2048L_1^2\sigma_{g,1}^2\log(e/\delta)}\right\}, \quad T_0 = \frac{\log\left(256L_1^2\|y_0^{\init}-y_0^*\|^2\right)}{\log\left(2/(2-\mu\alpha^{\init})\right)}, \quad
        \gamma = \frac{16}{\mu}(1-\beta), \quad \alpha = \frac{8}{\mu}(1-\beta), 
        \quad T = \frac{4\Delta_0}{\eta\eps'},
    \end{aligned}
\end{equation}
\end{small}%
where $\Delta_0, \Delta_{y,0}, \Delta_{z,0}, A, B$ are defined in \eqref{eq:ABdelta}, and $E$ and $G$ are defined as
\begin{small}
\begin{equation} \label{eq:E-high-prob}
    \begin{aligned}
        E \coloneqq \max\left\{\frac{4\Bar{\sigma}^2}{\sigma_{g,1}^2}, \frac{l_{g,2}^2l_{f,0}^2}{8\mu^2\sigma_{g,1}^2L_1^2}, \frac{L_0^2}{16\sigma_{g,1}^2L_1^2}\right\},
    \end{aligned}
\end{equation}
\end{small}%
\begin{small}
\begin{equation} \label{eq:G-high-prob}
    \begin{aligned}
        G \coloneqq \frac{\mu}{16\sigma_{g,1}L_0}\left(\sigma_{f,1} + \frac{l_{f,0}\sigma_{g,2}}{\mu} + 2\sigma_{g,2}\Delta_{z,0}\right) + \left(1 + \sqrt{E+\frac{l_{z^*}^2}{4l_{g,1}^2}}\right)\frac{l_{g,1}}{8L_0} + \frac{13}{8}.
    \end{aligned}
\end{equation}
\end{small}%

With Lemma~\ref{lm:z-distance recursion} and Proposition~\ref{prop:MGF-recursive control}, as well as the fact that event $\gE_{\init}\cap\gE_y\in\sigma(\widetilde{\gF}_{T_0}^1\cup\gF_{T}^1)$ are independent of event in $\gF_t^2$ for any $t\in[T]$, we are able to leverage the following distance tracking result with high probability.

\begin{lemma}[High-probability distance tracking] \label{lm:z-high-prob distance tracking}
Suppose that Assumption~\ref{ass:f-and-g} holds, let $\{z_t\}$ be the iterates produced by Algorithm~\ref{alg:bilevel} with constant learning rate $\gamma\leq 1/(4l_{g,1})$.
Then under event $\gE_{\init} \cap \gE_y$, for any fixed $t\in[T]$ and $\delta\in(0,1)$, the following estimate holds with probability at least $1-\delta$ over the randomness in $\gF_t^2$:
\begin{small}
\begin{equation} \label{eq:z-one-distance}
    \|z_t-z_t^*\|^2 \leq \left(1-\frac{\mu\gamma}{4}\right)^t\|z_0-z_0^*\|^2 + \left[\frac{16\gamma\Bar{\sigma}^2}{\mu} + \frac{8l_{z^*}^2\eta^2}{\mu^2\gamma^2} + \frac{4\gamma}{\mu}\frac{l_{g,2}^2l_{f,0}^2}{8\mu^2L_1^2} + \frac{4\gamma}{16\mu L_1^2}(L_{y,0}+L_{y,1}l_{f,0})^2\right]\log\left(\frac{e}{\delta}\right),
\end{equation}
\end{small}%
where we denote $\Bar{\sigma} = \sigma_z + \sigma_{f,1}$. As a consequence, under event $\gE_{\init} \cap \gE_y$, for any given $\delta\in(0,1)$ and all $t\in[T]$, the following estimate holds with probability at least $1-\delta$ over the randomness in $\gF_{T}^2$:
\begin{small}
\begin{equation} \label{eq:z-all-distance}
    \|z_t-z_t^*\|^2 \leq \left(1-\frac{\mu\gamma}{4}\right)^t\|z_0-z_0^*\|^2 + \left[\frac{16\gamma\Bar{\sigma}^2}{\mu} + \frac{8l_{z^*}^2\eta^2}{\mu^2\gamma^2} + \frac{4\gamma}{\mu}\frac{l_{g,2}^2l_{f,0}^2}{8\mu^2L_1^2} + \frac{4\gamma}{16\mu L_1^2}(L_{y,0}+L_{y,1}l_{f,0})^2\right]\log\left(\frac{eT}{\delta}\right).
\end{equation}
\end{small}%
\end{lemma}

\begin{proof}[Proof of Lemma~\ref{lm:z-high-prob distance tracking}]
For simplicity, we define $\E_{\mid\gE_{\init} \cap \gE_y} \coloneqq \E[\cdot \mid \gE_{\init} \cap \gE_y]$, where events $\gE_{\init}$ and $\gE_y$ are defined in Lemma~\ref{lm:warm-up} and Lemma~\ref{lm:high-prob 1/8L_1 bound for y}. 
Under event $\gE_{\init} \cap \gE_y$, 
Lemma~\ref{lm:z-distance recursion} in the regime $\gamma \leq 1/(4l_{g,1})$ directly yields
\begin{small}
\begin{equation*}
    \begin{aligned}
        \|z_{t+1}-z_{t+1}^*\|^2 
        &\leq \left(1-\frac{\mu\gamma}{2}\right)\|z_t-z_t^*\|^2 + 2\gamma\langle v_t, r_t \rangle \|z_t-z_t^*\| + 4\gamma^2\|v_t\|^2 + \frac{2}{\mu\gamma}\|z_t^*-z_{t+1}^*\|^2 \\ 
        &\quad+ 8\gamma^2\frac{l_{g,2}^2l_{f,0}^2}{\mu^2}\|y_t-y_t^*\|^2 + 4\gamma^2(L_{y,0}+L_{y,1}l_{f,0})^2\|y_t-y_t^*\|^2,  \\
        &\leq \left(1-\frac{\mu\gamma}{2}\right)\|z_t-z_t^*\|^2 + 2\gamma\langle v_t, z_t-z_t^* \rangle + \frac{2\gamma^2}{1-2l_{g,1}\gamma}\|v_t\|^2 + \left(1+\frac{1}{\mu\gamma}\right)\|z_t^*-z_{t+1}^*\|^2 \\ 
        &\quad+ \gamma^2\frac{l_{g,2}^2l_{f,0}^2}{8\mu^2L_1^2} + \frac{\gamma^2}{16L_1^2}(L_{y,0}+L_{y,1}l_{f,0})^2.  \\
    \end{aligned}
\end{equation*}
\end{small}%
where for the first inequality we set $r_t \coloneqq \frac{z_t-z_t^*}{\|z_t-z_t^*\|}$ if $z_t$ is distinct from $z_t^*$ and set it to zero otherwise, and for the second inequality we use Lemma~\ref{lm:high-prob 1/8L_1 bound for y}, namely $\|y_t-y_t^*\|\leq 1/(8L_1)$. The right-hand side
has the form of a contraction factor, gradient noise, estimation error from $y_t$ and drift. Now the goal is to control the moment generating function $\E_{\mid\gE_{\init} \cap \gE_y}[\exp(\lambda\|z_t-z_t^*\|)]$ through this recursion. We apply  Proposition~\ref{prop:MGF-recursive control}, with $\gH_t = \gF_t^2$, $V_t = \|z_t-z_t^*\|^2$, $U_t = 2\gamma\langle v_t,r_t \rangle$, $X_t = 4\gamma^2\|v_t\|^2 + 2\|z_t^*-z_{t+1}^*\|^2/(\mu\gamma)$, $\alpha_t = 1-\mu\gamma/2$, $\kappa_t = C$, $\sigma_t = 2\gamma(\sigma_z+\sigma_{f,1})$ and $\nu_t = 4\gamma^2(\sigma_z+\sigma_{f,1})^2 + 2l_{z^*}^2\eta^2/(\mu\gamma)$, 
where we define constant $C$ as
\begin{small}
\begin{equation*}
    C \coloneqq \gamma^2\frac{l_{g,2}^2l_{f,0}^2}{8\mu^2L_1^2} + \frac{\gamma^2}{16L_1^2}(L_{y,0}+L_{y,1}l_{f,0})^2,
\end{equation*}
\end{small}%
then yielding the estimate
\begin{small}
\begin{equation} \label{eq:mgf-recur}
    \begin{aligned}
        \E_{\mid\gE_{\init} \cap \gE_y}\left[\exp\left(\lambda\|z_{t+1}-z_{t+1}^*\|^2\right)\right] \leq \exp\left[\lambda\left(4\gamma^2\Bar{\sigma}^2 + \frac{2l_{z^*}^2\eta^2}{\mu\gamma} + R\right)\right]\E_{\mid\gE_{\init} \cap \gE_y}\left[\exp\left(\lambda\left(1-\frac{\mu\gamma}{4}\right)\|z_t-z_t^*\|^2\right)\right] \\
    \end{aligned}
\end{equation}
\end{small}%
for all
\begin{small}
\begin{equation*}
    0 \leq \lambda \leq \min\left\{\frac{\mu}{16\gamma\Bar{\sigma}^2}, \frac{1}{8\gamma^2\Bar{\sigma}^2 + 4l_{z^*}^2\eta^2/(\mu\gamma)}\right\}.
\end{equation*}
\end{small}%
where we denote $\Bar{\sigma} = \sigma_z+\sigma_{f,1}$ for simplicity. We deduce the following by iterating the recursion \eqref{eq:mgf-recur}:
\begin{small}
\begin{equation*}
    \begin{aligned}
        \E_{\mid\gE_{\init} \cap \gE_y}&\left[\exp\left(\lambda\|z_t-z_t^*\|^2\right)\right] 
        \leq
        \exp\left[\lambda\left(1-\frac{\mu\gamma}{4}\right)^t\|z_0-z_0^*\|^2 + \lambda\left(4\gamma^2\Bar{\sigma}^2 + \frac{2l_{z^*}^2\eta^2}{\mu\gamma} + C\right)\sum_{i=0}^{t-1}\left(1-\frac{\mu\gamma}{4}\right)^i\right] \\
        &\leq \exp\left\{\lambda\left[\left(1-\frac{\mu\gamma}{4}\right)^t\|z_0-z_0^*\|^2 + \frac{4}{\mu\gamma}\left(4\gamma^2\Bar{\sigma}^2 + \frac{2l_{z^*}^2\eta^2}{\mu\gamma} + C\right)\right]\right\} \\
        &\leq \exp\left\{\lambda\left[\left(1-\frac{\mu\gamma}{4}\right)^t\|z_0-z_0^*\|^2 + \frac{4}{\mu\gamma}\left(4\gamma^2\Bar{\sigma}^2 + \frac{2l_{z^*}^2\eta^2}{\mu\gamma} + \gamma^2\frac{l_{g,2}^2l_{f,0}^2}{8\mu^2L_1^2} + \frac{\gamma^2}{16L_1^2}(L_{y,0}+L_{y,1}l_{f,0})^2\right)\right]\right\} \\
        &\leq \exp\left\{\lambda\left[\left(1-\frac{\mu\gamma}{4}\right)^t\|z_0-z_0^*\|^2 + \frac{16\gamma\Bar{\sigma}^2}{\mu} + \frac{8l_{z^*}^2\eta^2}{\mu^2\gamma^2} + \frac{4\gamma}{\mu}\frac{l_{g,2}^2l_{f,0}^2}{8\mu^2L_1^2} + \frac{4\gamma}{16\mu L_1^2}(L_{y,0}+L_{y,1}l_{f,0})^2\right]\right\} \\
    \end{aligned}
\end{equation*}
\end{small}%
for all
\begin{small}
\begin{equation*}
    0 \leq \lambda \leq \min\left\{\frac{\mu}{16\gamma\Bar{\sigma}^2}, \frac{1}{8\gamma^2\Bar{\sigma}^2 + 4l_{z^*}^2\eta^2/(\mu\gamma)}\right\}.
\end{equation*}
\end{small}%
Moreover, setting
\begin{small}
\begin{equation*}
    \nu \coloneqq \frac{16\gamma\Bar{\sigma}^2}{\mu} + \frac{8l_{z^*}^2\eta^2}{\mu^2\gamma^2} + \frac{4\gamma}{\mu}\frac{l_{g,2}^2l_{f,0}^2}{8\mu^2L_1^2} + \frac{4\gamma}{16\mu L_1^2}(L_{y,0}+L_{y,1}l_{f,0})^2
\end{equation*}
\end{small}%
and taking into account $\mu\gamma \leq 1$, we have
\begin{small}
\begin{equation*}
    0 \leq \lambda \leq \frac{1}{\nu} \leq \min\left\{\frac{\mu}{16\gamma\Bar{\sigma}^2}, \frac{1}{8\gamma^2\Bar{\sigma}^2 + 4l_{z^*}^2\eta^2/(\mu\gamma)}\right\}.
\end{equation*}
\end{small}%
Hence we obtain
\begin{small}
\begin{equation*}
    \E_{\mid\gE_{\init} \cap \gE_y}\left[\exp\left[\lambda\left(\|z_t-z_t^*\|^2 - \left(1-\frac{\mu\gamma}{4}\right)^t\|z_0-z_0^*\|^2\right)\right]\right] \leq \exp(\lambda\nu), \quad \forall 0\leq \lambda \leq 1/\nu.
\end{equation*}
\end{small}%
Taking $\lambda = 1/\nu$ and applying Markov's inequality yields that, for any given $\delta\in(0,1)$, under event $\gE_{\init} \cap \gE_y$, with probability at least $1-\delta$ over the randomness in $\gF_t^2$:
\begin{small}
\begin{equation*}
    \|z_t-z_t^*\|^2 \leq \left(1-\frac{\mu\gamma}{4}\right)^t\|z_0-z_0^*\|^2 + \left[\frac{16\gamma\Bar{\sigma}^2}{\mu} + \frac{8l_{z^*}^2\eta^2}{\mu^2\gamma^2} + \frac{4\gamma}{\mu}\frac{l_{g,2}^2l_{f,0}^2}{8\mu^2L_1^2} + \frac{4\gamma}{16\mu L_1^2}(L_{y,0}+L_{y,1}l_{f,0})^2\right]\log\left(\frac{e}{\delta}\right),
\end{equation*}
\end{small}%
as claimed in \eqref{eq:z-one-distance}. We obtain \eqref{eq:z-all-distance} by applying union bound.
\end{proof}

With Lemma~\ref{lm:z-one-step improvement},~\ref{lm:z-distance recursion} and~\ref{lm:z-high-prob distance tracking}, under suitable parameter choice as in \eqref{eq:eps-high-prob}, \eqref{eq:beta-high-prob}, \eqref{eq:eta-high-prob} and \eqref{eq:gamma+alpha-high-prob}, we are now able to leverage Lemma~\ref{lm:z-two-bound} which provides refined control for $z_t$, and is also similar to what we did in Lemma~\ref{lm:high-prob 1/8L_1 bound for y}.

\begin{lemma} \label{lm:z-two-bound}
Under Assumptions~\ref{ass:relax-smooth},~\ref{ass:f-and-g} and the parameter setting \eqref{eq:eps-high-prob}, \eqref{eq:beta-high-prob}, \eqref{eq:eta-high-prob} and \eqref{eq:gamma+alpha-high-prob}, 
run Algorithm~\ref{alg:bilevel} for $T=\frac{4\Delta_0}{\eta\eps'}$ iterations. Then under event $\gE_{\init} \cap \gE_y$, for all $t\in[T]$ and any given $\delta\in(0,1)$, Algorithm~\ref{alg:bilevel} guarantees with probability at least $1-\delta$  over the randomness in $\gF_T^2$ (we denote this event as $\gE_z$) that:
\begin{enumerate}
    \item $\|z_t-z_t^*\|\leq 2\Delta_{z,0}$,
    \item $\frac{1}{T}(1-\beta)\sum_{t=0}^{T-1}\sum_{i=0}^{t}\beta^{t-i}\|z_i-z_i^*\| \leq \left(1 + \sqrt{E+l_{z^*}^2/(4l_{g,1}^2)}\right)\frac{\eps'}{32L_0}$,
\end{enumerate}
where constant $E$ is defined in \eqref{eq:E-high-prob}.
\end{lemma}

\begin{proof}[Proof of Lemma~\ref{lm:z-two-bound}]
We apply \eqref{eq:z-all-distance} and \eqref{eq:gamma+alpha-high-prob} to obtain
\begin{small}
\begin{equation*}
    \begin{aligned}
        &\|z_t-z_t^*\|^2 
        \leq \left(1-\frac{\mu\gamma}{4}\right)^t\|z_0-z_0^*\|^2 + \left[\frac{16\gamma\Bar{\sigma}^2}{\mu} + \frac{8l_{z^*}^2\eta^2}{\mu^2\gamma^2} + \frac{4\gamma}{\mu}\frac{l_{g,2}^2l_{f,0}^2}{8\mu^2L_1^2} + \frac{4\gamma}{16\mu L_1^2}(L_{y,0}+L_{y,1}l_{f,0})^2\right]\log\left(\frac{eT}{\delta}\right) \\
        &\quad= \left(1-\frac{\mu\gamma}{4}\right)^t\|z_0-z_0^*\|^2 + \left[\frac{256(1-\beta)\Bar{\sigma}^2}{\mu^2} + \frac{l_{z^*}^2\eta^2}{32\mu^2(1-\beta)^2} + \frac{8(1-\beta)}{\mu^2}\frac{l_{g,2}^2l_{f,0}^2}{\mu^2L_1^2} + \frac{4(1-\beta)}{\mu^2L_1^2}(L_{y,0}+L_{y,1}l_{f,0})^2\right]\log\left(\frac{eT}{\delta}\right) \\
        &\quad\leq \left(1-\frac{\mu\gamma}{4}\right)^t\|z_0-z_0^*\|^2 + \left(\frac{\eta^2l_{g,1}^2}{16\mu^2(1-\beta)^2}E + \frac{\eta^2l_{g,1}^2}{8\mu^2(1-\beta)^2}\frac{l_{z^*}^2}{4l_{g,1}^2}\right)\log\left(\frac{eT}{\delta}\right) \\
        &\quad\leq \left(1-\frac{\mu\gamma}{4}\right)^t\|z_0-z_0^*\|^2 + \left(E+\frac{l_{z^*}^2}{4l_{g,1}^2}\right)\frac{\eps'^2}{1024L_0^2}, \\
    \end{aligned}
\end{equation*}
\end{small}%
where for the first equality we use \eqref{eq:gamma+alpha-high-prob}, for the second inequality we use conclusion of step 3 in Lemma~\ref{lm:high-prob 1/8L_1 bound for y} to deduce
\begin{small}
\begin{equation*}
    \begin{aligned}
        \max\left\{\frac{256(1-\beta)\Bar{\sigma}^2}{\mu^2}, \frac{8(1-\beta)}{\mu^2}\frac{l_{g,2}^2l_{f,0}^2}{\mu^2L_1^2}, \frac{4(1-\beta)}{\mu^2L_1^2}(L_{y,0}+L_{y,1}l_{f,0})^2\right\} \leq \frac{64(1-\beta) \sigma_{g,1}^2}{\mu^2}E \leq \frac{\eta^2l_{g,1}^2}{16\mu^2(1-\beta)^2}E,
    \end{aligned}
\end{equation*}
\end{small}%
with constant $E$ defined as
\begin{small}
\begin{equation*}
    E \coloneqq \max\left\{\frac{4\Bar{\sigma}^2}{\sigma_{g,1}^2}, \frac{l_{g,2}^2l_{f,0}^2}{8\mu^2\sigma_{g,1}^2L_1^2}, \frac{L_0^2}{16\sigma_{g,1}^2L_1^2}\right\},
\end{equation*}
\end{small}%
and for the last inequality we apply \eqref{eq:merge-step} in Lemma~\ref{lm:high-prob 1/8L_1 bound for y}.

By \eqref{eq:eps-high-prob}, we have 
\begin{small}
\begin{equation*}
    \begin{aligned}
        \eps' \leq \frac{64l_{g,1}\Delta_{z,0}L_0}{\sqrt{4El_{g,1}^2+l_{z^*}^2}}
    \end{aligned}
\end{equation*}
\end{small}%
which implies for all $t\in[T]$ that 
\begin{small}
\begin{equation*}
    \|z_t-z_t^*\|^2 \leq \left(1-\frac{\mu\gamma}{4}\right)^t\|z_0-z_0^*\|^2 + \left(E+\frac{l_{z^*}^2}{4l_{g,1}^2}\right)\frac{\eps'^2}{1024L_0^2} \leq 2\Delta_{z,0}^2 
    \quad \Longrightarrow \quad
    \|z_t-z_t^*\| \leq 2\Delta_{z,0},
\end{equation*}
\end{small}%
thus the first part of the result is as claimed.

As for the second part, we have
\begin{small}
\begin{equation} \label{eq:z-second-part1}
    \begin{aligned}
        \frac{1}{T}&(1-\beta)\sum_{t=0}^{T-1}\sum_{i=0}^{t}\beta^{t-i}\|z_i-z_i^*\|
        \leq \frac{(1-\beta)}{T}\sum_{t=0}^{T-1}\sum_{i=0}^{t}\beta^{t-i}\sqrt{\left(1-\frac{\mu\gamma}{4}\right)^i\|z_0-z_0^*\|^2 + \left(E+\frac{l_{z^*}^2}{4l_{g,1}^2}\right)\frac{\eps'^2}{1024L_0^2}} \\
        &\leq \frac{1}{T}(1-\beta)\sum_{t=0}^{T-1}\sum_{i=0}^{t}\beta^{t-i}\left[\left(1-\frac{\mu\gamma}{4}\right)^{i/2}\|z_0-z_0^*\| + \sqrt{E+\frac{l_{z^*}^2}{4l_{g,1}^2}}\frac{\eps'}{32L_0}\right] \\
        &\leq \frac{1}{T}(1-\beta)\sum_{t=0}^{T-1}\left[\beta^t\sum_{i=0}^{t}\left(\frac{\sqrt{1-\mu\gamma/4}}{\beta}\right)^i\|z_0-z_0^*\| + \frac{1}{1-\beta}\sqrt{E+\frac{l_{z^*}^2}{4l_{g,1}^2}}\frac{\eps'}{32L_0}\right] \\
        &\leq \frac{8\Delta_{z,0}}{T(\mu\gamma-8(1-\beta))} + \sqrt{E+\frac{l_{z^*}^2}{4l_{g,1}^2}}\frac{\eps'}{32L_0}, \\
    \end{aligned}
\end{equation}
\end{small}%
where in the last inequality we use
\begin{small}
\begin{equation*}
    \begin{aligned}
        (1-\beta)\sum_{t=0}^{T-1}\beta^t\sum_{i=0}^{t}\left(\frac{\sqrt{1-\mu\gamma/4}}{\beta}\right)^i
        &\leq (1-\beta)\sum_{t=0}^{T-1}\beta^t\frac{\beta}{\beta-\sqrt{1-\mu\gamma/4}}
        \leq \frac{\beta}{\beta-\sqrt{1-\mu\gamma/4}}
        \leq \frac{\beta\left(\beta+\sqrt{1-\mu\gamma/4}\right)}{\mu\gamma/4-(1-\beta^2)} \\
        &\leq \frac{2}{\mu\gamma/4-(1-\beta)(1+\beta)}
        \leq \frac{2}{\mu\gamma/4-2(1-\beta)}
        = \frac{8}{\mu\gamma-8(1-\beta)}.
    \end{aligned}
\end{equation*}
\end{small}%
Moreover, we have
\begin{sequation} \label{eq:z-second-part2}
    \begin{aligned}
        \frac{8\Delta_{z,0}}{T(\mu\alpha-8(1-\beta))}
        \leq \frac{\Delta_{z,0}}{T(1-\beta)}
        = \frac{\eta\eps'\Delta_{z,0}}{4\Delta_0(1-\beta)}
        \leq \frac{\eps'}{32L_0},
    \end{aligned}
\end{sequation}%
where in the last inequality we use \eqref{eq:eta-high-prob} that
\begin{small}
\begin{equation*}
    \eta \leq \frac{\Delta_0}{8\Delta_{z,0}L_0}(1-\beta).
\end{equation*}
\end{small}%
Combining \eqref{eq:z-second-part1} and \eqref{eq:z-second-part2} yields
\begin{small}
\begin{equation*}
    \frac{1}{T}(1-\beta)\sum_{t=0}^{T-1}\sum_{i=0}^{t}\beta^{t-i}\|z_i-z_i^*\| \leq \frac{8\Delta_{z,0}}{T(\mu\gamma-8(1-\beta))} + \sqrt{E+\frac{l_{z^*}^2}{4l_{g,1}^2}}\frac{\eps'}{32L_0} \leq \left(1 + \sqrt{E+\frac{l_{z^*}^2}{4l_{g,1}^2}}\right)\frac{\eps'}{32L_0}.
\end{equation*}
\end{small}
\end{proof}

Next, we proceed to give high probability bound for hypergradient estimation error, and we will use the following lemma in \cite{liu2023nearlyoptimal} as a technical tool.

\begin{lemma}[Lemma 2.4 in \cite{liu2023nearlyoptimal}] \label{lm:MDS}
Suppose $X_1, \dots , X_T$ is a martingale difference sequence adapted to a filtration $F_1, F_2, \dots$ in a Hilbert space such that $\|X_t\| \leq R_t, \forall t\in[T]$ almost surely for some constant $R_t\geq 0$. Then for any given $\delta\in(0,1)$, with probability at least $1-\delta$, for all $t\in[T]$ we have
\begin{equation*}
    \left\|\sum_{s=1}^{t}X_s\right\| \leq 4\sqrt{\log(2/\delta)\sum_{s=1}^{T}R_s^2}.
\end{equation*}
\end{lemma}

With Lemma~\ref{lm:MDS}, we are ready to give high probability bound for the following martingale difference sequence, which is one part of the hypergradient estimation error.

\begin{lemma} \label{lm:MDS-var}
Under Assumptions~\ref{ass:relax-smooth} and event $\gE_{\init} \cap \gE_y\cap\gE_z$, for any given $\delta\in(0,1)$ and fixed $t\in[T]$, the following estimate holds with probability at least $1-\delta$ over the randomness in $\gF_t^3$ (we denote this event as $\gE_x$):
\begin{small}
\begin{equation} \label{eq:event-G-1}
    \left\|\sum_{i=0}^{t}\beta^{t-i}\left(\hatphi(x_i,y_i,z_i;\xi_i',\zeta_i')-\E_t[\hatphi(x_i,y_i,z_i;\xi_i',\zeta_i')]\right)\right\| \leq 4\left(\sigma_{f,1} + \frac{l_{f,0}\sigma_{g,2}}{\mu} + 2\sigma_{g,2}\Delta_{z,0}\right)\sqrt{\frac{\log(2/\delta)}{1-\beta}}.
\end{equation}
\end{small}%
As a consequence, under event $\gE_{\init} \cap \gE_y\cap\gE_z$, for any given $\delta\in(0,1)$ and all $t\in[T]$, the following estimate holds with probability at least $1-\delta$ over the randomness in $\gF_{T}^3$:
\begin{small}
\begin{equation} \label{eq:event-G-2}
    \left\|\sum_{i=0}^{t}\beta^{t-i}\left(\hatphi(x_i,y_i,z_i;\xi_i',\zeta_i')-\E_t[\hatphi(x_i,y_i,z_i;\xi_i',\zeta_i')]\right)\right\| \leq 4\left(\sigma_{f,1} + \frac{l_{f,0}\sigma_{g,2}}{\mu} + 2\sigma_{g,2}\Delta_{z,0}\right)\sqrt{\frac{\log(2T/\delta)}{1-\beta}}.
\end{equation}
\end{small}%
\end{lemma}

\begin{proof}[Proof of Lemma~\ref{lm:MDS-var}]
By definition of $\hatphi(x_t,y_t,z_t;\xi_t',\zeta_t')$, we have the following decomposition:
\begin{small}
\begin{equation*}
    \begin{aligned}
        &\hatphi(x_t,y_t,z_t;\xi_t',\zeta_t') - \E_t[\hatphi(x_t,y_t,z_t;\xi_t',\zeta_t')] \\
        &= [\gdx F(x_t,y_t;\xi_t') - \gdxy G(x_t,y_t,;\zeta_t')z_t] - [\gdx f(x_t,y_t) - \gdxy g(x_t,y_t)]z_t \\
        &= [\gdx F(x_t,y_t;\xi_t') - \gdx f(x_t,y_t)] - [\gdxy G(x_t,y_t;\zeta_t') - \gdxy g(x_t,y_t)]z_t \\
        &= [\gdx F(x_t,y_t;\xi_t) - \gdx f(x_t,y_t)] + [\gdxy G(x_t,y_t;\zeta_t) - \gdxy g(x_t,y_t)]z_t^* - [\gdxy G(x_t,y_t;\zeta_t) - \gdxy g(x_t,y_t)](z_t-z_t^*). \\
    \end{aligned}
\end{equation*}
\end{small}%
For simplicity, we define $\gE_{yz} \coloneqq \gE_{\init} \cap \gE_y\cap\gE_z$ and $\E_{\mid\gE_{yz}}[\cdot] \coloneqq \E[\cdot \mid \gE_{yz}]$, where events $\gE_{\init}$, $\gE_y$ and $\gE_z$ are defined in Lemma~\ref{lm:warm-up},~\ref{lm:high-prob 1/8L_1 bound for y} and~\ref{lm:z-two-bound}.
Then for any $i\in[t]$, we have
\begin{small}
\begin{equation*}
    \begin{aligned}
        \beta^{t-i}[\gdx F(x_i,y_i;\xi_i') - \gdx f(x_i,y_i)] \in \gF_{i+1}^3, \quad &\E_{\mid \gE_{yz}}\left[\beta^{t-i}[\gdx F(x_i,y_i;\xi_i') - \gdx f(x_i,y_i)] \mid \gF_i^3\right] = 0; \\
        \beta^{t-i}[\gdxy G(x_i,y_i;\zeta_i') - \gdxy g(x_i,y_i)]z_i^* \in \gF_{i+1}^3, \quad &\E_{\mid \gE_{yz}}\left[\beta^{t-i}[\gdxy G(x_i,y_i;\zeta_i') - \gdxy g(x_i,y_i)]z_i^* \mid \gF_i^3\right] = 0; \\
        \beta^{t-i}[\gdxy G(x_i,y_i;\zeta_i') - \gdxy g(x_i,y_i)](z_i-z_i^*) \in \gF_{i+1}^3, \quad &\E_{\mid \gE_{yz}}\left[\beta^{t-i}[\gdxy G(x_i,y_i;\zeta_i') - \gdxy g(x_i,y_i)](z_i-z_i^*) \mid \gF_i^3\right] = 0. \\
    \end{aligned}
\end{equation*}
\end{small}%
Also by Assumption~\ref{ass:highprob} and Lemma~\ref{lm:z-two-bound}, under event $\gE_{yz}\in\sigma(\widetilde{\gF}_{T_0}^1\cup\gF_T^1 \cup \gF_T^2)$, the followings hold almost surely in $\gF_t^3$:
\begin{small}
\begin{equation*}
    \begin{aligned}
        \|\beta^{t-i}[\gdx F(x_i,y_i;\xi_i') - \gdx f(x_i,y_i)]\| &\leq \beta^{t-i}\sigma_{f,1} \\
        \|\beta^{t-i}[\gdxy G(x_i,y_i;\zeta_i') - \gdxy g(x_i,y_i)]z_i^*\| &\leq \beta^{t-i}\sigma_{g,2}\frac{l_{f,0}}{\mu} \\
        \|\beta^{t-i}[\gdxy G(x_i,y_i;\zeta_i') - \gdxy g(x_i,y_i)](z_i-z_i^*)\| &\leq 2\beta^{t-i}\sigma_{g,2}\Delta_{z,0}
    \end{aligned}
\end{equation*}
\end{small}%
Thus under event $\gE_{yz}$, for any $i\in[t]$, $\beta^{t-i}(\hatphi(x_i,y_i,z_i;\xi_i',\zeta_i')-\E_i[\hatphi(x_i,y_i,z_i;\xi_i',\zeta_i')])$ is a (almost surely) bounded martingale difference sequence. Now for any given $\delta\in(0,1)$, we apply Lemma~\ref{lm:MDS} to obtain under event $\gE_{yz}$, with probability at least $1-\delta$ over the randomness in $\gF_t^3$ that:
\begin{small}
\begin{equation*}
    \begin{aligned}
        \left\|\sum_{i=0}^{t}\beta^{t-i}\left(\hatphi(x_i,y_i,z_i;\xi_i',\zeta_i')-\E_i[\hatphi(x_i,y_i,z_i;\xi_i',\zeta_i')]\right)\right\| 
        &\leq 4\sqrt{\log(2/\delta)\sum_{i=0}^{t}\left[\beta^{t-i}\left(\sigma_{f,1} + \frac{l_{f,0}\sigma_{g,2}}{\mu} + 2\sigma_{g,2}\Delta_{z,0}\right)\right]^2} \\
        &\leq 4\left(\sigma_{f,1} + \frac{l_{f,0}\sigma_{g,2}}{\mu} + 2\sigma_{g,2}\Delta_{z,0}\right)\sqrt{\frac{\log(2/\delta)}{1-\beta}},
    \end{aligned}
\end{equation*}
\end{small}%
which is as claimed in \eqref{eq:event-G-1}. We obtain \eqref{eq:event-G-2} by applying union bound.
\end{proof}

\subsection{Proof of Theorem~\ref{thm:highprob}}

By incorporating Lemma~\ref{lm:high-prob 1/8L_1 bound for y},~\ref{lm:z-two-bound} and~\ref{lm:MDS-var}, we begin to prove Theorem~\ref{thm:highprob}.

\begin{theorem}[Restatement of Theorem~\ref{thm:highprob}] \label{thm:high-prob-main-appendix}
Suppose Assumptions~\ref{ass:relax-smooth} and~\ref{ass:highprob} hold. Let $\{x_t\}$ be the iterates produced by Algorithm~\ref{alg:bilevel}. For any given $\delta\in(0,1)$ and sufficiently small $\eps$ (see exact choice of $\eps$ in \eqref{eq:eps-high-prob}), if we choose $\alpha^{\init}, \alpha,\beta,\gamma,\eta, T_0$ as \eqref{eq:beta-high-prob}, \eqref{eq:eta-high-prob} and \eqref{eq:gamma+alpha-high-prob}, 
then with probability at least $1-4\delta$ over the randomness in $\mathcal{F}_T$, Algorithm~\ref{alg:bilevel} guarantees $\frac{1}{T}\sum_{t=0}^{T}\|\gdphi(x_t)\| \leq \eps$ with at most $T = \frac{4\Delta_0}{\eta\eps}$ iterations. 
\end{theorem}

\begin{proof}[Proof of Theorem~\ref{thm:high-prob-main-appendix}]
By similar approach as in Lemma~\ref{lm:moving-average error in expectation}, we obtain
\begin{small}
\begin{equation}
    \begin{aligned}
        &\sum_{t=0}^{T-1}\|\eps_t\|
        \leq (1-\beta)\sum_{t=0}^{T-1}\left\|\sum_{i=0}^{t}\beta^{t-i}\left(\hatphi(x_i,y_i,z_i;\xi_i',\zeta_i')-\E_i[\hatphi(x_i,y_i,z_i;\xi_i',\zeta_i')]\right)\right\| + \frac{\eta L_1\beta}{1-\beta}\sum_{t=0}^{T-1}\|\gdphi(x_t)\| \\
        &\quad\quad+ (1-\beta)\sum_{t=0}^{T-1}\left\|\sum_{i=0}^{t}\beta^{t-i}\left(\E_{i}[\hatphi(x_i,y_i,z_i;\xi_i',\zeta_i')]-\gdphi(x_i)\right)\right\|
        + \frac{T\eta L_0\beta}{1-\beta} + \frac{\beta}{1-\beta}\|m_0-\gdphi(x_0)\| \\
        &\leq 4T(1-\beta)\left(\sigma_{f,1} + \frac{l_{f,0}\sigma_{g,2}}{\mu} + 2\sigma_{g,2}\Delta_{z,0}\right)\sqrt{\frac{\log(2T/\delta)}{1-\beta}} + \frac{\eta L_1\beta}{1-\beta}\sum_{t=0}^{T-1}\|\gdphi(x_t)\| \\
        &\quad+ (1-\beta)\sum_{t=0}^{T-1}\sum_{i=0}^{t}\beta^{t-i}L_{x,1}\|y_i-y_i^*\|\|\gdphi(x_i)\| + (1-\beta)\sum_{t=0}^{T-1}\sum_{i=0}^{t}\beta^{t-i}\left(L_{x,0} + L_{x,1}\frac{l_{g,1}l_{f,0}}{\mu} + \frac{l_{g,2}l_{f,0}}{\mu}\right)\|y_i-y_i^*\| \\
        &\quad+ (1-\beta)\sum_{t=0}^{T-1}\sum_{i=0}^{t}\beta^{t-i}l_{g,1}\|z_i-z_i^*\| + \frac{T\eta L_0\beta}{1-\beta} + \frac{\beta}{1-\beta}\|\gdphi(x_0)\| \\
        &\leq 4T(1-\beta)\left(\sigma_{f,1} + \frac{l_{f,0}\sigma_{g,2}}{\mu} + 2\sigma_{g,2}\Delta_{z,0}\right)\sqrt{\frac{\log(2T/\delta)}{1-\beta}} + \left(\frac{\eta L_1\beta}{1-\beta}+\frac{1}{8}\right)\sum_{t=0}^{T-1}\|\gdphi(x_t)\| + \frac{3T\eps'}{32} \\
        &\quad+ T\left(1 + \sqrt{E+\frac{l_{z^*}^2}{4l_{g,1}^2}}\right)\frac{l_{g,1}\eps'}{32L_0} + \frac{T\eta L_0\beta}{1-\beta} + \frac{\beta}{1-\beta}\|\gdphi(x_0)\| \\
    \end{aligned}
\end{equation}
\end{small}%
By \eqref{eq:aaa} (without taking expectation), we have
\begin{small}
\begin{equation*}
    \left(1-\frac{1}{2}\eta L_1\right)\frac{1}{T}\sum_{t=0}^{T-1}\|\gdphi(x_t)\| \leq \frac{\Delta_0}{T\eta} + \frac{1}{2}\eta L_0 + \frac{2}{T}\sum_{t=0}^{T-1}\|\eps_t\|.
\end{equation*}
\end{small}%
Under event $\gE_{\init}\cap\gE_y\cap\gE_z$, plug \eqref{eq:moving-average error sum} into the above inequality, then we have 
\begin{small}
    \begin{align}
        &\left(1-\left(\frac{1}{2} + \frac{2\beta}{1-\beta}\right)\eta L_1 - \frac{1}{4}\right) \frac{1}{T}\sum_{t=0}^{T-1}\E\|\gdphi(x_t)\| \label{eq:h-1} \\
        &\quad\quad\leq 8\left(\sigma_{f,1} + \frac{l_{f,0}\sigma_{g,2}}{\mu} + 2\sigma_{g,2}\Delta_{z,0}\right)\sqrt{(1-\beta)\log\left(\frac{2T}{\delta}\right)} + \frac{3\eps'}{16} + \left(1 + \sqrt{E+\frac{l_{z^*}^2}{4l_{g,1}^2}}\right)\frac{l_{g,1}\eps'}{16L_0} \label{eq:h-2} \\
        &\quad\quad\quad+ 
        \frac{2\eta L_0\beta}{1-\beta} + \frac{\Delta_0}{T\eta} + \frac{1}{2}\eta L_0 + \frac{2\beta}{T(1-\beta)}\|\gdphi(x_0)\|. \label{eq:h-3}
    \end{align}
\end{small}%
Now we proceed to bound \eqref{eq:h-1}, \eqref{eq:h-2} and \eqref{eq:h-3}, respectively.

For left-hand side \eqref{eq:h-1} of the inequality, by \eqref{eq:main-line-1-bound} we have 
\begin{small}
\begin{equation} \label{eq:h-1-bound}
    \left(1-\left(\frac{1}{2} + \frac{2\beta}{1-\beta}\right)\eta L_1 - \frac{1}{4}\right) \frac{1}{T}\sum_{t=0}^{T-1}\|\gdphi(x_t)\| \geq \frac{1}{2T}\sum_{t=0}^{T-1}\|\gdphi(x_t)\|.
\end{equation}
\end{small}%

For \eqref{eq:h-2} on right-hand side of the inequality, we have
\begin{small}
\begin{equation} \label{eq:h-2-bound}
    \begin{aligned}
        8\left(\sigma_{f,1} + \frac{l_{f,0}\sigma_{g,2}}{\mu} + 2\sigma_{g,2}\Delta_{z,0}\right)\sqrt{(1-\beta)\log\left(\frac{2T}{\delta}\right)} 
        &\leq 8\left(\sigma_{f,1} + \frac{l_{f,0}\sigma_{g,2}}{\mu} + 2\sigma_{g,2}\Delta_{z,0}\right)\sqrt{\frac{\mu^2}{64\sigma_{g,1}^2}\frac{\eps'^2}{1024L_0^2}} \\
        &= \frac{\mu}{\sigma_{g,1}}\left(\sigma_{f,1} + \frac{l_{f,0}\sigma_{g,2}}{\mu} + 2\sigma_{g,2}\Delta_{z,0}\right)\frac{\eps'}{32L_0}, \\
    \end{aligned}
\end{equation}
\end{small}%
where for the first inequality we use \eqref{eq:merge-step} in Lemma~\ref{lm:high-prob 1/8L_1 bound for y},
\begin{small}
\begin{equation*}
    \frac{64(1-\beta)\sigma_{g,1}^2}{\mu^2}\log\left(\frac{eT}{\delta}\right) \leq \frac{\eps'^2}{1024L_0^2}.
\end{equation*}
\end{small}%
For \eqref{eq:h-3} on right-hand side of the inequality, by \eqref{eq:main-line-2-bound} we have
\begin{small}
\begin{equation} \label{eq:h-3-bound}
    \begin{aligned}
        \frac{2\eta L_0\beta}{1-\beta} + \frac{\Delta_0}{T\eta} + \frac{1}{2}\eta L_0 + \frac{2\beta}{T(1-\beta)}\|\gdphi(x_0)\|
        \leq \frac{1}{4}\eps' + \frac{1}{4}\eps' + \frac{1}{16}\eps' + \frac{1}{16}\eps' = \frac{5}{8}\eps'.
    \end{aligned}
\end{equation}
\end{small}%
Combining \eqref{eq:h-1-bound}, \eqref{eq:h-2-bound} and \eqref{eq:h-3-bound} together yields
\begin{small}
\begin{equation*}
    \begin{aligned}
        \frac{1}{T}\sum_{t=0}^{T-1}\|\gdphi(x_t)\| 
        &\leq 2\left[\frac{\mu}{\sigma_{g,1}}\left(\sigma_{f,1} + \frac{l_{f,0}\sigma_{g,2}}{\mu} + 2\sigma_{g,2}\Delta_{z,0}\right)\frac{\eps'}{32L_0} + \frac{3\eps'}{16} + \left(1 + \sqrt{E+\frac{l_{z^*}^2}{4l_{g,1}^2}}\right)\frac{l_{g,1}\eps'}{16L_0} + \frac{5\eps'}{8}\right] \\
        &= \eps'\left[\frac{\mu}{16\sigma_{g,1}L_0}\left(\sigma_{f,1} + \frac{l_{f,0}\sigma_{g,2}}{\mu} + 2\sigma_{g,2}\Delta_{z,0}\right) + \left(1 + \sqrt{E+\frac{l_{z^*}^2}{4l_{g,1}^2}}\right)\frac{l_{g,1}}{8L_0} + \frac{13}{8}\right].
    \end{aligned}
\end{equation*}
\end{small}%
Recall constant $G$ in \eqref{eq:eps-high-prob} is defined as
\begin{small}
\begin{equation*}
    \begin{aligned}
        G \coloneqq \frac{\mu}{16\sigma_{g,1}L_0}\left(\sigma_{f,1} + \frac{l_{f,0}\sigma_{g,2}}{\mu} + 2\sigma_{g,2}\Delta_{z,0}\right) + \left(1 + \sqrt{E+\frac{l_{z^*}^2}{4l_{g,1}^2}}\right)\frac{l_{g,1}}{8L_0} + \frac{13}{8}.
    \end{aligned}
\end{equation*}
\end{small}%
Again by \eqref{eq:eps-high-prob}, we have $\eps'=\eps/G$, thus we obtain $\frac{1}{T}\sum_{t=0}^{T-1}\|\gdphi(x_t)\| \leq \eps$.
Also note that 
\begin{small}
\begin{equation*}
    \Pr(\gE_x \cap \gE_y \cap \gE_z \cap \gE_{\init}) = \Pr(\gE_x \mid \gE_y \cap \gE_z \cap \gE_{\init}) \cdot \Pr(\gE_z \mid \gE_y \cap \gE_{\init}) \cdot \Pr(\gE_y \mid \gE_{\init}) \cdot \Pr(\gE_{\init}) \geq (1-\delta)^4 \geq 1 - 4\delta.
\end{equation*}
\end{small}%
Therefore, with probability at least $1-4\delta$ over the randomness in $\mathcal{F}_T$, we have $\frac{1}{T}\sum_{t=0}^{T}\|\gdphi(x_t)\| \leq \eps$. 
\end{proof}

\section{Experimental Parameter Selection}
\subsection{Hyper-representation} \label{sec:exp_set_hr}
In the experiments, we utilized grid search to find the best lower-level and upper-level learning rates within the scope of $(0.001, 0.5)$ for various methods. Specifically, on the SNLI dataset, optimal learning rate pairs are determined as $(0.01, 0.01)$ for MAML, $(0.01, 0.05)$ for ANIL, $(0.01, 0.01)$ for StocBio, $(0.02, 0.1)$ for TTSA, $(0.01, 0.05)$ for SABA, $(0.05, 0.05)$ for MA-SOBA, and $(0.05, 0.1)$ for both BO-REP and SLIP. On ARD dataset, the best combinations are $(0.05, 0.1)$ for both MAML and ANIL, $(0.05, 0.05)$ for StocBio, $(0.1, 0.01)$ for TTSA, $(0.05, 0.05)$ for SABA, $(0.05, 0.1)$ for MA-SOBA, $(0.001, 0.01)$ for BO-REP, and $(0.01, 0.1)$ for SLIP.

The following settings are applied to both SNLI and ARD datasets: For double-loop frameworks such as MAML, ANIL, and StocBio, the inner-loop iteration count is searched among 5, 10, and 20, where 5 is the best choice for these methods. For approaches SABA, MA-SOBA, BO-REP, and SLIP, the step size for the linear system variable $z$ is consistently chosen as $0.01$, based on the optimal parameter tuning from range of $(0.001, 0.1)$. For the fully first-order method F$^2$SA, which incorporates three distinct decision variables, corresponding learning rates are fine-tuned from a range of $(0.001, 0.5)$, with the best configuration being $(0.05, 0.05, 0.01)$. The momentum parameter $\beta$ for MA-SOBA, BO-REP, and SLIP is set to $0.9$. The Lagrangian multiplier $\lambda$ in F$^2$SA is increased by $0.01$ with each outer update. BO-REP updates the lower-level variable at every 2 outer intervals and conducts 3 iterations for each update. For SLIP, the number of warm-start steps for lower-level updates is set to 3. The batch size is set to 50 for all the methods.

\subsection{Data Hyper-cleaning} \label{sec:exp_set_hc}
In data hyper-cleaning, we also employ a grid search technique for all baseliens to determine the optimal lower-level and upper-level learning rates from the range of  $(0.01, 0.1)$. The best learning rate pairings for lower-level and upper-level are $(0.05, 0.05)$ for StocBio, $(0.05, 0.01)$ for TTSA, $(0.1, 0.05)$ for both SABA and MA-SOBA, and $(0.05, 0.05)$ for BO-REP and SLIP. For the F$^2$SA algorithm, which manages three decision variables, the chosen learning rates were $(0.05, 0.05, 0.01)$. Additionally, the step size for addressing the linear system variable $z$ in SABA, MA-SOBA, BO-REP, and SLIP is set as $0.01$.

The number of inner loops in StocBio was fixed as 3, based on a selection range of ${3, 5, 10}$. For BO-REP, the update interval and the number of iterations per lower-level update are consistently set at 2 and 3, respectively. A uniform batch size of 512 was applied for all baselines. Other key experimental settings, such as the momentum parameters for MA-SOBA, BO-REP, and SLIP,as well as the increment of the Lagrangian multiplier in F$^2$SA, are the same as the specifications detailed in Section~\ref{sec:exp_set_hr}.

\section{Experiments with a Revised Epoch Definition}
\label{sec:epoch-def}

In this section, we clarify the notion of ``epoch'' in our previous experiments, where an epoch means a full pass over the validation set (for upper-level variable $x$ update). For a more comprehensive comparison, we re-conduct experiments where an epoch means a full pass over the training set (for lower-level variable $y$ update). In this case there are fewer $x$ updates than the previous run due to the warm-start phase. The results show that SLIP is still empirically better than other baselines.

\subsection{Hyper-representation Learning}

\begin{figure*}[!h]
\begin{center}
\subfigure[Training accuracy on SNLI]{\includegraphics[width=0.24\linewidth]{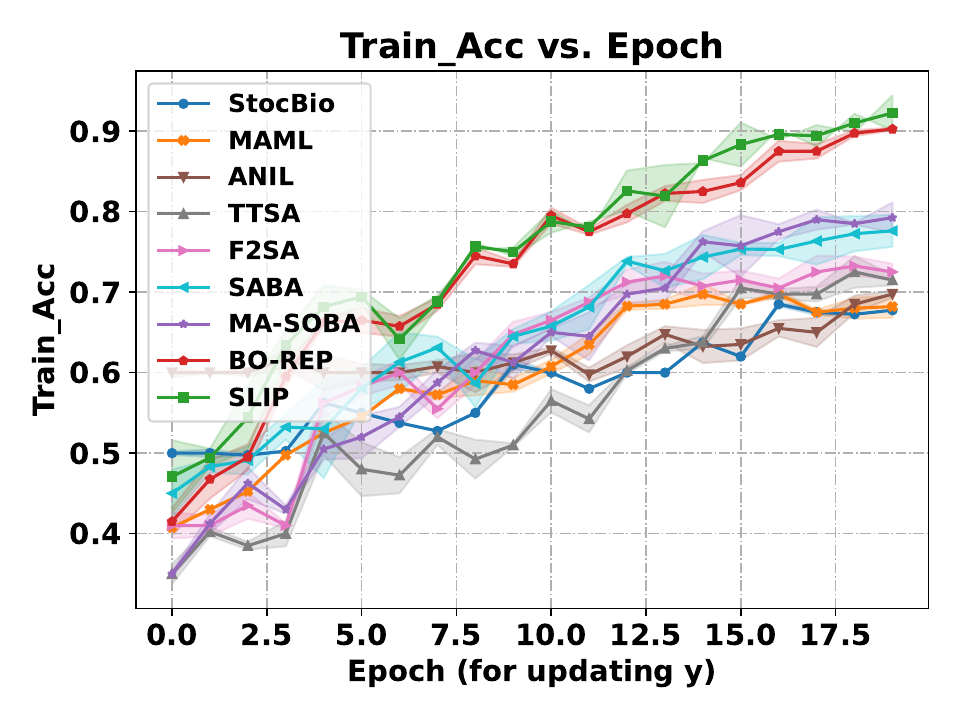}}   \   
\subfigure[Test accuracy on SNLI]{\includegraphics[width=0.24\linewidth]{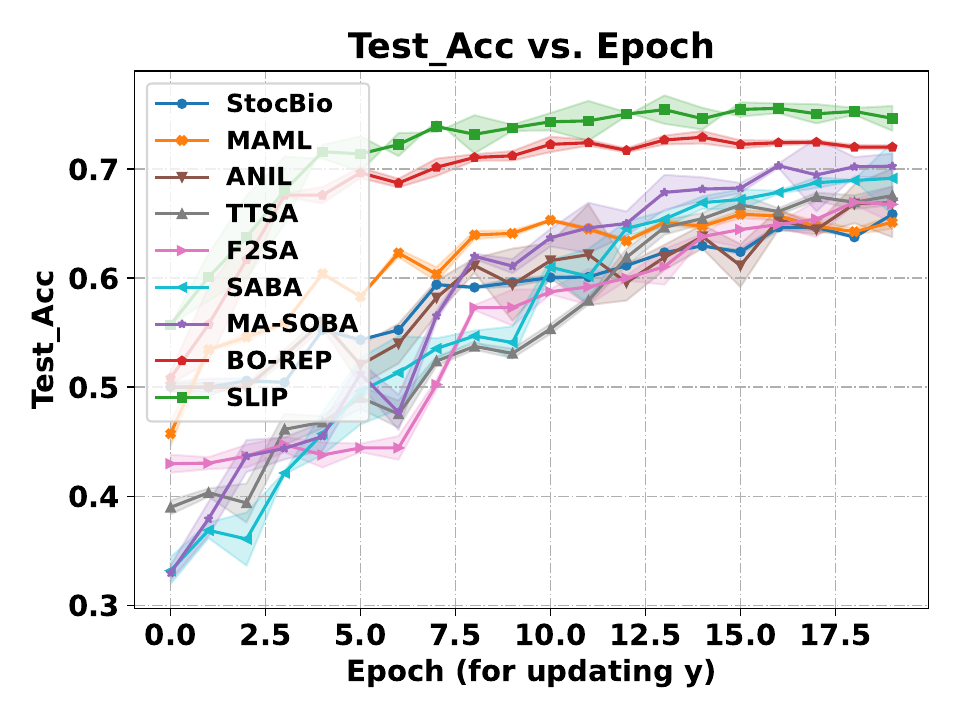}} \ 
\subfigure[Training accuracy on ARD]{\includegraphics[width=0.24\linewidth]{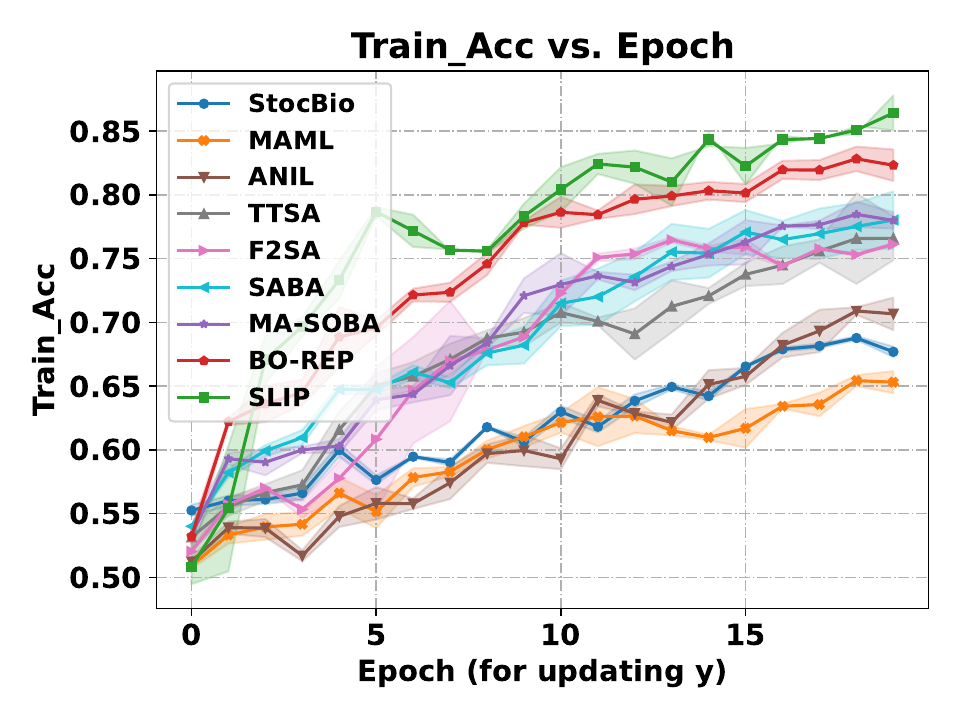}} \
\subfigure[Test accuracy on ARD]{\includegraphics[width=0.24\linewidth]{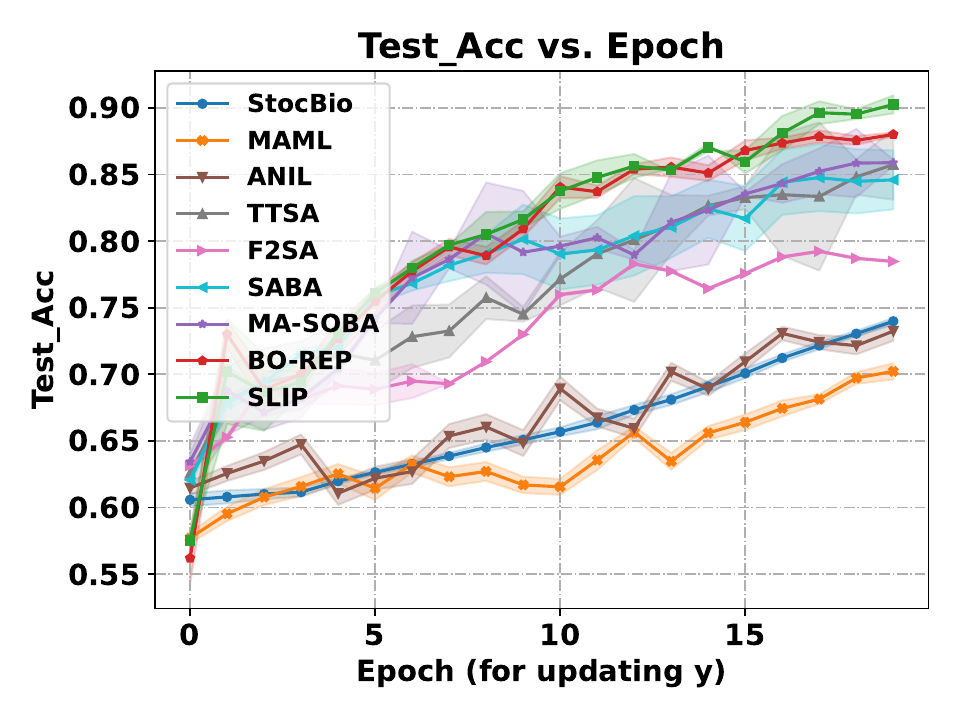}}
\end{center}
\vspace*{-0.2in}
\caption{Comparison with bilevel optimization baselines on Hyper-representation. Figure (a) and (b) are the results in the SNLI dataset. Figures (c) and (d) are the results of the Amazon Review Dataset (ARD). }
\label{fig:acc_HR_y}
\end{figure*}

\subsection{Data Hyper-cleaning}

\begin{figure*}[!h]
\begin{center}
\subfigure[Training acc with $p=0.2$]{\includegraphics[width=0.245\linewidth]{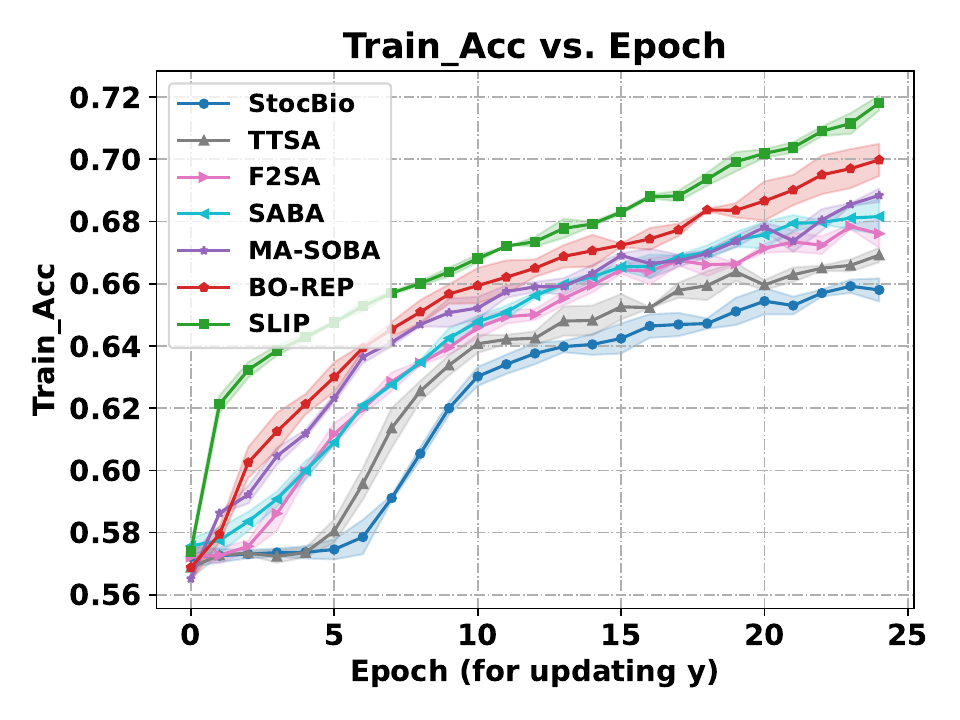}}   
\subfigure[Test acc with $p=0.2$]{\includegraphics[width=0.245\linewidth]{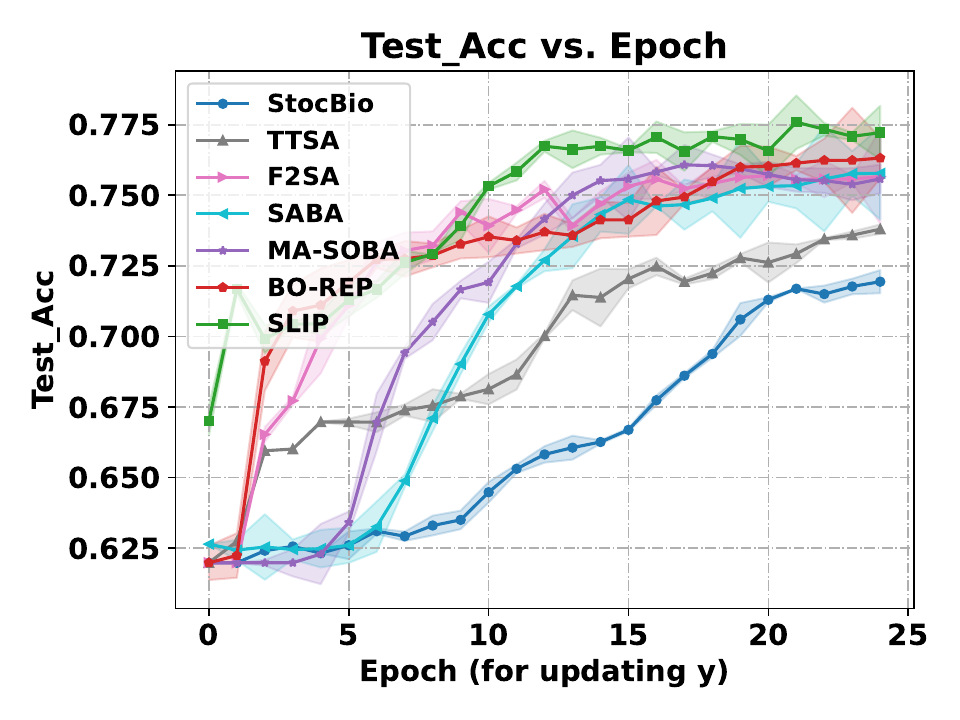}}   
\subfigure[Training acc with $p=0.4$]{\includegraphics[width=0.245\linewidth]{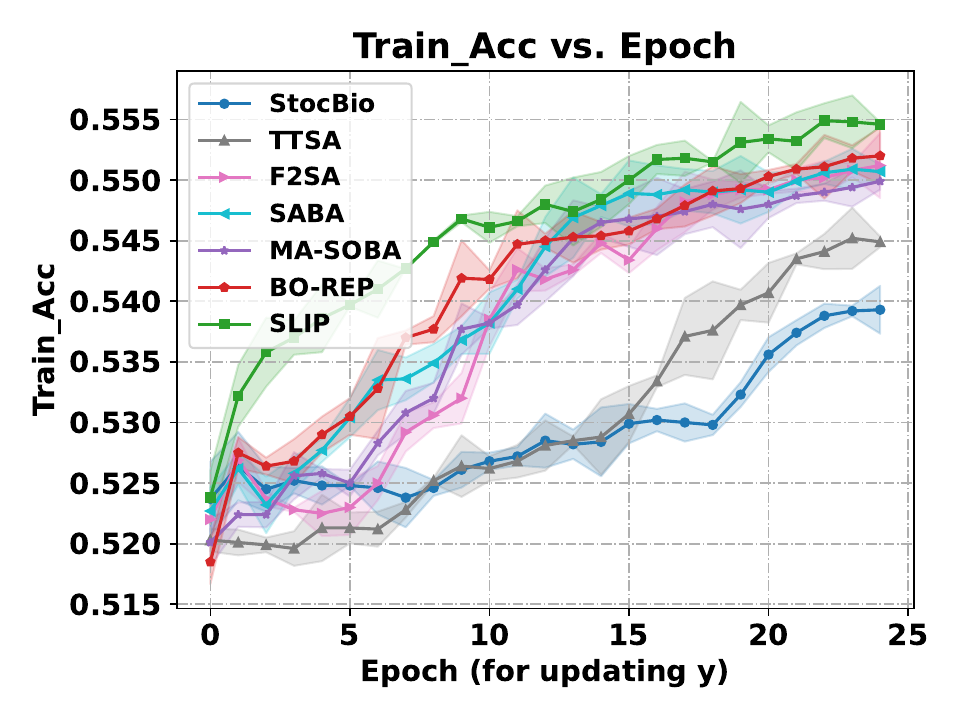}}  
\subfigure[Test acc with $p=0.4$]{\includegraphics[width=0.245\linewidth]{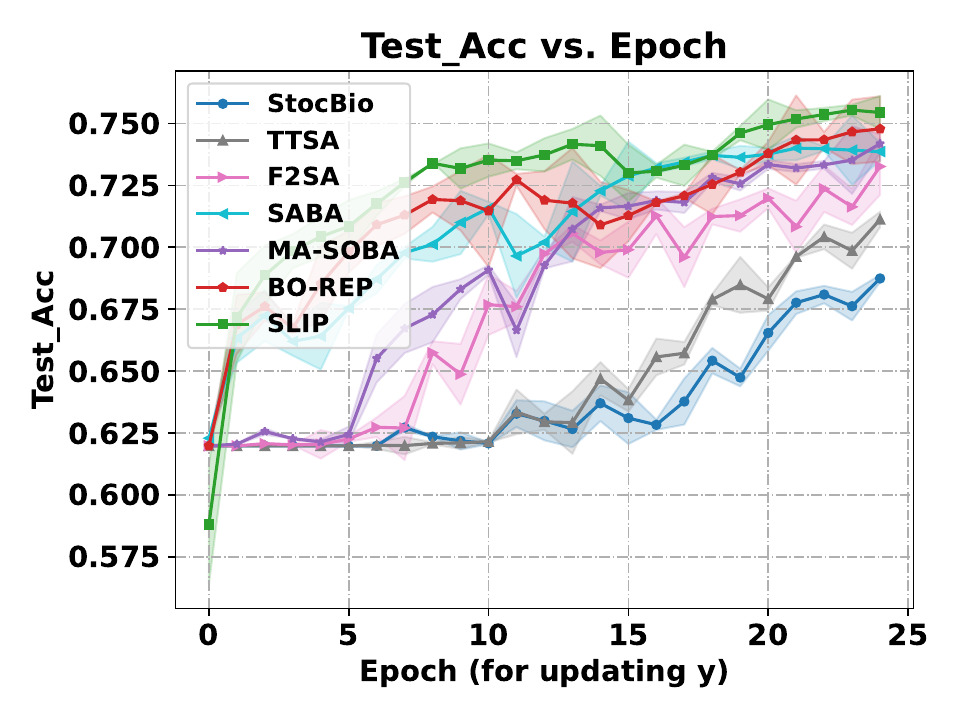}}
\end{center}
\vspace*{-0.2in}
\caption{Comparison with bilevel optimization baselinses on data hyper-cleaning. Figure (a), (b) are the results with the corruption rate $p=0.2$. Figure (c), (d) are the results with the corruption rate $p=0.4$. }
\label{fig:acc_HC_y}
\end{figure*}